\documentclass[10pt]{article} % For LaTeX2e
\usepackage[accepted]{tmlr}
\usepackage{hyperref}
\usepackage{url}
\usepackage{caption}
\usepackage{subcaption}
\usepackage{graphicx}
\usepackage{amsmath}
\usepackage{amssymb}
\usepackage{mathtools}
\usepackage{amsthm}
\usepackage{bm}
\usepackage[capitalize]{cleveref}
\usepackage{float}

\newcommand{\PutAlexNetInD}{
\begin{figure}[H]
    \centering
    \begin{subfigure}[b]{0.40\textwidth}
         \centering
         \includegraphics[width=\textwidth]{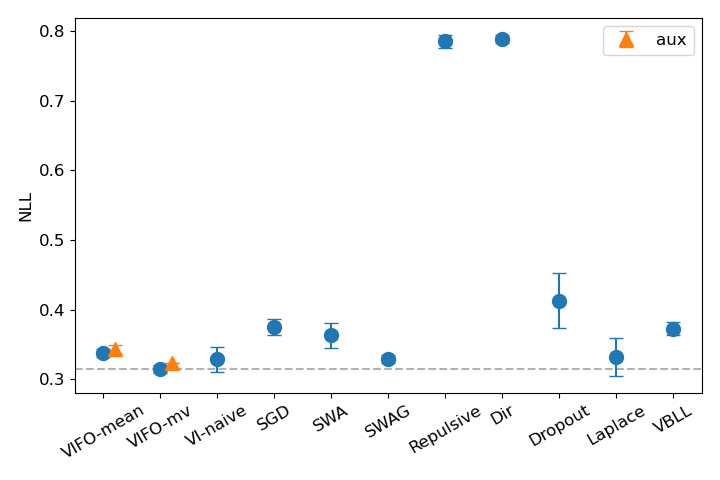}
         \caption{CIFAR10}
    \end{subfigure}
    \begin{subfigure}[b]{0.40\textwidth}
         \centering
         \includegraphics[width=\textwidth]{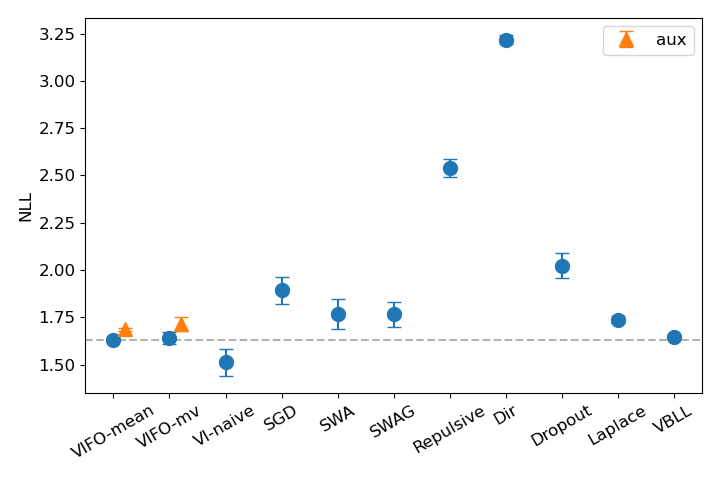}
         \caption{CIFAR100}
    \end{subfigure}
    \begin{subfigure}[b]{0.40\textwidth}
         \centering
         \includegraphics[width=\textwidth]{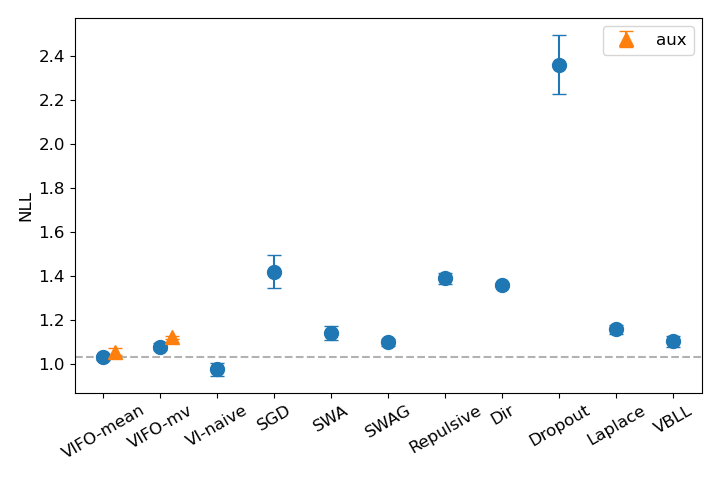}
         \caption{STL10}
    \end{subfigure}
    \begin{subfigure}[b]{0.40\textwidth}
         \centering
         \includegraphics[width=\textwidth]{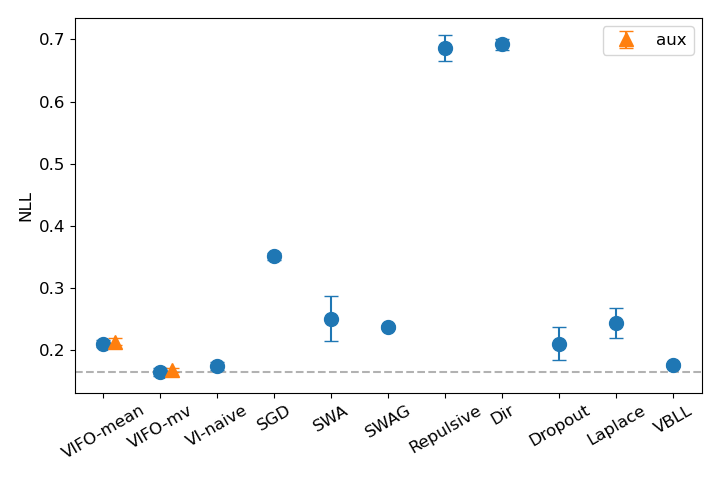}
         \caption{SVHN}
    \end{subfigure}
    \caption{Test log loss ($\downarrow$) of image datasets on AlexNet. Dashed lines indicate the best version of VIFO. The error bar is three times of the standard deviation for better visualization and same for other figures.
    }
    \label{fig:log-alexnet}
\end{figure}

\begin{figure}[H]
    \centering
    \begin{subfigure}[b]{0.40\textwidth}
         \centering
         \includegraphics[width=\textwidth]{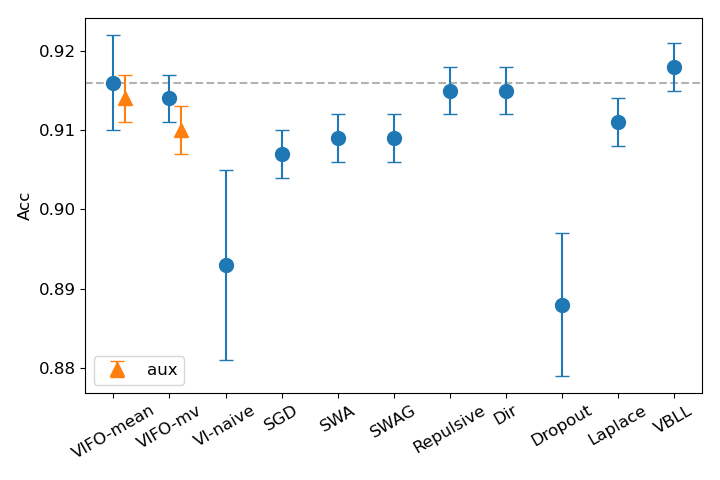}
         \caption{CIFAR10}
    \end{subfigure}
    \begin{subfigure}[b]{0.40\textwidth}
         \centering
         \includegraphics[width=\textwidth]{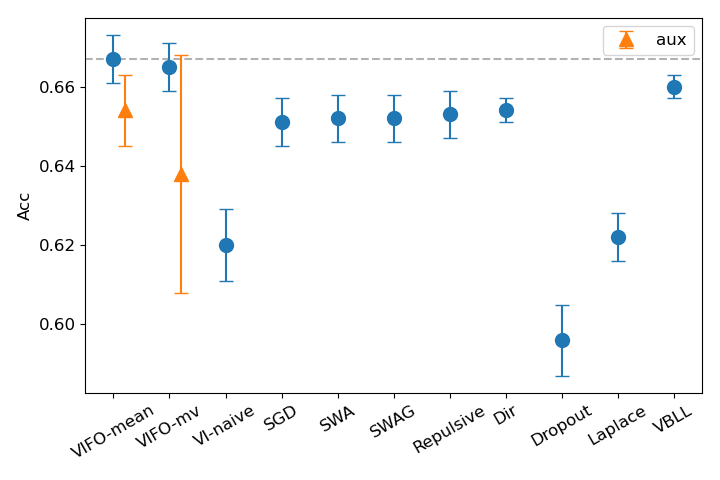}
         \caption{CIFAR100}
    \end{subfigure}
    \begin{subfigure}[b]{0.40\textwidth}
         \centering
         \includegraphics[width=\textwidth]{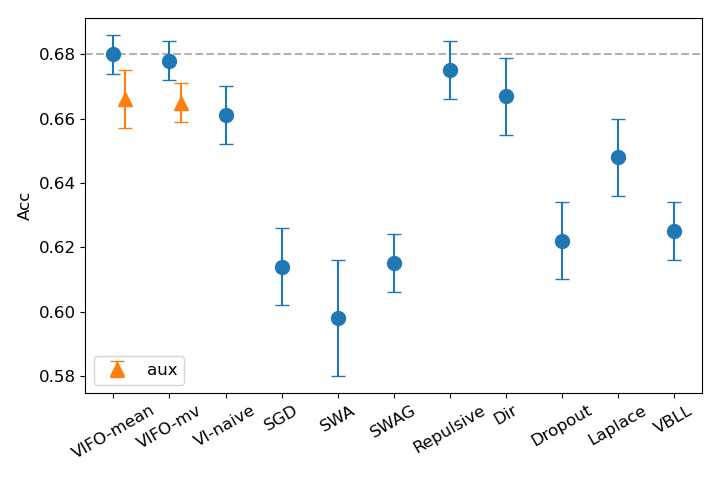}
         \caption{STL10}
    \end{subfigure}
    \begin{subfigure}[b]{0.40\textwidth}
         \centering
         \includegraphics[width=\textwidth]{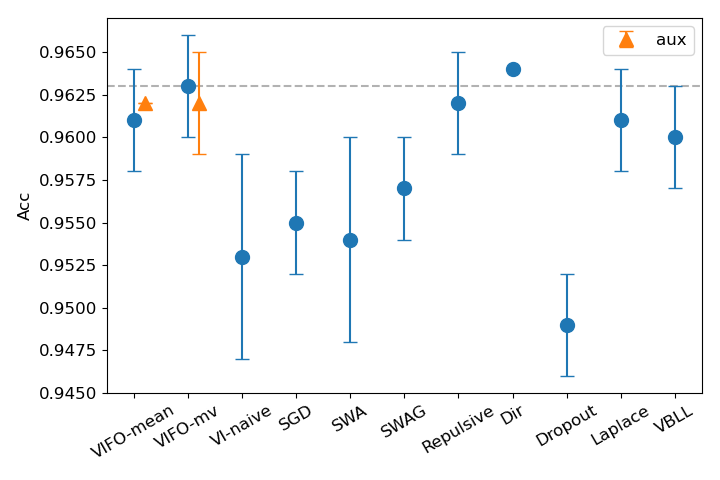}
         \caption{SVHN}
    \end{subfigure}
    \caption{Test accuracy ($\uparrow$) on AlexNet. Dashed lines indicate the best version of VIFO.
    }
    \label{fig:acc-alexnet}
\end{figure}

}

\newcommand{\PutPreResNetInD}{
\begin{figure}[h]
    \centering
    \begin{subfigure}[b]{0.40\textwidth}
         \centering
         \includegraphics[width=\textwidth]{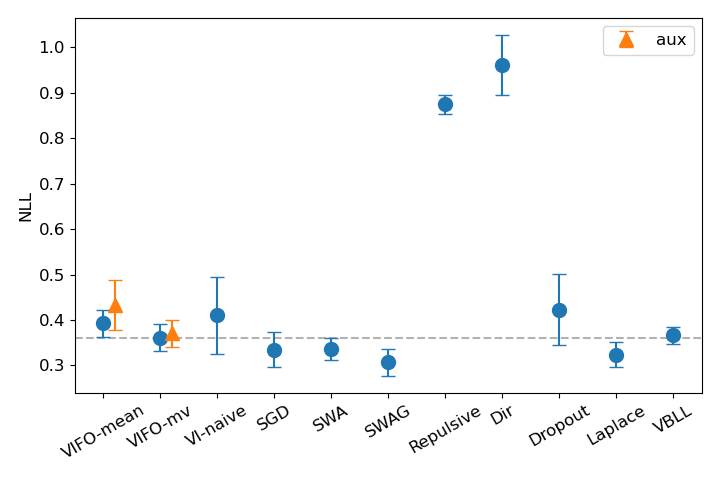}
         \caption{CIFAR10}
    \end{subfigure}
    \begin{subfigure}[b]{0.40\textwidth}
         \centering
         \includegraphics[width=\textwidth]{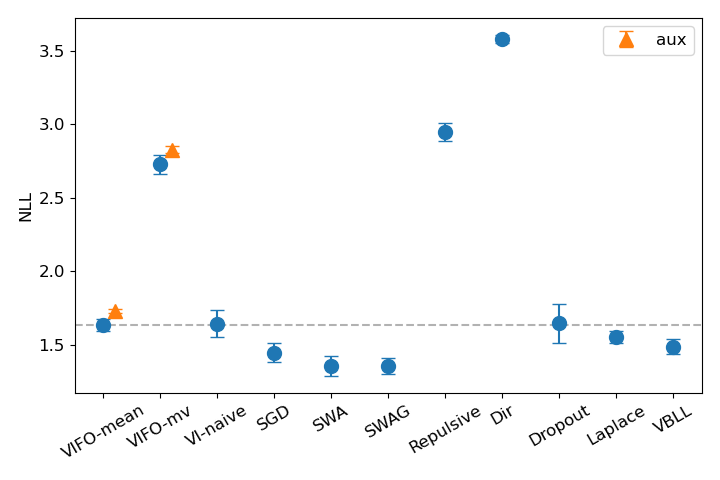}
         \caption{CIFAR100}
    \end{subfigure}
    \begin{subfigure}[b]{0.40\textwidth}
         \centering
         \includegraphics[width=\textwidth]{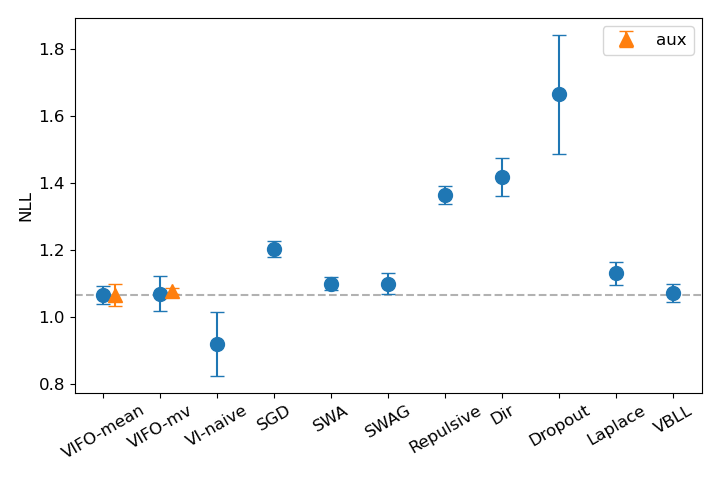}
         \caption{STL10}
    \end{subfigure}
    \begin{subfigure}[b]{0.40\textwidth}
         \centering
         \includegraphics[width=\textwidth]{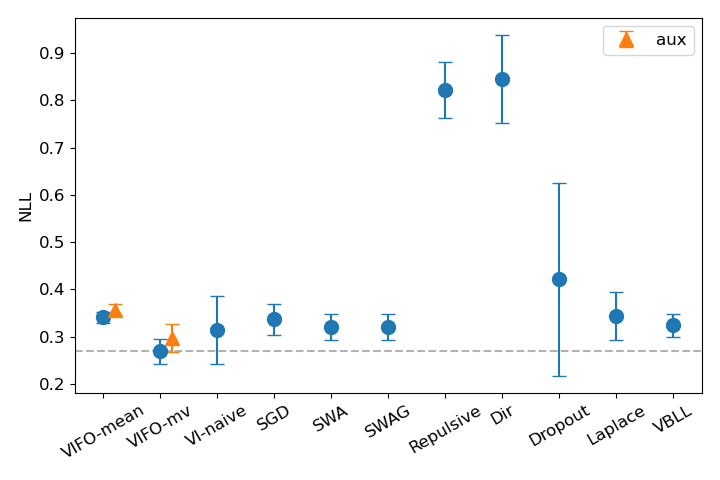}
         \caption{SVHN}
    \end{subfigure}
    \caption{Test log loss ($\downarrow$) on PreResNet20. Dashed lines indicate the best version of VIFO. The error bar is three times of the standard deviation for better visualization and same for other figures.
    }
    \label{fig:log-preresnet}
\end{figure}
}

\newcommand{\PutEntropyAlexNet}{
\begin{figure}[H]
    \centering
    \begin{subfigure}[b]{0.40\textwidth}
         \centering
         \includegraphics[width=\textwidth]{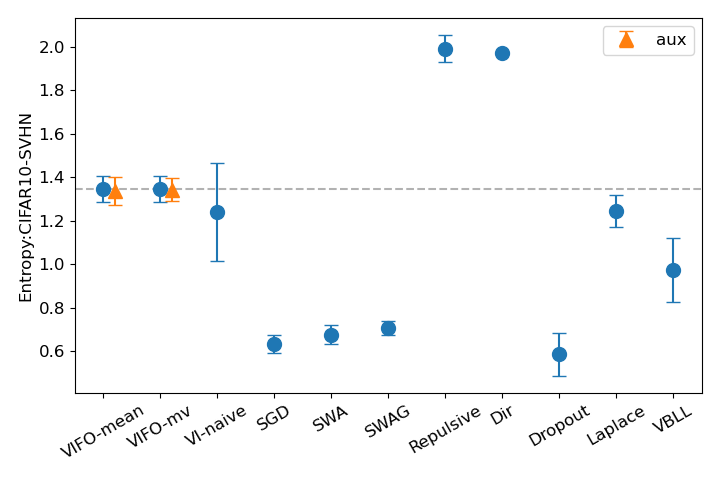}
         \caption{CIFAR10$\rightarrow$SVHN}
    \end{subfigure}
    \begin{subfigure}[b]{0.40\textwidth}
         \centering
         \includegraphics[width=\textwidth]{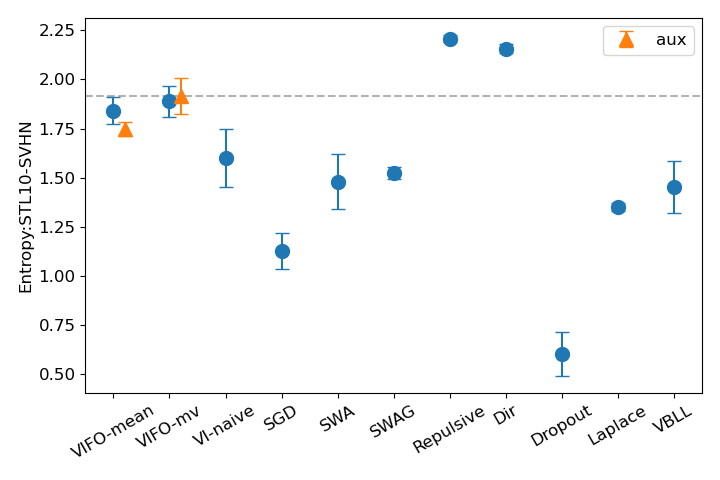}
         \caption{STL10$\rightarrow$SVHN}
    \end{subfigure}
    \begin{subfigure}[b]{0.40\textwidth}
         \centering
         \includegraphics[width=\textwidth]{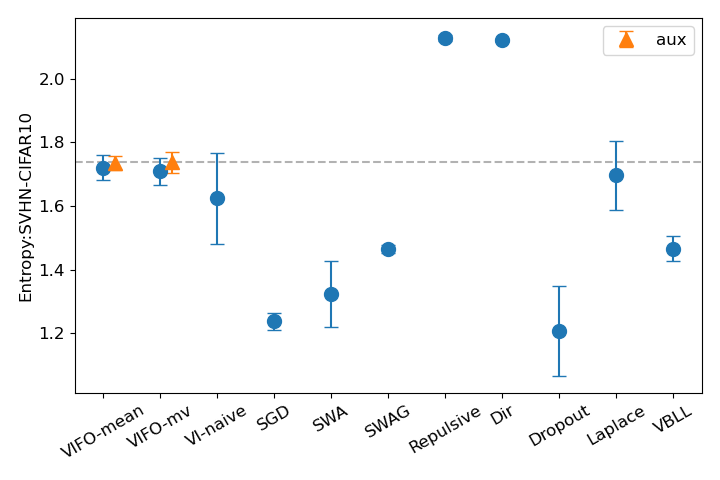}
         \caption{SVHN$\rightarrow$CIFAR10}
    \end{subfigure}
    \begin{subfigure}[b]{0.40\textwidth}
         \centering
         \includegraphics[width=\textwidth]{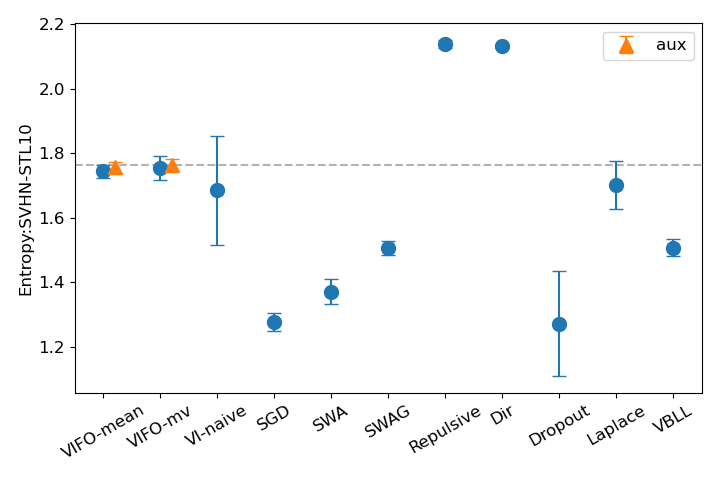}
         \caption{SVHN$\rightarrow$STL10}
    \end{subfigure}
    \caption{Entropy ($\uparrow$) on AlexNet.
    }
    \label{fig:entropy-alexnet}
\end{figure}

}

\newcommand{\PutEntropyPreResNet}{
\begin{figure}[h!]
    \centering
    \begin{subfigure}[b]{0.40\textwidth}
         \centering
         \includegraphics[width=\textwidth]{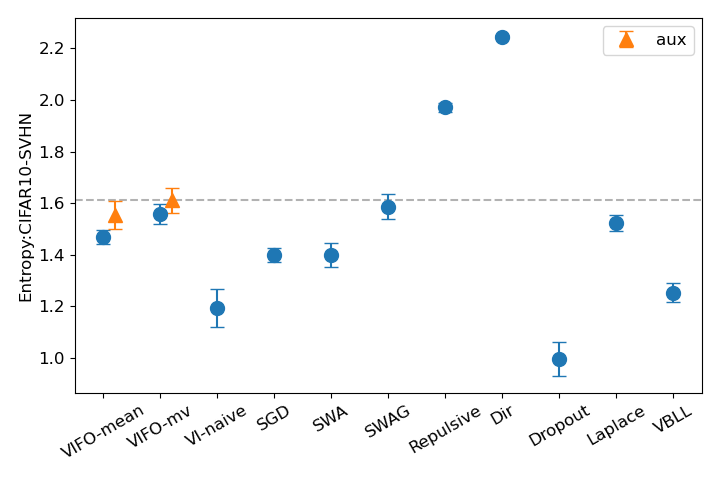}
         \caption{CIFAR10$\rightarrow$SVHN}
    \end{subfigure}
    \begin{subfigure}[b]{0.40\textwidth}
         \centering
         \includegraphics[width=\textwidth]{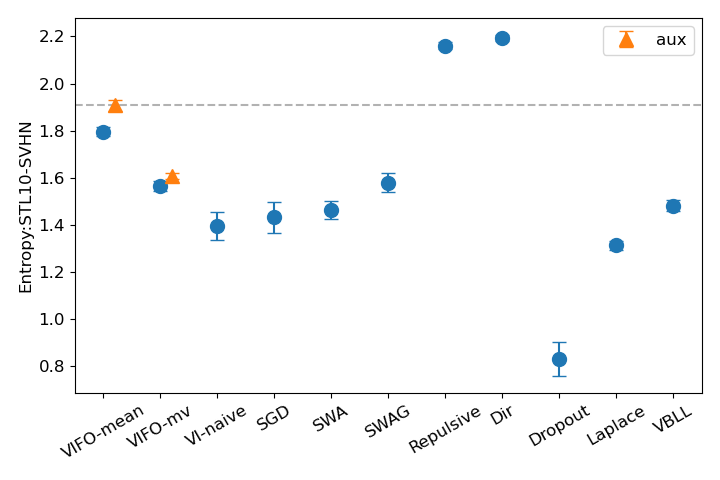}
         \caption{STL10$\rightarrow$SVHN}
    \end{subfigure}
    \begin{subfigure}[b]{0.40\textwidth}
         \centering
         \includegraphics[width=\textwidth]{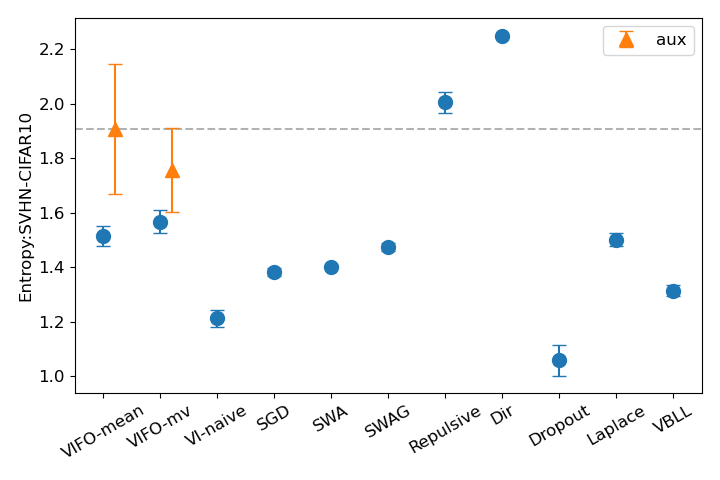}
         \caption{SVHN$\rightarrow$CIFAR10}
    \end{subfigure}
    \begin{subfigure}[b]{0.40\textwidth}
         \centering
         \includegraphics[width=\textwidth]{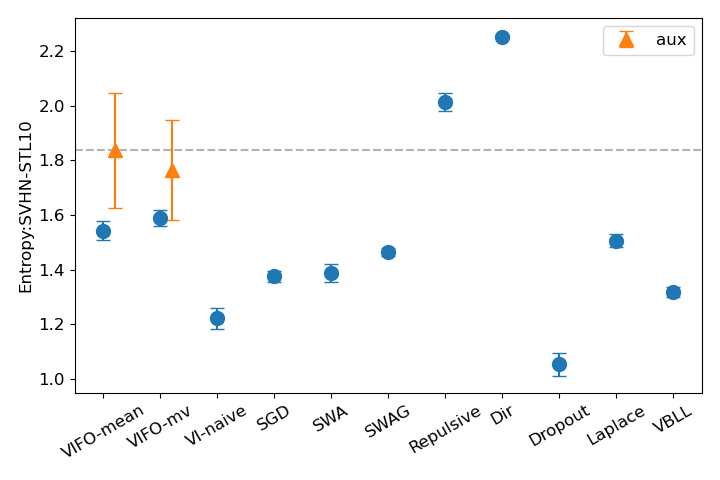}
         \caption{SVHN$\rightarrow$STL10}
    \end{subfigure}
    \caption{Entropy ($\uparrow$) on PreResNet20.
    }
    \label{fig:entropy-preresnet}
\end{figure}

}

\title{Variational Inference on the Final-Layer Output of Neural Networks}

\author{\name Yadi Wei \email weiyadi@iu.edu \\
      \addr Luddy School of Informatics, Computing, and Engineering \\
      Indiana University
      \AND
      \name Roni Khardon \email rkhardon@iu.edu \\
      \addr Luddy School of Informatics, Computing, and Engineering \\
      Indiana University
      }

\def\month{09}  % Insert correct month for camera-ready version
\def\year{2024} % Insert correct year for camera-ready version
\def\openreview{\url{https://openreview.net/forum?id=mTOzXLmLKr}} % Insert correct link to OpenReview for camera-ready version

\theoremstyle{plain}
\newtheorem{theorem}{Theorem}[section]

\newtheorem{lemma}[theorem]{Lemma}
\newtheorem{corollary}[theorem]{Corollary}
\theoremstyle{definition}

\newtheorem{assumption}[theorem]{Assumption}
\theoremstyle{remark}
\newtheorem{remark}[theorem]{Remark}

\newcommand{\dN}{\mathcal{N}}
\newcommand{\dG}{\mathcal{IG}}

\newcommand{\diag}{\text{diag}}
\newcommand{\tr}{\text{tr}}
\newcommand{\E}{\mathbb{E}}

\newcommand{\kl}{\text{KL}}

\newcommand{\aux}{\text{aux}}

\DeclareMathOperator*{\argmin}{arg\,min}

\begin{document}

\maketitle

\begin{abstract}
Traditional neural networks are simple to train but they typically produce overconfident predictions. In contrast, Bayesian neural networks provide good uncertainty quantification but optimizing them is time consuming due to the large parameter space. 
This paper proposes to combine the advantages of both approaches by performing Variational Inference in the Final layer Output space (VIFO), because the output space is much smaller than the parameter space. 
We use neural networks to learn the mean and the variance of the probabilistic output. 
Using the Bayesian formulation we incorporate collapsed variational inference with VIFO which significantly improves the performance in practice.
On the other hand, like standard, non-Bayesian models, VIFO enjoys simple training and one can use Rademacher complexity to provide risk bounds for the model.
Experiments show that VIFO %and ensembles of VIFO 
provides a good tradeoff in terms of run time and uncertainty quantification, especially for out of distribution data. 
\end{abstract}

\section{Introduction} 
With the development of training and representation methods for deep learning, models using neural networks provide excellent predictions. However, such models fall behind in terms of uncertainty quantification and their predictions are often overconfident \citep{temp-scaling}. 
Bayesian methods provide a methodology for uncertainty quantification by placing a prior over parameters and computing a posterior given observed data, but the computation required for such methods is often infeasible.  
Variational inference (VI) is one of the most popular approaches for approximating the Bayesian outcome, e.g., \citep{BNN, Practical-VI, DVI}. By minimizing the KL divergence between the variational distribution and the true posterior and constructing an evidence lower bound (ELBO), one can find the best approximation to the intractable posterior. However, when applied to deep learning, VI requires sampling 
%needs samples drawn from the variational distribution 
to compute the ELBO, and it suffers from both high computational cost and large variance in gradient estimation. \citet{DVI} have proposed a deterministic variational inference (DVI) approach to alleviate the latter problem. The idea relies on the central limit theorem, which implies that with sufficiently many hidden neurons, the distribution of the output of each layer forms a multivariate Gaussian distribution. Thus we only need to compute the mean and covariance of the output of each layer. However, DVI still suffers from high computational cost and complex optimization.
%demands heavy computation of the mean and covariance. 

Inspired by DVI, we observe that the only aspect that affects the prediction is the distribution of the output of the final layer in the neural network. We therefore propose to perform variational inference in the final-layer {\em output space} (rather than parameter space), where the posterior mean and diagonal variance are learned by a neural network. We call this method VIFO. 
% VIFO avoids calculating a distribution over neural network weights and saves intermediate computations. 
Like all Bayesian methods, VIFO
induces a distribution over its probabilistic predictions and has the advantage of uncertainty quantification in predictions. 
At the same time, VIFO has a single set of parameters and thus enjoys simple optimization as in non-Bayesian methods.

We can motivate VIFO from several theoretical perspectives.
First, we derive improved priors (or regularizers) for VIFO motivated by collapsed variational inference \citep{collapsed-elbo} and empirical bayes \citep{DVI}. The new regularizers greatly improve the performance of VIFO. 
Second, we show that, for the linear case, with expressive priors VIFO can capture the same predictions as standard variational inference. 
On the other hand, with practical priors and deep networks VIFO exhibits limited expressiveness. 
We propose to overcome this limitation by using ensembles that enable fast training and further improve uncertainty quantification.
% With practical priors and deep networks the model can be seen as a simpler alternative. 
Third, due to its simplicity, one can derive risk bounds for the model through Rademacher complexity. 
Thus, VIFO was motivated as an effective simplification of VI and DVI, but the ensembles of VIFO can be seen as a Bayesian extension of Deep Ensembles \citep{Deep-ensembles}.
We discuss the connections to other Bayesian predictors below. 

%In experiments, we mainly compare 
An experimental evaluation compares
VIFO with VI and 
other state of the art approximation methods
and with non-Bayesian neural networks. 
%(which we refer to as {\em base models}). 
The results show that (1) VIFO is much faster than VI and only slightly slower than base models, and (2) 
ensembles of VIFO achieve better uncertainty quantification on shifted and out-of-distribution data 
while preserving the quality of in-distribution predictions.
Overall, VIFO provides a good tradeoff in terms of run time and uncertainty quantification especially for out-of-distribution data.

\section{VIFO}
\label{sec:vifo}
In this section we describe our VIFO method in detail.
We start with a description of the non-Bayesian {\em base model}.
Given a neural network parametrized by weights $W$ and input $x$, the output is $z=f_W(x) \in \mathbb{R}^K$. 
The base model provides probabilistic predictions by combining the output of the network with any prediction likelihood $p(y|z)$. Traditional, non-Bayesian models, minimize $-\log p(y|z)$ or a regularized variant.

%As in other models, the same methodology can be used for any type of prediction likelihood $p(y|z)$. 
%This forms the base model. 
%Traditional, non-Bayesian models, minimize $-\log p(y|z)$ or a regularized variant. 
%We refer to these models as base models.

\begin{remark} 
The base model and VIFO are applicable with any likelihood function and our development of VIFO below is general. 
To illustrate we discuss classification and regression. 
In classification, $K$ is the number of classes. The probability of being class $i$ is defined as
\begin{align}
    p(y=i|z)=\text{softmax}(z)_i=\frac{\exp z_i}{\sum_j \exp z_j}.
    \label{eq:classification-predictor}
\end{align}
In regression, $z=(m, l)$ is a 2-dimensional vector and $K=2$. We apply a function $g$ on $l$ that maps $l$ to a positive real number. The probability of the output $y$ is:
\begin{align}
    p(y|z) = \dN(y|m, g(l)) = \frac{1}{\sqrt{2\pi g(l)}} \exp\left(-\frac{(y-m)^2}{2 g(l)} \right).
    \label{eq:regression-predictor}
\end{align}
\end{remark}

By fixing the weights $W$, base models map $x$ to $z$ deterministically.
%while Bayesian methods seek to map $x$ to a distribution over $z$. 
Bayesian inference puts a distribution over $W$ and marginalizes out to get a distribution over $z$ from which predictions can be calculated. 
Since exact marginalization is not tractable, variational inference provides an approximation which yields the well known ELBO objective for optimization:
%The standard Bayesian approach puts a prior on the weights $W$ and derives the standard ELBO:
\begin{align}
    \log p(\mathcal D) \geq \E_{q(W)} \left[\log \frac{p(W, \mathcal D)}{q(W)} \right]
    = \sum_{(x, y)\in \mathcal D}\E_{q(W)} [\log p(y|W, x)] - \kl(q(W) \Vert p(W)).
\end{align} 

As shown by \citet{DVI}, by the central limit theorem, with a sufficiently wide neural network the marginal distribution of $z$ is Gaussian.
DVI explicitly calculates an analytic approximation of the mean and variance of the output of each layer (valid for specific activation functions) and avoids the sampling typically used for optimization of the ELBO in other methods. 

VIFO pursues this in a direct manner. It has two sets of weights, $W_1$ and $W_2$ (with potentially shared components), to model the mean and variance of $z$. That is, $\mu_q(x)=f_{W_1}(x)$, $\sigma_q(x) = g(f_{W_2}(x))$, where $g:\mathbb{R} \rightarrow \mathbb{R}^{+}$ maps the output to positive real numbers as the variance is positive. Thus, $q(z|x)=\dN(z|\mu_q(x), \diag(\sigma^2_q(x)))$, where $\mu_q(x), \sigma_q^2(x)$ are vectors of the corresponding dimension. We will call $q(z|x)$ the %\emph{variational predictive distribution}. 
\emph{variational output distribution}. 
As in the base model, 
given $z$, $y$ is generated from the likelihood $p(y|z)$.

\begin{remark}
    VIFO in regression is different from the existing models known as the mean-variance estimator \citep{mve, mve2, uq_cv}. Instead, mean-variance estimators are the base models that VIFO can be applied on. 
    Applying VIFO to these models results in four outputs: $\mu_m$ and $\mu_l$, which are the means of $m$ and $l$, and $\sigma^2_m$ and $\sigma^2_l$, which are the variances of $m$ and $l$. These variances come from the variational output distribution. We sample $m \sim \dN(\mu_m, \sigma_m^2)$ and $l \sim \dN(\mu_l, \sigma^2_l)$, then form $z = (m, l)$. Like all Bayesian methods VIFO computes a distribution over distributions which is lacking in non-Bayesian predictions. 
\end{remark}
% Note that in regression, the prediction over $y$ given by $\int q(z|x)p(y|z) dz$ is different from 

Unlike VI which puts a prior over $W$, VIFO models the distribution over $z$ and therefore we put a prior directly over $z$. We consider two options, a conditional prior $p(z|x)$ and a simpler prior $p(z)$.
Both of these choices yield a valid ELBO using the same steps:
\begin{align}
\label{eq:naive-elbo}
    \log p(y|x) \geq \E_{q(z|x)} \left[\log \frac{p(y,z|x)}{q(z|x)} \right]
    = \E_{q(z|x)} [\log p(y|z)] - \kl(q(z|x) \Vert p(z|x)).
\end{align} 
The approach has some similarity to
Dirichlet-based models \citep{dirichlet, dir-flow, pitfalls}.
%also have a similar objective.
However, we perform inference on the output whereas, as discussed by \citet{pitfalls}, these models implicitly perform variational inference on the prediction. 
In particular, in that work 
$z$ is interpreted as a vector in the simplex and $q(z|x)$ and 
$p(z)$ are Dirichlet distributions, 
whereas when using VIFO for classification $z$ has a Gaussian distribution and $p(y|z)$ is on the simplex.
In other words,
we model and regularize different distributions.
We discuss related work in more details below.

%On the other hand, our method can be seen to provide a connection between Dirichlet-based models and traditional variational inference in parameter space which are otherwise hard to relate directly. 
%We discuss Dirichlet-based models further in related work and in Appendix.
%We will discuss Dirichlet-based models further in Appendix \ref{sec:dir}. 

\cref{eq:naive-elbo} is defined for every $(x, y)$. For a dataset $\mathcal D=\{(x, y)\}$, we optimize $W_1$ and $W_2$ such that
\begin{align*}
    \sum_{(x,y)\in \mathcal D} \Big\{ \E_{q(z|x)}[\log p(y|z)] - \kl(q(z|x) \Vert p(z|x)) \Big\}
\end{align*}
is maximized.
We regard the negation of the first term $\E_{q(z|x)}[-\log p(y|z)]$ as the loss term and treat $\kl(q(z|x) \Vert p(z|x))$ as a regularizer. 

\begin{figure}[t]
    \centering
    \begin{subfigure}[b]{0.35\textwidth}
         \centering
         \includegraphics[width=\textwidth]{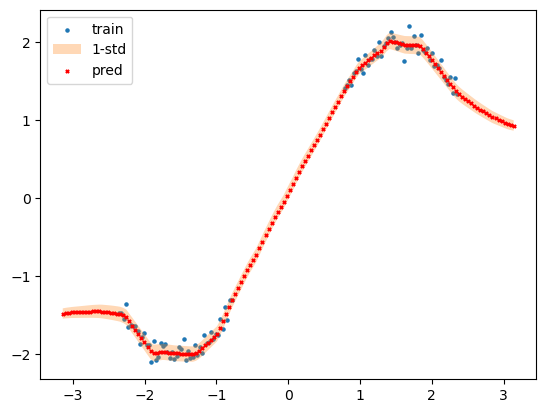}
         \caption{without auxiliary training}
    \end{subfigure}
    \begin{subfigure}[b]{0.35\textwidth}
         \centering
         \includegraphics[width=\textwidth]{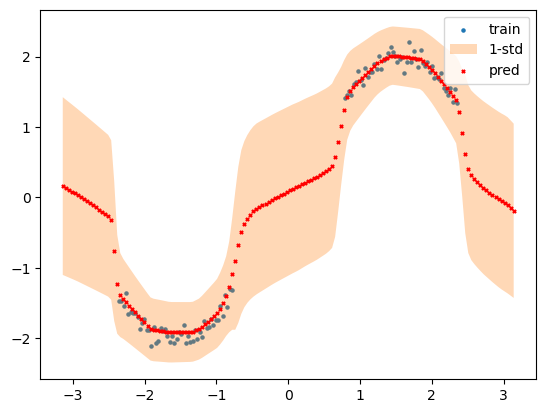}
         \caption{with auxiliary training}
    \end{subfigure}
    \caption{Predictive distribution of VIFO using an MLP. Blue points are training data generated from a sinusoidal function, red points are the predicted mean, shaded area indicates the 1 standard deviation. More details are in Appendix \ref{sec:detail-artificial}.}
    \label{fig:aux}
\end{figure}

\subsection{Auxiliary Training}
As in prior work \citep{functionalVI}, 
to improve uncertainty quantification we introduce auxiliary input $x_\aux$ and include $\kl(q(z|x_\aux) \Vert p(z|x_\aux))$ as an additional regularization term.
We include corresponding coefficients $\eta$ and $\eta_\aux$ on the regularizers, as is often done in variational approximations 
(e.g., \citep{betavae,Sheth2017,jankowiak2019sparse,HMC, dlm-sgp, dlm-bnn}.). 
Then, viewed as a regularized loss minimization, the optimization objective for VIFO becomes:
\begin{align}
    \min_{W_1, W_2} \sum_{(x, y)\in \mathcal D} \Big\{ \E_{q(z|x)}[-\log p(y|z)] + \eta \kl(q(z|x) \Vert p(z|x)) + \eta_\aux \sum_{x_\aux} \kl (q(z|x_\aux) \Vert p(z | x_\aux))\Big\}.
    \label{eq:regularized-loss}
\end{align}
Generally the loss term is intractable, so we use Monte Carlo samples to approximate it.
In practice, since auxiliary data is not available, 
we uniformly sample $x_\aux^{(i)} \sim \text{Unif} [x_{\min}^{(i)} - \frac{d}{2}, x_{\max}^{(i)} + \frac{d}{2}]$ where $d=x_{\max}^{(i)} - x_{\min}^{(i)}$ for each entry $i$. Figure \ref{fig:aux} shows an example where an MLP is used to learn a complex function over 1 dimensional input space, illustrating that such regularization can improve uncertainty quantification in the area where the data is missing. 

\subsection{Collapsed VIFO}
\label{sec:cvi}
Bayesian methods are often sensitive to the choice of prior parameters. To overcome this, \citet{DVI} used empirical Bayes (EB) to select the value of the prior parameters, and \citet{collapsed-elbo} proposed collapsed variational inference, which defined a hierarchical model and performed inference on the prior parameters as well. 
We show how these ideas are applicable in VIFO and derive empirical Bayes as a special case of collapsed variational inference.
In addition to $z$, we model the prior mean $\mu_p$ and variance $\sigma_p^2$ as Bayesian parameters. Now the prior becomes $p(z|\mu_p, \sigma_p^2) p(\mu_p, \sigma_p^2)$ and the variational distribution is $q(z|x) q(\mu_p, \sigma_p^2)$.
% Consider the prior $p(z)=\dN(z|\mu_p, \sigma_p^2)$ and now $\mu_p$ and $\sigma_p^2$ are also variational parameters, with prior $p(\mu_p, \sigma_p^2)$. Let $q(\mu_p, \sigma_p^2)$ and $p(\mu_p, \sigma_p^2)$ are the variational and prior distribution for $\mu_p$ and $\sigma_p^2$. 
Then the objective becomes:
\begin{align}
\label{eq:collapsed-elbo}
    \log p(y|x) &\geq \E_{q(z|x) q(\mu_p, \sigma_p^2)} \left[\log \frac{p(y, z, \mu_p, \sigma_p^2|x)}{q(z|x) q(\mu_p, \sigma_p^2)} \right] \nonumber \\
    &=\E_{q(z|x)}[\log p(y|z)] - \E_{q(\mu_p, \sigma_p^2)}[\kl(q(z|x)\Vert p(z|\mu_p, \sigma_p^2))] -\kl(q(\mu_p, \sigma_p^2) \Vert p(\mu_p, \sigma_p^2)) .
\end{align}
Similar to \cref{eq:regularized-loss}, we treat the first term as a loss and the other two terms as a regularizer along with a coefficient $\eta$ and aggregate over all data.
% \begin{align}
%     \sum_{(x, y)\in \mathcal D} \Bigg\{ E_{q(z|x)}[-\log p(y|z)] + \eta \left( \E_{q(\mu_p, \sigma_p^2)}[\kl(q(z|x)\Vert p(z|\mu_p, \sigma_p^2))] + \kl(q(\mu_p, \sigma_p^2)||p(\mu_p, \sigma_p^2)) \right) \Bigg\}
%     \label{eq:eta-collapsed-vi}
% \end{align}
Since the loss does not contain $\mu_p$ and $\sigma_p^2$, we can get the optimal $q^*(\mu_p, \sigma_p^2)$ by optimizing the regularizer and the choice of $\eta$ will not affect $q^*(\mu_p, \sigma_p^2)$. 
% We have two potential directions for finding the optimal $q(\mu_p, \sigma_p^2)$,
% \begin{itemize}
%     \item Optimize $q(\mu_p, \sigma_p^2 | x)$ for every data point $x$:
%     \begin{align}
%         q^*(\mu_p, \sigma_p^2 | x) =\argmin_{q(\mu_p, \sigma_p^2)} \{\E_{q(\mu_p, \sigma_p^2)}[\kl(q(z|x)\Vert p(z|\mu_p, \sigma_p^2))] + \kl(q(\mu_p, \sigma_p^2)||p(\mu_p, \sigma_p^2))\};
%     \end{align}
%
%     \item Optimize $q(\mu_p, \sigma_p^2)$ for all data:
%     \begin{align}
%         q^*(\mu_p, \sigma_p^2) =\argmin_{q(\mu_p, \sigma_p^2)} \sum_{(x,y)\in \mathcal D}\{\E_{q(\mu_p, \sigma_p^2)}[\kl(q(z|x)\Vert p(z|\mu_p, \sigma_p^2))] + \kl(q(\mu_p, \sigma_p^2)||p(\mu_p, \sigma_p^2))\}.
%     \end{align}
% \end{itemize}
Then we can plug in the value of $q^*$ into \cref{eq:collapsed-elbo}. We next show how to compute $q^*(\mu_p, \sigma_p^2)$ and the final collapsed variational inference objective. The derivations are similar to the ones by \citet{collapsed-elbo} but they are applied on $z$ not on $W$. 
%In practice, optimizing $q(\mu_p, \sigma_p^2)$ for all data does not perform very well. So we put the derivations of such cases in appendix.
% In the following, we use $K$ to denote the dimension of $z$. 
Recall that $K$ is the dimension of $z$.

\paragraph{Learn mean, fix variance}
%\subsection{Learn mean, fix variance}
Let $p(z|\mu_p)=\dN(z|\mu_p, \gamma I)$, $p(\mu_p)=\dN(\mu_p|0, \alpha I)$. Then $q^*(\mu_p|x)$ is
%the minimizer of 
\begin{align*}
    \argmin_{q(\mu_p)} 
    \E_{q(\mu_p)}[\kl(q(z|x) \Vert p(z|\mu_p))] + \kl(q(\mu_p) \Vert p(\mu_p)),
\end{align*}
and the optimal $q^*(\mu_p|x)$ can be computed as:
\begin{align*}
    \log q^*(\mu_p | x) \propto -\frac{(\mu_q(x) - \mu_p)^\top (\mu_q(x) - \mu_p)}{2 \gamma} - \frac{\mu_p^\top \mu_p}{2\alpha},
\end{align*}
and $q^*(\mu_p |x) = \dN(\mu_p| \frac{\alpha}{\alpha + \gamma} \mu_q(x), \frac{\alpha \gamma}{\alpha + \gamma})$. Notice that, unlike the prior, $q^*(\mu_p)$ depends on $x$. If we put $q^*$ back in the regularizer of \cref{eq:collapsed-elbo}, the regularizer becomes:
\begin{align}
\label{eq:collapsed-mean}
    \frac{1}{2\gamma} \left[1^\top \sigma^2_q(x) + \frac{\gamma}{\gamma + \alpha} \mu_q(x)^\top \mu_q(x) \right] - \frac{1}{2} 1^\top \log \sigma_q^2(x) 
    +\frac{K}{2} \log (\gamma+\alpha) - \frac{K}{2}.
\end{align}
As in \cite{collapsed-elbo}, \cref{eq:collapsed-mean} puts a factor $\frac{\gamma}{\gamma+\alpha} < 1$ in front of $\mu_q(x)^\top \mu_q(x)$, which weakens the regularization on $\mu_q(x)$. 
% This allows more space to adjust the mean. 
We refer to this method as ``VIFO-mean''. 

\cref{fig:prior}, shows the learned prior for VIFO-mean and VI for the same example as in Figure \ref{fig:aux}. We observe that VIFO-mean
allows for diverse prior distributions and captures the data distribution well.

\begin{figure}[t]
    \centering
    \begin{subfigure}[b]{0.32\textwidth}
         \centering
         \includegraphics[width=\textwidth]{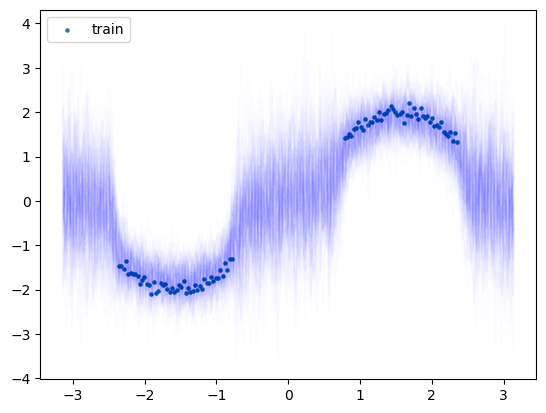}
         \caption{VIFO-mean}
    \end{subfigure}
    \begin{subfigure}[b]{0.32\textwidth}
         \centering
         \includegraphics[width=\textwidth]{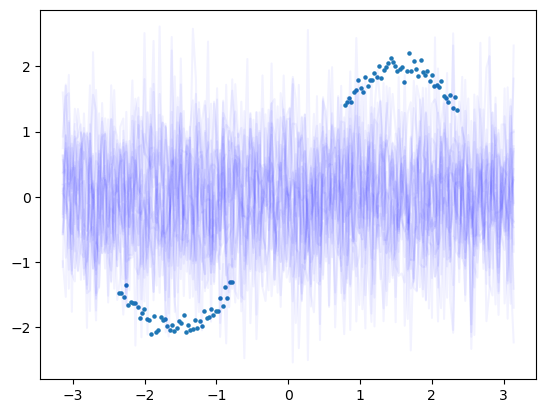}
         \caption{VI-naive}
    \end{subfigure}
    \begin{subfigure}[b]{0.32\textwidth}
         \centering
         \includegraphics[width=\textwidth]{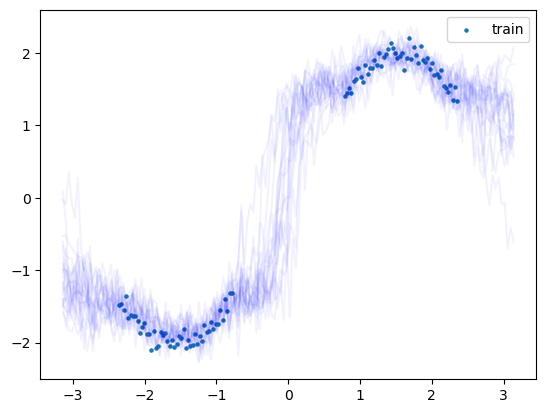}
         \caption{VI-mean}
    \end{subfigure}
    \caption{Induced predictions by learned prior distribution for different methods. 
    Note that VI has a prior over weights and VIFO has a prior over $z$. 
    For each method we sample values from the prior and calculate predictions $y$ based on the sampled values. We then plot the $y$ values. 
    As we can see, VI-naive induces a uniform prior that does not capture the data distribution, VI-mean has an increased variance in areas where data is missing and VIFO-mean does so to a larger extent. Details are given in Appendix \ref{sec:detail-artificial}.}
    \label{fig:prior}
\end{figure}

\paragraph{Other Regularizers}
The same approach can be used for a joint prior 
$p(z|\mu_p, \sigma_p^2)=\dN(z|\mu_p, \sigma_p^2)$, $p(\mu_p)=\dN(\mu_p|0, \frac{1}{t} \sigma_p^2)$, $p(\sigma_p^2)=\dG(\sigma_p^2|\alpha, \beta)$, where $\dG$ is inverse Gamma,
yielding a method we call ``VIFO-mv''.
Similarly, the hierarchical prior in empirical Bayes models the variance but not the mean
$p(\sigma_p^2)=\dG(\sigma_p^2|\alpha, \beta)$, $p(z|\sigma_p^2)=\dN(z|0, \sigma_p^2)$ and yields ``VIFO-eb''.
Derivations are given in Appendix \ref{sec:derive-cvi}. 
% Finally, 
% by analogy to MAP solutions we can also regularize the weights directly by replacing the regularizers with $\frac{\eta}{N} \lVert W \rVert_2$ in the objective. We include this in our experiments and call it ``VIFO-l2''.

\newcommand{\ignoretext}[1]{}
\ignoretext{
%%===> start ignore
\subsection{Learn both mean and variance}
Let $p(z|\mu_p, \sigma_p^2)=\dN(z|\mu_p, \sigma_p^2)$, $p(\mu_p)=\dN(\mu_p|0, \frac{1}{t} \sigma_p^2)$, $p(\sigma_p^2)=\dG(\sigma_p^2|\alpha, \beta)$, where $\dG$ indicates the inverse Gamma distribution. The posterior of $\mu_p$ is Gaussian and the posterior of $\sigma_p^2$ is inverse Gamma. In this case we can show that $q^*(\mu_p|x) = \dN(\mu_p| \frac{1}{t+1} \mu_q(x), \frac{1}{(1+t)} \sigma_p^2)$ and $q^*(\sigma_p^2 | x) = \dG(\sigma_p^2|(\alpha + \frac{1}{2})1, \beta+\frac{t}{2(t+1)} \mu_q(x)^2 + \frac{1}{2}\sigma_q^2(x))$ and the regularizer becomes
\begin{align}
    &(\alpha+\frac{1}{2}) 1^\top \log \left[\beta 1 + \frac{t}{2(1+t)} \mu_q(x)^2 + \frac{1}{2}\sigma_q^2(x) \right] - \frac{1}{2} 1^\top \log \sigma_q^2(x).
\end{align}
We refer to this method as ``VIFO-mv''.

\subsection{Empirical Bayes}
Let $p(\sigma_p^2)=\dG(\sigma_p^2|\alpha, \beta)$, $p(z|\sigma_p^2)=\dN(z|0, \sigma_p^2)$ and let $q(\sigma_p^2)$ be a delta distribution $\delta(s^*(x))$. Then the regularizer of \cref{eq:collapsed-elbo} becomes:
\begin{align}
\label{eq:eb}
    \kl(q(z|x)||p(z|\sigma^2_p)) - \log p(\sigma_p^2).
\end{align}
By minimizing this objective we obtain the optimal value $s^*(x)=\frac{\mu_q(x)^\top \mu_q(x) + 1^\top\sigma_q^2(x) + 2\beta}{K+2\alpha+2}$. Plugging this value into the KL term 
we obtain:
%and omitting $\log p(\sigma_p^2)$ as in \cite{DVI}, we have the regularizer:
\begin{align}
    &\frac{1}{2} \left[K\log \frac{\mu_q(x)^\top \mu_q(x) + 1^\top\sigma_q^2(x) + 2\beta}{K + 2\alpha + 2} - 1^\top \log |\sigma_q^2(x)| \right] \nonumber \\
    &- \frac{K}{2}+ \frac{1}{2} \frac{(K+2\alpha+2)(\mu_q(x)^\top \mu_q(x) + 1^\top\sigma_q^2(x))}{\mu_q(x)^\top \mu_q(x) + 1^\top\sigma_q^2(x) + 2\beta} .
    \label{eq:eb-regularizer}
\end{align}
Following \cite{DVI},
as shown in appendix \ref{eb}, including the prior term in the regularizer reduces its complexity and this harms performance in practice. 
Hence, we use Eq~(\ref{eq:eb-regularizer}) as the regularizer. 
We call this method ``VIFO-eb''.

% On the other hand, we can also get the optimal $\sigma_p^2$ by optimizing the sum of \cref{eq:eb} for the overall dataset $\mathcal D$ and then we will get the optimal $\sigma_p^2=\frac{\Tilde{\mu}_q^\top \Tilde{\mu}_q + 1^\top \Tilde{\sigma}_p^2 + 2\beta }{D + 2\alpha + 2 }$ where $\Tilde{\mu}_q = \sqrt{\frac{1}{N} \sum_x \mu_q(x)^2}$ and $\Tilde{\sigma}_q = \sqrt{\frac{1}{N} \sum_x \sigma_q(x)^2}$. The objective is:
% \begin{align}
%     &\frac{ND}{2} \log \frac{2\beta + \Tilde{\mu}_q^2 + \Tilde{\sigma}_q^2}{D + 2\alpha + 2} - \frac{1}{2} \sum_x 1^\top \log \sigma_q(x)^2 - \frac{ND}{2} \\
%     +& \frac{1}{2} \frac{D+2\alpha+2}{2\beta + \Tilde{\mu}_q^2 + \Tilde{\sigma}_q^2}\sum_x (\mu_q(x)^\top \mu_q(x) + 1^\top \sigma_q^2(x)).
% \end{align}

\subsection{L2 regularization}
Notice that all regularizations above are regularizing the variational predictive distribution $q(z|x)$, not the parameters of the neural network. By analogy to MAP solutions of Bayesian neural networks, we can also regularize the weights directly by replacing the regularizers with $\frac{\eta}{N} \lVert W \rVert_2$ in the objective, even though our posterior is on $z$. We include this in our experiments and call it ``VIFO-l2''.

%%===> end ignore
}

\section{Expressiveness of VIFO}
\label{sec:VI-VIFO}
VIFO is inspired by DVI and it highly reduces the computational cost. 
In this section we explore whether VIFO can produce exactly the same predictive distribution as VI. We show that this is the case for linear models but that for deep models VIFO is less powerful. 
%In the simple Bayesian linear models with a carefully selected prior, we can see that the objective of VIFO is exactly the same as VI, and thus the variational predictive distribution of VIFO is the same as the marginalized distribution of VI. 
We first introduce the setting of linear models. Let the parameter be $\theta$, then the model is:
\begin{align}
    y|x, \theta \sim p(y|\theta^\top x).
\end{align}
For example, $p(y|\theta^\top x) = \dN(y|\theta^\top x, \frac{1}{\beta})$ where $\beta$ is a constant for Bayesian linear regression; and $p(y=1|\theta^\top x) = \frac{1}{1+\exp(-\theta^\top x)}$ for Bayesian binary classification.

For simplicity, we assume $\theta \in \mathbb{R}^d$, where $d$ is the dimension of $x$, and then the output dimension $K=1$. The standard approach specifies the prior of $\theta$ to be $p(\theta) = \dN(\theta|m_0, S_0)$, and uses $q(\theta)=\dN(\theta|m, S)$. Then the ELBO objective, with a dataset $X_N=(x_1, x_2, \dots, x_N)\in \mathbb{R}^{d\times N}$ and $Y_N=(y_1, y_2, \dots, y_N)\in \mathbb{R}^N$, is
\begin{align}
    &\sum_{i=1}^N \E_{q(\theta)}[\log p(y_i|\theta^\top x_i)] - \kl(q(\theta) \Vert p(\theta)) \nonumber \\
    =&\sum_{i=1}^N \E_{q(\theta)}[\log p(y_i|\theta^\top x_i)] - \frac{1}{2}[\tr(S_0^{-1} S) - \log |S_0^{-1} S|] 
    - \frac{1}{2} (m-m_0)^\top S_0^{-1} (m-m_0) + \frac{d}{2}.
    \label{eq:VI-blr}
\end{align}
As the following theorem shows, if we use a conditional correlated prior and a variational posterior that correlates data points, then in the linear case VIFO can recover the ELBO and VI solution. We defer the proof and discussion of $K>1$ to Appendix \ref{sec:proof-VI-VIFO}.

\begin{theorem}
\label{thm:linear}
    Let $q(z|x)=\dN(z|w^\top x, x^\top V x)$ be the variational predictive distribution of VIFO, where $w$ and $V$ are the parameters to be optimized, and
    let $p(z|X_N)=\dN(z|m_0^\top X_N, X_N^\top S_0 X_N)$ and $q(z|X_N)=\dN(z|w^\top X_N, X_N^\top V X_N)$ be a correlated and data-specific prior and posterior (which means that for different data $x$, we have a different prior/posterior over $z$). Then the VIFO objective is equivalent to the ELBO objective 
    %thus yielding identical predictive distributions.
    %and thus, the predictive distributions are the same.
    implying identical predictive distributions.
\end{theorem}

However, as the next theorem shows, for the non-linear case %even if we ignore the prior we may not be able to recover the loss term exactly.
%in some extreme cases. 
we cannot produce the variational output distribution $q(z|x)$ as if it is marginalized over the posterior on $W$.

%in general we cannot close the gap between VI and VIFO. First, for neural networks, computing a correlated and data-specific prior $p(z|x)$ is as complicated as computing the marginal distribution in variational inference and we prefer to use a simple uniform prior $p(z)$ independent of $x$. This sacrifices the performance, but encourages uncertainty calibration. Second, we do not learn a full covariance matrix. These two points indicate that we cannot make the regularizer of VIFO equal to the regularizer of VI.
%However, as the next theorem shows, even if we ignore the regularizer, we may not be able to recover the loss term exactly in some extreme cases. 

% One insight from eq\cref{eq:uq-simplified} is that if we use homoscedastic diagonal covariance matrix $S_0 = \sigma_0^2 I$ and zero mean, then eq\cref{eq:uq-simplified} becomes:
% \begin{align}
%     -\sum_i \frac{1}{2\beta} \{(y_i - \theta_1^\top x_i)^2 + x_i^\top V_2 x_i \} - \frac{1}{2 \sigma_0^2} \tr(V_2)+ \frac{1}{2} \log |V_2| - \frac{1}{2\sigma_0^2} \lVert \theta_1 \rVert_2^2
% \end{align}

% However, it cannot fully recover the marginal distribution of $z$ in some extreme cases. We will address this issue and show that ensembles of VIFO method will help alleviate it. 

\begin{theorem}
\label{thm:not-recover-vi}
    Given a neural network $f_W$ parametrized by $W$ and a mean-field Gaussian distribution $q(W)$ over $W$, there may not exist a set of parameters $\Tilde{W}$ such that for all input $x$ we have $\E_{q(W)}[f_W(x)]=f_{\Tilde{W}}(x)$.
\end{theorem}
The proof is given in Appendix \ref{sec:proof-VI-VIFO}.
The significance of these results is twofold.  
On the one hand, we see from \Cref{thm:not-recover-vi} and the conditions of \Cref{thm:linear} that the representation is more limited, i.e., efficiency comes at some cost. 
On the other hand, \Cref{thm:linear} shows the connection of VIFO to VI, which gives a better perspective on the approximation it provides. Moreover, this facilitates the use of existing improvements in VI for VIFO such as collaposed VI applied to VIFO.

In practice, a correlated and data-specific prior $p(z|x)$ is complex, and tuning its hyperparameters would be challenging. Hence, for a practical algorithm we propose to use a simple prior $p(z)$ independent of $x$. 
In addition, to reduce computational complexity, we do not learn a full covariance matrix and focus on the diagonal approximation. 
%These further limit expressive power but enable fast training. 
%As shown in our experiments, while this limits expressive power, it enables fast training of ensembles of VIFO that provide strong performance. 
These aspects limit expressive power but enable fast training of VIFO and hence also ensembles of VIFO.
%that provide strong performance.

%\section{Theoretical Analysis through Rademacher Complexity}
\section{Rademacher Complexity of VIFO}
\label{sec:rademacher}
In this section we provide generalization bounds for VIFO through Rademacher Complexity. We need to make the following assumptions. These assumptions hold for classification and with a smoothed loss for regression as shown in \cref{app:rademacher}.
\begin{assumption}
\label{assumption:loss}
$\log p(y|z)$ is $L_0$-Lipschitz in $z$, i.e., $|\log p(y|z) -\log p(y|z')| \leq L_0 \lVert z - z' \rVert_2$.
\end{assumption}
\begin{assumption}
\label{assumption:link}
    The link function $g$ is $L_1$-Lipschitz.
\end{assumption}

Recall that the Rademacher complexity of a set of vectors $A\subseteq \mathbb{R}^N$ is defined as 
$R(A)=$ $\frac{1}{N} E_{\sigma\sim\{-1,1\}^N}[\sup_{a\in A}\sum_i \sigma_i a_i]$. 
The Rademacher complexity of the set of loss values induced by functions ${f\in \cal F}$ over a dataset $S$ has been used to derive generalization bounds for learning of the class ${\cal F}$.
We need the following technical lemma, proved in Appendix~\ref{app:rademacher}, that generalizes well known Lipschitz based bounds \citep{SSBD2014} to multi-input functions. 

\begin{lemma}
\label{lemma:rademacher-bivariate}
Consider an $L$-Lipschitz function $\phi:\mathbb{R} \times \mathbb{R} \rightarrow \mathbb{R}$, i.e. $\phi(a_1, b_1)-\phi(a_2, b_2) \leq L (|a_1-a_2| + |b_1-b_2|)$. For $\bm{a}, \bm{b} \in \mathbb{R}^N$, let $\phi(\bm{a}, \bm{b})$ denote the vector $(\phi(a_1, b_1), \dots, \phi(a_N, b_N))$. Let $\phi(A \times B)$ denote $\{\phi(\bm{a}, \bm{b}): \bm{a}\in A, \bm{b}\in B\}$, then 
\begin{align}
    R(\phi(A \times B)) \leq L (R(A) + R(B)).
\end{align}
\end{lemma}

Applying the previous lemma sequentially over multiple dimensions we obtain:

\begin{corollary}
\label{cor:multi-lipschitz}
Consider an $L$-Lipschitz function $\phi:\mathbb{R}^d \rightarrow \mathbb{R}$, i.e., for any $x, x' \in \mathbb{R}^d$, $\phi(x)-\phi(x') \leq L \lVert x - x' \rVert_1$. 
Let $\phi(A^d)=\{\phi(a_{1:d, i}):\bm{a}_1, \bm{a}_2, \dots, \bm{a}_d \in A \subset \mathbb{R}^N\}$,
then $R(\phi(A^d)) \leq L d R(A)$.
\end{corollary}

With the assumptions and technical lemma, we derive the main result:
\begin{theorem}
\label{thm:rademacher}
    Let $\mathcal{H}$ be the set of functions that can be represented with neural networks with parameter space $\mathcal{W}$,
    $\mathcal{H} = \{f_W(\cdot) | W \in \mathcal{W}\}$.
    % VIFO networks,
    VIFO has two components, so the VIFO hypothesis class is $\mathcal{H} \times \mathcal{H}$ $=$ $\{ (f_{W_1}(\cdot),f_{W_2}(\cdot))$  $| W=(W_1,W_2), W_1, W_2 \in \mathcal{W}\}$.  
    Let $l$ be the loss function for VIFO, $l(W,(x,y))=E_{q_W(z|x)}[-\log p(y|z)]$.
    Then the Rademacher complexity of VIFO is bounded as $ R(l \circ (\mathcal{H} \times \mathcal{H}) \circ S) \leq 2 (L_0 \max\{1, L_1\} K) \cdot R(\mathcal{H} \circ S)$, where $K$ is the dimension of $z$ and $S$ is training dataset.
\end{theorem}

The proof is in \cref{app:rademacher} and it shows how reparamertization can facilitate computation of Rademacher bounds for Bayesian predictors. 
The Rademacher complexity for VIFO is bounded through the Rademacher complexity of deterministic neural networks.
This shows one advantage of VIFO which is more amenable to analysis than standard VI due to its simplicity. 
Risk bounds for VI have been recently developed (e.g., \citep{German2016,Sheth2017}) but they require different proof techniques. 
The Rademacher complexity for neural networks is $O\left(\frac{B_W B_x}{\sqrt{N}} \right)$ \citep{nn-rademacher}, where $B_W$ bounds the norm of the weights and $B_x$ bounds the input. The Rademacher complexity of VIFO is of the same order. 
%Compared to variational inference, VIFO is more straightforward to analyze.

% \textbf{Should we say PAC-Bayes for VI is also order $1/\sqrt{N}$ as in Germain 2016?}

\section{Related Work}
VIFO is related to but distinct from a number of variational and one-pass methods.
Dirichlet-based methods \citep{dirichlet, dir-flow, pitfalls}, discussed above, 
implicitly perform variational inference on the prediction and the network output provides parameters of a Dirichlet distribution.
%parameters of a in functional space with the variational distribution and the prior distribution being Dirichlet distributions. 
Like VIFO they provide Bayesian predictions in a single pass over the network, but
%These methods are motivated by uncertainty quantification and 
their relation to the standard variational inference in parameter space is non obvious. 
On the other hand, 
VIFO is a single pass method clearly related to VI in parameter space which enables the benefits of collapsed variational inference. Thus VIFO can be seen to bridge between Dirichlet methods and VI.
DUQ \citep{AmersfoortSTG20} provides an alternative approach using one pass on the network. It first embeds examples into a latent space, similar to $z$, but computes classification prediction and uncertainty quantification through RBF distances to centroids of classes in that space. Hence its predictions are very different.
Another related line of work 
%is functional variational inference 
\citep{functionalVI, functionalPrior} performs variational inference in function space. However, they focus on choosing a better prior in weight space which is induced from Gaussian Process priors on function space, 
whereas
%and then inducing a flexible functional prior, e.g., a GP prior, while 
VIFO 
directly induces a simple prior on function space. 
\citet{no-fully-stochastic} model the distribution of the last layer by adding random noise as input and do not give an explicit form of the output distribution.

VIFO differs from other existing variational inference methods as well. 
%VIFO is faster.
The local reparametrization trick \citep{lrt, variational_nn} 
reduces the variance from sampling in VI. This is done by performing two forward passes with the mean and variance at each layer before sampling the output for the layer.
Hence this modifies the sampling process of VI whereas VIFO only requires one pass on the network and samples only the output of the last layer for prediction.

Last-layer variational inference \citep{lastLayer,Kristiadi0H20,DaxbergerKIEBH21,LiuLPTBL20,HarrisonWS24} performs variational inference on the {\em parameters} of the last layer, while we perform variational inference on the {\em output} of the last layer. 
Note that the last layer usually contains more parameters than the output which has constant size. 
Last Layer Laplace \cite{Kristiadi0H20,DaxbergerKIEBH21} reduces training complexity by first estimating the MAP solution, and then estimating the covariance of the parameters in one pass. SNGP \citep{LiuLPTBL20} is a variant of this method, that aims to mimic the sensitivity of Gaussian processes to distances among examples, by incorporating Fourier features at the one to last layer. 
Finally, VBLL \citep{HarrisonWS24} still maintains a distribution on last layer parameters, but approximates the expectation over these parameters in closed form (for specific likelihoods) to reduce the complexity of training and prediction. 
Thus, all these methods are much closer to VI because they maintain distributions over weights whereas VIFO produces distributions in output space. % and optimizes in a different space of weights. 

VIFO shares some aspects with the model of \citet{uq_cv}, 
where both use neural networks to output the mean and covariance of the last layer.
%However, unlike their method VIFO represents a variational lower bound,  and there are differences in the loss function and approach for epistemic uncertainty which are elaborated further in Section \ref{sec:vifo}.
However \citet{uq_cv} use the cross entropy loss, $-\log \E_{q(z|x)} p(y|z)$
instead of our loss in \cref{eq:regularized-loss}, they use dropout for epistemic uncertainty, and their objective has no explicit regularization. Hence unlike VIFO their formulation does not correspond to a standard ELBO.  
It is also interesting to compare VIFO to the Deep Variational information bottleneck \citep{AlemiFD017}. 
The model is motivated from a different perspective but its final optimization objective, obtained after some approximations, is similar to our \cref{eq:naive-elbo}. In this sense the model is close to the VIFO-naive. However, in their formulation, $z$ is the output of a bottleneck layer which is not the final layer (because it is meant to constrain the information that flows to the final layer), and $p(z)$ which is the prior in our model is a posterior on the marginal posterior on $z$. Nonetheless, our development of rich priors through collapsed inference can help inform the choice of $p(z)$ in that model, which is typically taken to be a standard Normal.

%Besides variational inference, there are many other Bayesian techniques that capture uncertainty. 
%Since the exact posterior is intractable for deep neural networks, we need some approximation. 

Various alternative Bayesian techniques have been proposed.
One direction is to get samples from the true posterior, as in Markov chain Monte Carlo methods \cite{HMC, hmc_bayesian}.
%and another way is to find a tractable approximation that is closest to the true posterior, including 
Expectation propagation aims to minimize the reverse KL divergence to the true posterior
\citep{ep, ep2}. These Bayesian methods, including variational inference, often suffer from high computational cost and therefore hybrid methods were proposed. Stochastic weight averaging Gaussian \citep{SWAG} forms a Gaussian distribution over parameters from the stochastic gradient descent trajectory in the base model. 
Dropout \citep{Dropout} randomly sets weights 0 to capture uncertainty in the model. Deep ensembles \citep{Deep-ensembles} use ensembles of base models learned with random initialization and shuffling of data points and then average the predictions. These methods implicitly perform approximate inference. 
In addition to these methods, there are also non-Bayesian methods to calibrate overconfident predictions, for example, temperature scaling \citep{temp-scaling} introduces a temperature parameter to anneal the predictive distribution to avoid high confidence. 
VIFO strikes a balance between simplicity and modelling power to enable simple training and Bayesian uncertainty quantification.
%However, the majority of these techniques are Bayesian. 
%, which aim to compute the posterior of parameters given a prior distribution.
On the one hand, VIFO can be seen as a simplification of VI. On the other hand, it can be seen as an extension of the base model. From this perspective, the use of ensembles of VIFO, which extend the ensembles of \citet{Deep-ensembles}, are highly motivated as a practical algorithm. 
% As shown below, ensembles of VIFO are indeed very effective in practice.

\section{Experiments}
In this section, we compare the empirical performance of VIFO with VI and hybrid methods that use the base model. %related to stochastic gradient descent, 
%(see Appendix \ref{sec:dir}). 
In VIFO, $W_1$ and $W_2$ share all parameters except those in the last layer.
VI candidates include the VI algorithm (``VI-naive''\citep{BNN}) with fixed prior parameters, and other variations from collapsed variational inference \citep{collapsed-elbo} and empirical Bayes \citep{DVI}. 
Non-Bayesian and hybrid methods include the base model (``SGD'', because it uses stochastic gradient descent as optimizer),
%SGD-related methods include the non-Bayesian network optimized by stochastic gradient descent (base model), 
stochastic weight averaging (``SWA'', which uses the average of the SGD trajectory on the base model as the final weights) from \citet{swa} and SWA-Gaussian (``SWAG'', which uses the SGD trajectory to form a Gaussian distribution over the neural network weight space) from \citet{SWAG}.
%These methods except the base model are called hybrid methods because they do not have an explicit prior distribution over parameters but they can be interpreted as Bayesian methods.
We use ensembles of the base models which are known as deep ensembles \citep{Deep-ensembles}, 
and the ensembles of SWAG models, which are the multiSWAG model of \cite{multiswag2020},
both of which are considered strong baselines for uncertainty quantification \citep{uncertainty-benchmark}.
In addition to these methods, we include other approximate Bayesian algorithms for comparison. These include repulsive ensembles (``Repulsive'', \citep{re}), the Dirichlet-based model (``Dir'', \citep{dirichlet}), dropout (\citep{Dropout}), last layer Laplace with prior optimization (``Laplace'', \citep{DaxbergerKIEBH21}), and variational Bayes last layer (``VBLL'', \citep{HarrisonWS24}).
% The performance of each method is evaluated by the log loss on the test datasets and by uncertainty quantification measures on shifted and out-of-distribution datasets.
Our main goal is to show:
\begin{itemize}
    \item VIFO is much faster than VI and only slightly slower than base models;
    %\item VIFO and its ensembles achieve better uncertainty quantification on shifted and out-of-distribution data than most baselines except VI and the ensembles of VI;
    \item Ensembles of VIFO preserve the quality of in-distribution predictions;
    \item Ensembles of VIFO achieve better uncertainty quantification on shifted and out-of-distribution (OOD) data than all baselines.
\end{itemize}

% {\bf [[Should we add this?]]} In addition our experiments show that when controlling the strength of regularization ($\eta<1$) ensembles of VI, while being computationally expensive, provide excellent performance in practice. This model has not received much attention in the literature and our work highlight its potential. 

For our main experiments, we pick four large datasets, CIFAR10, CIFAR100, SVHN, STL10, together with two types of neural networks, AlexNet \citep{alexnet} and PreResNet20 \citep{preresnet}. The regularization parameter $\eta$ is fixed to 0.1 for both VIFO and VI, as this choice yields better performance compared with the standard choice $\eta=1$. 
%
% SWA and SWAG are trained along with the base model using the code of \citet{swa} and \citet{SWAG}. Other models are optimized with the Adam optimizer. 
Empirically we observe that using collapsed variational inference in VI does not improve the performance. This is because \citet{collapsed-elbo} used $\eta=1$ to obtain their results
whereas we use $\eta=0.1$ which provides a much stronger baseline. 
For auxiliary training, we experiment with $\eta_{\aux} \in \{0.0, 0.1, 0.5, 1.0\}$. Larger values of $\eta_{\aux}$ generally improve OOD data detection at the cost of increased in-distribution loss, and there is no generic optimal value of $\eta_{\aux}$. In our main paper, we present only the case where $\eta_{\aux}=0.1$ because it provides a balance between in-distribution and OOD performance, with performance of other choices of $\eta_{\aux}$ provided in the appendix.
In addition, VIFO-mean and VIFO-mv perform better than other variants of VIFO. Thus, we only list these variants in our main paper and provide full results for other variants for VIFO and VI in the appendix.
For each method we run 5 independent runs and report means and standard deviations in results. 
Complete details for the setup and hyperparameters are given in Appendix \ref{sec:detail-exp}. Our code is available on \url{https://github.com/weiyadi/VIFO}.
% We use the average of the predictions from these runs as the prediction of the corresponding ensemble, so that we have one ensemble run per method.
% In addition, we performed experiments with regression and classification datasets from the UCI repository \citep{uci}, using a single-layer neural network with 50 hidden units.
% Due to space constraints the main paper reports detailed results for AlexNet and discusses other results as needed. 
% Other experimental details and the full set of results is included in the Appendix. 
% Note that results for PreResNet20 yield similar conclusions to the ones from AlexNet, so they are not discussed separately in the text.

\subsection{Run Time}
\label{sec:run-time}
\begin{table}
    \centering
    \caption{Running time (seconds) for training 1 epoch with batch size 512, AlexNet}
    \label{tab:time-AlexNet}
    \begin{tabular}{ccccc}
    \hline
        dataset & CIFAR10 & CIFAR100 & SVHN & STL10 \\
        \hline
        size & 50000 & 50000 & 73257 & 500 \\
        \hline
        VI & $8.51 \pm 0.41$ & $8.27 \pm 0.40$ & $11.56 \pm 0.39$ & $1.75 \pm 0.41$ \\
        VIFO & $2.18 \pm 0.39$ & $2.17 \pm 0.43$ & $2.72 \pm 0.38$ & $1.16 \pm 0.40$ \\
        base & $1.97 \pm 0.41$ & $1.99 \pm 0.43$ & $2.46 \pm 0.40$ & $1.12 \pm 0.38$\\
        \hline
    \end{tabular}
\end{table}
Ignoring the data preprocessing time, we compare the run time of training 1 epoch of VI, VIFO and the base model. 
%Other methods such as dropout and Dirichlet method does not have sampling procedure and have similar running time as the base model. 
In Table \ref{tab:time-AlexNet} we show the mean and standard deviation of 10 runs of these methods. 
%VI, VIFO and SGD. 
Different regularizers do not affect run time, so we only show that of VI-naive for VI and VIFO-mean for VIFO. 
In addition, as shown in Table \ref{tab:time-AlexNet}, VIFO is much faster than VI and is slightly slower than the base model. As shown in \cref{fig:learning-curves}, VIFO converges faster than, or as fast as VI. Consequently, the training time until convergence for VIFO is shorter than for VI.

The differences in run time are dominated by sampling and forward passes in the network. Let $P$ denote the number of parameters in the \emph{base} model and thus each forward/backward pass takes $O(P)$ time. The time complexity for computing the loss for each output of the base model is $O(1)$.
The base model only needs 1 forward pass without sampling and thus the time complexity is $O(P)$. 
VIFO needs 1 forward pass and $M$ samples to compute the loss so the time complexity is $O(P+M)$. 
VI needs $M$ samples of the parameter space and $M$ forward passes, thus the time complexity is $O(PM + M)=O(PM)$. 
The same facts apply for predictions on test data, where the advantage can be important for real time applications.  

\subsection{Ensembles of VIFO}
\cref{thm:not-recover-vi} points out that the expressiveness of VIFO is limited. To overcome this, we use ensembles of VIFO, which independently train multiple VIFO models and average their predictions. \cref{sec:run-time} establishes fast training of VIFO, allowing us to train VIFO models simultaneously while still maintaining the running time advantage of VIFO. 
We investigate the impact of ensemble size on performance in \cref{sec:ensemble}. While increasing the ensemble size enhances performance, the improvement diminishes once the size exceeds 5. Therefore, we choose an ensemble size of 5.
\cref{tab:single-ens} shows that with ensembles, VIFO with auxiliary training achieves much better log loss than when using a single model. The same holds without auxiliary training. This indicates that ensembles of VIFO are much more expressive than a single VIFO. In the following experiments, we use ensembles of VIFO. \textbf{For a fair comparison in the remainder of the paper, we use ensembles for all methods} except for VI (which is time-consuming) and repulsive ensembles (which are themselves ensembles). 
 
\begin{table}
    \centering
    \caption{Test log loss ($\downarrow$) of single VIFO and ensembles of VIFO.}
    \label{tab:single-ens}
    \begin{tabular}{ccccc}
    \hline
         & \multicolumn{2}{c}{VIFO-mean} & \multicolumn{2}{c}{VIFO-mv} \\
        \hline
         & single & ensemble & single & ensemble \\
        \hline
        CIFAR10 & $0.527 \pm 0.015$ & $0.345 \pm 0.003$ & $0.626 \pm 0.010$ & $0.324 \pm 0.001$ \\
        CIFAR100 & $2.253 \pm 0.032$ & $1.688 \pm 0.006$ & $2.688 \pm 0.029$ & $1.725 \pm 0.003$ \\
        STL10 & $1.333 \pm 0.065$ & $1.055 \pm 0.008$ & $1.531 \pm 0.019$ & $1.123 \pm 0.008$\\
        SVHN & $0.509 \pm 0.029$ & $0.351 \pm 0.005$ & $0.520 \pm 0.027$ & $0.298 \pm 0.009$ \\
        \hline
    \end{tabular}
\end{table}

\subsection{In-distribution Performance}
In this section we use log loss and accuracy to measure the performance for in-distribution data. 

\cref{fig:log-preresnet} and \cref{fig:log-alexnet} in Appendix compare main methods in terms of log loss.
%Here we only use a single version of VI as we observe that using collapsed variational inference in VI does not improve the performance as shown by \citet{collapsed-elbo} (see Appendix for the performance of other VI variants). This is because we use $\eta=0.1$ that yields better performance while \citet{collapsed-elbo} use $\eta=1$. 
First, we observe that repulsive ensembles and the Dirichlet method have much worse log loss than all other methods and they tend to give underconfident predictions.
Second, we observe that using auxiliary training slightly increases the log loss but the increase is negligible. Later we can see that auxiliary training improves the uncertainty quantification for out-of-distribution data. 
We observe that VIFO is competitive with all methods in terms of log loss, with relatively small differences between the top group of methods in each case. \cref{fig:acc-preresnet} and \cref{fig:acc-alexnet} show accuracy on test data in the same experiments, revealing that in many cases VIFO outperforms VI and it is competitive with all methods.
Finally, there is no clear winner between VIFO-mean and VIFO-mv; 
VIFO-mv provides a small advantage overall but might be more sensitive as illustrated by the performance on CIFAR100 with PreResNet20.
\PutPreResNetInD

\subsection{Uncertainty Quantification}
In this section we examine whether VIFO can capture the uncertainty in predictions for shifted and OOD data. 
We measure performance using ECE, Entropy and AUC for detecting OOD data.
% and the appendix includes additional count confidence curves. 
These represent a comprehensive set of measures from the literature. 
For datasets,
for uncertainty under data shift, STL10 and CIFAR10 can be treated as a shifted dataset for each other, as the figure size of STL10 is different from CIFAR10, and STL10 shares some classes with CIFAR10 so the labels are meaningful.
For uncertainty under OOD data, we choose the SVHN dataset as an OOD dataset for CIFAR10 and STL10, as SVHN contains images of digits and the labels of SVHN are not meaningful in the context of CIFAR10.

\begin{figure}[h!]
    \centering
    \begin{subfigure}[b]{0.40\textwidth}
         \centering
         \includegraphics[width=\textwidth]{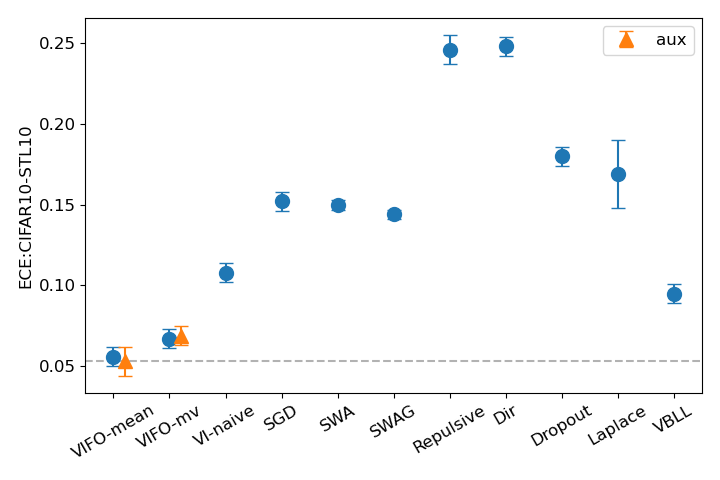}
         \caption{CIFAR10$\rightarrow$STL10, AlexNet}
    \end{subfigure}
    \begin{subfigure}[b]{0.40\textwidth}
         \centering
         \includegraphics[width=\textwidth]{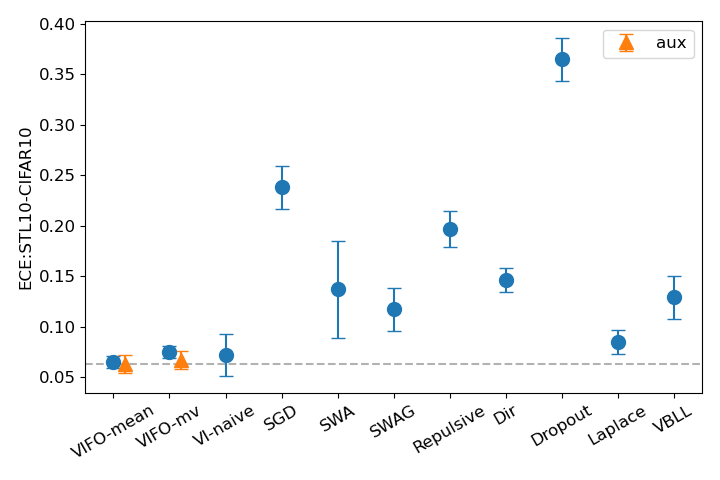}
         \caption{STL10$\rightarrow$CIFAR10, AlexNet}
    \end{subfigure}
    \begin{subfigure}[b]{0.40\textwidth}
         \centering
         \includegraphics[width=\textwidth]{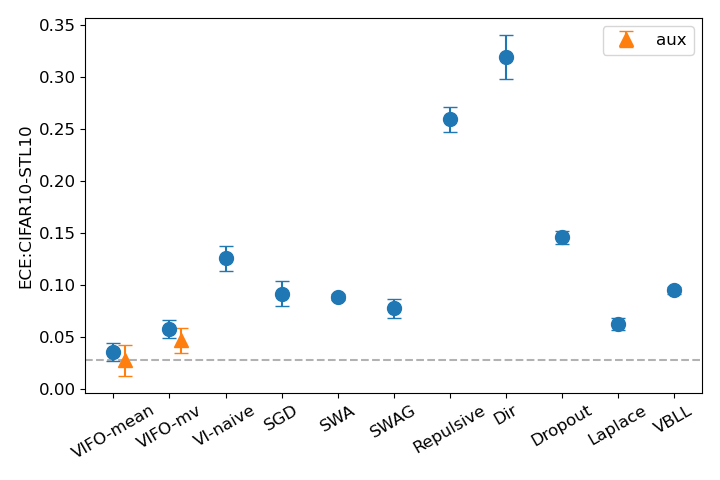}
         \caption{CIFAR10$\rightarrow$STL10, PreResNet20}
    \end{subfigure}
    \begin{subfigure}[b]{0.40\textwidth}
         \centering
         \includegraphics[width=\textwidth]{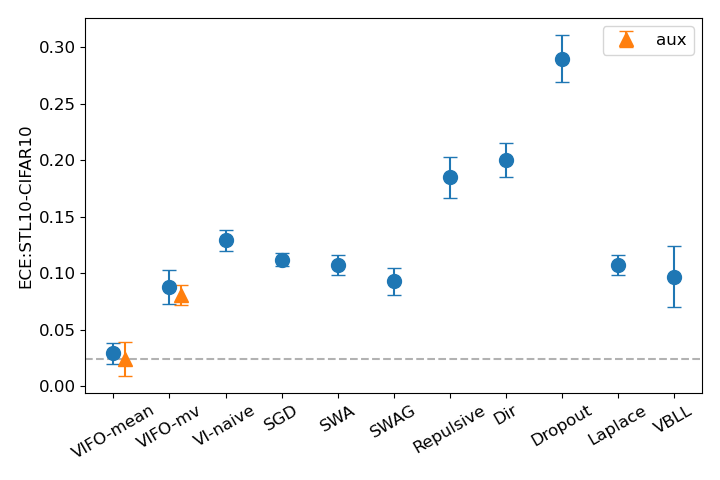}
         \caption{STL10$\rightarrow$CIFAR10, PreResNet20}
    \end{subfigure}
    \caption{ECE ($\downarrow$) on AlexNet and PreResNet20 under data shift. Dashed line indicates the best performance of VIFO. 
    % Each dot with error bar represents the mean and standard deviation of 5 independent runs and the triangle represents the ensemble of each method that aggregate these 5 runs for prediction (same for other figures). 
    Numerical results are listed in the Appendix.}
    \label{fig:ece}
\end{figure}
\paragraph{Expected Calibration Error (ECE)} ECE \citep{ece, uncertainty-benchmark} is often used to measure the uncertainty quantification under data shift. We separate data into bins of the same size according to the confidence level, calculate the difference between the accuracy and the averaged confidence in each bin and then average the absolute differences among all bins. 
% For better calibrated models, the ECE should be lower.
Better calibrated models have lower ECE.
ECE has its faults (for example the trivial classifier has zero ECE) but it is nonetheless informative.
% For the experiment
We selected the number of bins to be 20. 

\cref{fig:ece} shows the ECE of each method under data shift. As we can see, both VIFO-mean and VIFO-mv achieve the best performance compared to all other methods. 

\PutEntropyPreResNet
\paragraph{Entropy } Entropy \citep{uncertainty-benchmark} of the categorical predictive distribution is used to measure the uncertainty quantification for out-of-distribution (OOD) data as the labels for OOD data are meaningless. We want our model to be as uncertain as possible and this implies high entropy and low confidence (the maximum probability assigned to any class) in the predictive distribution. We summarize the averaged entropy for the entire dataset in \cref{fig:entropy-preresnet} and \cref{fig:entropy-alexnet}. 
We can see that both VIFO-mean and VIFO-mv are better than all other methods except repulsive ensembles and the Dirichlet method. 
However, as observed in \cref{fig:log-preresnet} and \cref{fig:log-alexnet}, repulsive ensembles and the Dirichlet method have poor performance in terms of log loss due to underconfident predictions. 
Hence they achieve high entropy by sacrificing in distribution performance whereas VIFO performs well.
Further, we observe from \cref{fig:entropy-preresnet} that auxiliary training greatly improve the performance of VIFO on PreResNet20.
Auxiliary training only has a small impact on VIFO with AlexNet (see \cref{fig:entropy-alexnet})
but VIFO already performs well without auxiliary training in this case.

%This suggests that for AlexNet we may need a better auxiliary input. We leave this for future improvement.
% Roni: I am not sure about this
% aux data is per data not architechture. 
% Perhpas this says that there is less overfitting with AlexNet?

\begin{figure}[h!]
    \centering
    \includegraphics[width=0.55\textwidth]{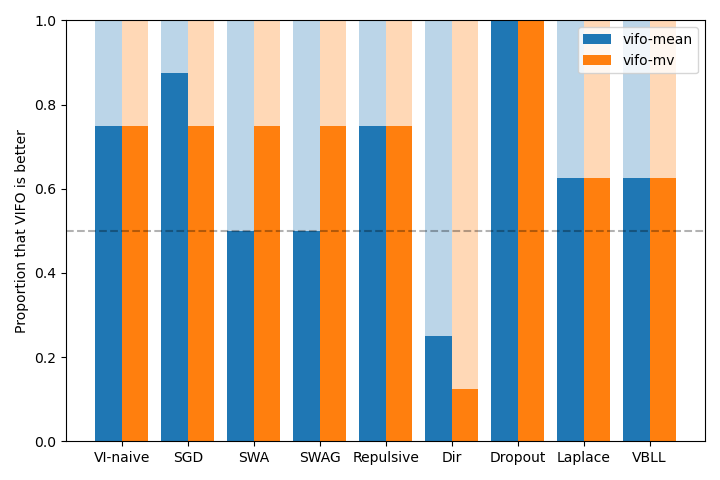}
    \caption{Comparison of VIFO with all other methods in terms of AUROC on OOD data. Y-axis is the proportion of experiments that VIFO is better than other methods. Exact AUROC values are provided in the appendix.}
    \label{fig:auroc}
\end{figure}
\paragraph{AUROC} We use maximum probability of the categorical predictive distribution as the criterion to separate in-distribution and OOD data and compute the area under the ROC curve \citep{dir-prior}. AUROC overcomes the drawbacks of ECE and entropy because a trivial model cannot yield the best performance. 
Detailed comparison plots are in given in \cref{fig:auroc-alexnet} and \cref{fig:auroc-preresnet} in the Appendix. 
We first note that, as above,
%Like observations from entropy, 
auxiliary training improves the performance on PreResNet20 but not significantly on AlexNet.
We found that there is no single method that consistently outperforms all other methods. 
Instead, for better visualization, we show the comparison of VIFO-mean and VIFO-mv with other methods in \cref{fig:auroc}. For each baseline, we count the number of experiments that VIFO performs better and get the corresponding proportion. 
We observe that overall, VIFO-mv is better than all other methods except the Dirichlet method and that it ranks better than VIFO-mean.
%is worse than VIFO-mv as it cannot win swa and swag. 
% Though the Dirichlet method is the best, it has the issue of underconfidence and we cannot claim in general it is better than VIFO.
% As discussed above, the success of the Dirichlet method on OOD data is achieved by sacrificing calibration and in distribution performance. 
Though the Dirichlet method performs better than VIFO on OOD data, its poor in-distribution performance makes it less desirable.
On the other hand, VIFO outperforms all other baselines for OOD and has strong in distribution performance and hence give better overall predictions.

\section{Conclusion}
In Bayesian neural networks, the distribution of the last layer directly affects the predictive distribution. Motivated by this fact, we proposed variational inference on the final-layer output, VIFO, that uses a neural network to directly learn the mean and variance of the last layer. 
We showed that VIFO can match the expressive power of VI in linear cases with a strong prior but that in general it provides a less expressive model.
On the other hand the simplicity of the model enables fast training 
of ensembles of VIFO
and facilitates convergence analysis through Rademacher bounds. 
In addition, VIFO can be derived as a non-standard variational lower bound, which provides an approximation for the last layer. 
This connection allowed us to derive better regularizations for VIFO by using collapsed variational inference over a hierarchical prior. 
Since VIFO treats each input separately, we can incorporate auxiliary inputs to help the model distinguish in-distribution and out-of-distribution data.
Empirical evaluation highlighted that ensembles of VIFO are competitive with or outperform other methods in terms of in-distribution loss and out-of-distribution data detection.  
Hence VIFO gives a new attractive approach for approximate inference in Bayesian models. 
The efficiency of VIFO also means faster test time predictions which can be important when deploying Bayesian models for real-time applications.
Future work could explore more informative auxiliary input to improve the performance of VIFO,
% explore the complexity performance tradeoff provided by VIFO, 
and investigate the connections to variational inference in functional space that induces more complex priors. 

\section*{Acknowledgments}
This work was partly supported by NSF under grants 1906694 and 2246261. Some of the experiments in this paper were run on the Big Red computing system at Indiana University, supported in part by Lilly Endowment, Inc., through its support for the Indiana University Pervasive Technology Institute.

\newpage
\bibliography{main}
\bibliographystyle{tmlr}

\appendix

\counterwithin{figure}{section}
\counterwithin{table}{section}
\section{Proofs}
\subsection{Proofs in Section \ref{sec:VI-VIFO}}
\label{sec:proof-VI-VIFO}
\begin{proof}[Proof of Theorem \ref{thm:linear}]%[Proof Sketch]
    %The loss term in VIFO objective is a reparametrized version of the first term in variational inference objective \eqref{eq:VI-blr} and thus they are the same. For the KL term, one can eliminate all $X_N$ with the properties of matrix inverse and the remaining is exactly the remaining term of \eqref{eq:VI-blr}. 
    %
    Assume $N>d$. 
    Note that with the the correlated prior and posterior the covariance function is rank deficient so we have to interpret inverses and determinants appropriately. Here we use pseudo inverse and pseudo determinant.
    The VIFO objective is:
    {
    %\small
    \begin{align}
        &\sum_{i=1}^N \Big\{\E_{q(z|x_i)} [\log p(y_i|z)] \Big\} - \kl(q(z|X_N) || p(z|X_N)) \\
        =&\sum_{i=1}^N \Big\{\E_{q(z|x_i)} [\log p(y_i|z)] \Big\} 
        - \frac{1}{2} \tr((X_N^\top S_0 X_N)^{-1} (X_N^\top V X_N)) + \frac{N}{2} \nonumber \\
        +& \frac{1}{2} \log|(X_N^\top S_0 X_N)^{-1} (X_N^\top V X_N)| 
        - \frac{1}{2} (w^\top X_N - m_0^\top X_N) (X_N^\top S_0 X_N)^{-1} (w^\top X_N - m_0^\top X_N)^\top.
        \label{eq:uq-complicatedm}
    \end{align}
    }
     
    First consider the loss term. Let $L$ be the Cholesky decomposition of $V$, i.e. $V=L L^\top$. By reparametrization, for $\epsilon \sim \dN(0, I_d)$, $w^\top x_i + x_i^\top L \epsilon \sim \dN(w^\top x_i, x_i^\top L L^\top x_i)$ and thus
    \begin{align}
        \E_{q(z|x_i)} [\log p(y_i|z)] &= \E_{\epsilon \sim \dN(0, I_d)} [\log p(y_i|w^\top x_i + x_i^\top L \epsilon)] \nonumber \\
        &= \E_{\epsilon \sim \dN(0, I_d)} [\log p(y_i | (w + L \epsilon)^\top x_i) ] \nonumber \\
        &= \E_{\theta \sim \dN(w, L L^\top)}[\log p(y_i | \theta^\top x_i)],
        \label{eq:loss-VIFOm}
    \end{align}
    where the last equality uses reparametrization in a reverse order. By aligning $w = m$ and $V = L L^\top = S$, we recognize that Eq~\eqref{eq:loss-VIFOm} is exactly the loss term in Eq~\eqref{eq:VI-blr}. Thus the low-dimensional posterior on $z$ yields the same loss term as the high-dimensional posterior over $W$.

For the regularization, we use the pseudo inverse derivation from Eq~(224) of \citet{matrixcookbook}, where for
$A=CD$ we have $A^{+}=D^\top (D D^\top)^{-1}(C^\top C)^{-1}C^\top$ to get
    \begin{align*}
        (X_N^\top S_0 X_N)^{-1}=X_N^\top (X_N X_N^\top)^{-1} S_0^{-1} (X_N X_N^\top)^{-1} X_N
        %$(X_N^\top S_0 X_N)^{-1}$ $=$ $X_N^\top (X_N X_N^\top)^{-1}$  $S_0^{-1} (X_N X_N^\top)^{-1} X_N$
    \end{align*}
    and the same for $V$. Thus,
    \begin{align*}
        (X_N^\top S_0 X_N)^{-1} (X_N^\top V X_N) &= X_N^\top (X_N X_N^\top)^{-1} S_0^{-1}(X_N X_N^\top)^{-1} X_N X_N^\top V X_N \\
        &= X_N^\top (X_N X_N^\top)^{-1} S_0^{-1} VX_N, \\
        \tr |X_N^\top (X_N X_N^\top)^{-1} S_0^{-1} V X_N| &= \tr |X_N X_N^\top (X_N X_N^\top)^{-1} S_0^{-1} V| \\
        &= \tr (S_0^{-1} V),
    \end{align*}
    and
    \begin{align*}
        &(w^\top X_N - m_0^\top X_N) (X_N^\top S_0 X_N)^{-1} (w^\top X_N - m_0^\top X_N)^\top \\
        =& (w - m_0)^\top X_N (X_N^\top (X_N X_N^\top)^{-1} S_0^{-1} (X_N X_N^\top)^{-1} X_N) X_N^\top (w - m_0) \\
        =& (w - m_0)^\top (X_N X_N^\top) (X_N X_N^\top)^{-1} S_0^{-1} (X_N X_N^\top)^{-1} (X_N X_N^\top) (w - m_0) \\
        =&(w-m_0)^\top S_0^{-1} (w-m_0).
    \end{align*}
    %The log-determinant term is a little tricky. Notice that $X_N^\top (X_N X_N^\top)^{-1} S_0^{-1} V X_N$ is not full rank, so the determinant here denotes the
    For the The log-determinant term we use the pseudo-determinant \citep{pseudo-det}, which is the product of non-zero eigenvalues. Let $(\lambda_i, u_i)_{i=1}^d$ be the set of eigenvalues and eigenvectors of $S_0^{-1} V$, i.e., $S_0^{-1} V u_i = \lambda u_i$, and let $X_N^{\ddagger}=X_N^\top (X_N X_N^\top)^{-1}$ denote the pseudo inverse of $X_N$, then
    \begin{align}
        (X_N^{\ddagger} S_0^{-1} V X_N) X_N^{\ddagger} u_i = X_N^{\ddagger} S_0^{-1} V u_i =\lambda X_N^{\ddagger} u_i,
    \end{align}
    thus $(\lambda_i, X_N^{\ddagger} u_i)_{i=1}^d$ is the eigenvalues and eigenvectors of $X_N^\top (X_N X_N^\top)^{-1} S_0^{-1} VX_N$. Since the rank of this matrix is at most $d$, other eigenvalues are 0 and the pseudo determinant is $\prod_{i=1}^d \lambda_i$, which is exactly the determinant of $S_0^{-1} V$. 
    Then the regularization term in \eqref{eq:VI-blr} can be simplified to: %becomes:
    {
    % \small 
    \begin{align}
        -\kl(q(z|X_N) || p(z|X_N))
        =- \frac{1}{2} \tr(S_0^{-1} V) + \frac{1}{2} \log |S_0^{-1} V| 
        - \frac{1}{2} (w - m_0)^\top S_0^{-1} (w - m_0) + \frac{N}{2}.
        \label{eq:uq-simplifiedm}
    \end{align}
    }
    By aligning $w = m, V= S$, we can seet that \eqref{eq:uq-simplifiedm} is exactly the regularizer in \eqref{eq:VI-blr} ignoring the constant.
\end{proof}

\noindent{\bf Note for the case $K>1$:} Let $\theta \in \mathbb{R}^{d \times K}$. 
For VI, we make a mean field assumption with $q(\theta_k) = \dN(\theta_k|m_k, S_k)$ and $q(\theta) = \prod_{k=1}^K q(\theta_k)$, where $\theta_k$ is the $k$-th column of $\theta$. For VIFO, using mean field let $q(z_k|x)= \dN(z_k|w_k^\top x, x^\top V_k x)$ and $q(z|x) = \prod_{k=1}^K q(z_k|x)$. By aligning $w_k = m_k$ and $V=S_k$, we can find $\E_{q(z|x_i)}[\log p(y_i|z_1, \dots z_K)] = \E_{q(\theta)} [\log p(y_i | (\theta_1^\top x_i, \dots, \theta_K^\top x_i))]$, and 
\begin{align}
    \kl(q(\theta) || p(\theta)) = \sum_k \kl(q(\theta_k), p(\theta_k)) \doteq \sum_{k} \kl(q(z_k|X_N)||p(z_k|X_N)),
\end{align}
where the second $\doteq$ means equivalence ignoring a constant difference.

\begin{proof}[Proof of Theorem \ref{thm:not-recover-vi}]
    Consider a neural network with one single hidden layer, denote the weights of the first layer as $u$, and the weights of the second layer as $w$. Thus, the $k$-th output can be computed as:
\begin{align*}
    z^{(k)} = \sum_{i=1}^I w_{k, i} \psi \left(\sum_{d=1}^D u_{i, d} x_d \right),
\end{align*}
where $I$ is the size of the hidden layer, $D$ is the input size and $\psi(a)=\max(0, a)$ is the ReLU activation function. We further simplify the setting by considering the special case where only $x_1$ is non-zero and $I=1$. Then the $k$-th output becomes:
\begin{align*}
    z^{(k)} = w_{k, 1} \psi(u_{1, 1} x_1).
\end{align*}
Consider a distribution $q(w_{k, i})=\dN(\bar{w}_{k, i}, \sigma_w^2), q(u_{i, d})=\dN(\bar{u}_{i,d}, \sigma_u^2)$. Then if $x_1\geq 0$, 
\begin{align}
    \label{eq:pos}
    \E_{q(w) q(u)} \left[z^{(k)} \right] &= \E_{w, u} \left[w_{k, 1} \psi(u_{1, 1} x_1) \right] \nonumber \\
    &=\bar{w}_{k, 1} \left(\bar{u}_{1, 1} + \frac{\phi \left(-\frac{\Bar{u}_{1, 1}}{\sigma_u}\right)}{1-\Phi \left(-\frac{\Bar{u}_{1, 1}}{\sigma_u}\right)} \sigma_u \right) \left(1-\Phi \left(-\frac{\Bar{u}_{1, 1}}{\sigma_u}\right) \right) x_1;
\end{align}
if $x_1 < 0$, then
\begin{align}
    \label{eq:neg}
    \E_{q(w) q(u)} \left[z^{(k)} \right] &= \E_{w, u} \left[w_{k, 1} \psi(u_{1, 1} x_1) \right] \nonumber \\
    &=\bar{w}_{k, 1} \left(\bar{u}_{1, 1} - \frac{\phi\left(-\frac{\Bar{u}_{1, 1}}{\sigma_u}\right)}{\Phi\left(-\frac{\Bar{u}_{1, 1}}{\sigma_u}\right)} \sigma_u \right) \Phi\left(-\frac{\Bar{u}_{1, 1}}{\sigma_u}\right) x_1,
\end{align}
where $\phi$ and $\Phi$ are the pdf and cdf of standard normal distribution and we directly use the expectation of the truncated normal distribution. Now consider $\Tilde{w}$ and $\Tilde{u}$ that aim to recover (\ref{eq:pos}) and (\ref{eq:neg}). If $\Tilde{u}_{1, 1} \geq 0$, it cannot successfully recover (\ref{eq:neg}) because the ReLU activation will have 0 when $x_1<0$ so that it cannot recover (\ref{eq:neg}); if $\Tilde{u}_{1, 1} <0$, for the same reason it cannot recover (\ref{eq:pos}). 
\end{proof}

\subsection{Proofs in Section \ref{sec:rademacher}}
\label{app:rademacher}

{\bf Verifying Assumption \ref{assumption:loss}:} We next verify that Assumption \ref{assumption:loss} holds for classification and (with a modified loss) for regression.

For $K$-classification, $z$ is $K$-dimensional and the negative log-likelihood is
\begin{align*}
    -\log p(y=k|z) = -\log \frac{\exp(z_k)}{\sum_{i=1}^K \exp(z_i)} = -z_k + \log \sum_{i=1}^K \exp(z_i)
\end{align*}
which is 1-Lipschitz in $z$.

For regression, $z=(m, l)$ is 2-dimensional, and the negative log-likelihood is:
\begin{align*}
    -\log p(y|z) = \frac{1}{2} (y-m)^2 \exp(-l) + \frac{1}{2} l. 
\end{align*}
Neither the quadratic function nor exponential function is Lipschitz. But we can replace the unbounded quadratic function $(y-m)^2$ with a bounded version $\min\{(y-m)^2, B_m^2\}$, and replace the exponential function $\exp(-l)$ with $\min\{\exp(-l), B_l\}$, where $B>0$, to guarantee the Lipschitzness. Now the negative log-likelihood is:
\begin{align*}
    -\log p(y|z) = \frac{1}{2} \min\{(y-m)^2, B_m^2\} \min\{\exp(-l), B_l\} + \frac{1}{2} l,
\end{align*}
is $(B_m B_l)$-Lipschitz in $m$, $\left(\frac{1}{2}+ \frac{1}{2}B_m^2 B_l \right)$-Lipschitz in $l$.

{\bf Verifying Assumption \ref{assumption:link}:} For Assumption \ref{assumption:link}, we can use $g(l)=\log(1+\exp(l))$ which is 1-Lipschitz.
If $g(l)=\exp(l)$ is the exponential function, we can use a bounded variant that satisfies the requirement $g(l)=\max\{\exp(x), B_g\}$. 

\begin{proof}[Proof of Lemma \ref{lemma:rademacher-bivariate}]
    We prove the lemma for $L=1$. If this is not the case, we can define $\phi' = \frac{1}{L} \phi$, and use the fact that $R(\phi(A\times B)) \leq L R(\phi'(A\times B))$. Let $C_i=\{(a_1+b_1, \dots, a_{i-1} + b_{i-1}, \phi'(a_i, b_i), a_{i+1} + b_{i+1}, \dots, a_N + b_N): a\in A, b\in B\}$. It suffices to prove that for any set $A, B$ and all $i$ we have $R(C_i)\leq R(A) +R(B)$. Without loss of generality we prove the case for $i=1$.
    \begin{align*}
        N R(C_1) &= \E_{\sigma}\left[\sup_{c\in C_1} \sigma_1 \phi(a_1, b_1) + \sum_{i=2}^N \sigma_i (a_i + b_i) \right] \\
        &=\frac{1}{2} \E_{\sigma_2,\dots, \sigma_N} \left[\sup_{a\in A, b\in B} \left(\phi(a_1, b_1) + \sum_{i=2}^N \sigma_i (a_i + b_i) \right) 
        % Roni: split the line b/c too wide; need right/left to avoid error
        \right.
        \\ & \mbox{\hspace{3cm}}
        \left.
        + \sup_{a'\in A, b'\in B} \left(-\phi(a_1', b_1') + \sum_{i=2}^N \sigma_i (a_i' + b_i') \right) \right] \\
        &=\frac{1}{2} \E_{\sigma_2,\dots, \sigma_N} \left[\sup_{a, a'\in A, b,b'\in B} \left(\phi(a_1, b_1) - \phi(a_1', b_1') + \sum_{i=2}^N \sigma_i (a_i +b_i) + \sum_{i=2}^N \sigma_i (a_i' + b_i') \right) \right] \\
        &\leq \frac{1}{2} \E_{\sigma_2 \dots \sigma_N} \left[\sup_{a, a'\in A, b, b' \in B}\left( |a_1 - a_1'| + |b_1-b_1'| + \sum_{i=2}^N \sigma_i (a_i + b_i) + \sum_{i=2}^N \sigma_i(a_i'+b_i') \right) \right] \\
        &=\frac{1}{2} \E_{\sigma_2, \dots, \sigma_N} \left[\sup_{a, a'\in A} \left(a_1 - a_1'+\sum_{i=2}^N \sigma_i a_i + \sum_{i=2}^N \sigma_i a_i' \right) \right] 
        % Roni: split the line b/c too wide
        \\ & \mbox{\hspace{3cm}}
        + \frac{1}{2}\E_{\sigma_2, \dots, \sigma_N} \left[\sup_{b, b'\in B} \left(b_1 - b_1'+\sum_{i=2}^N \sigma_i b_i + \sum_{i=2}^N \sigma_ib_i' \right) \right] \\
        &=NR(A)+NR(B).
    \end{align*}
\end{proof}

\begin{proof} [Proof of Theorem \ref{thm:rademacher}]
    %Recall that the function class is
    %$\mathcal{H} \times \mathcal{H}$ $=$ $\{ (f_{W_1}(x),f_{W_2}(x))$  $| W=(W_1,W_2)\}$,   
    %and the loss function is $l(W,(x,y))=E_{q_W(z|x)}[-\log p(y|x)]$.
    We show that the loss is Lipschitz in $f_{W_1}(x)$ and $f_{W_2}(x)$.
    Fix any $x$, and $W, W'$. We denote the mean and standard deviation of $q_{W}(z|x)$ by
    $\mu$ and $s$ and the same for $q_{W'}(z|x)$. We use $\cdot$ for Hadamard product.
    %and $m_2$ and $s_2$ as the mean and standard deviation of $q_{W_2}(z|x)$.
    \begin{align*}
        &\E_{q_{W}(z|x)}[-\log p(y|z)] - \E_{q_{W'}(z|x)} [-\log p(y|z)] \\
        =& \E_{\epsilon \sim \dN(0, I)}[\log p(y|\mu' + \epsilon \cdot s')-\log p(y|\mu+ \epsilon \cdot s)] \\
        \leq & \E_{\epsilon \sim \dN(0, I)} \left[L_{0} \lVert (\mu-\mu') + \epsilon \cdot (s - s') \rVert_2 \right] \quad \quad \text{(Lipschitz)} \\
        \leq & L_{0} \lVert \mu - \mu' \rVert_2 + L_0 \E_{\epsilon} \left[\sqrt{\lVert\epsilon \cdot (s - s')\rVert_2^2} \right] \\
        \leq & L_{0} \lVert \mu - \mu' \rVert_2 + L_{0} \sqrt{\E_{\epsilon} [\lVert\epsilon \cdot (s - s')\rVert_2^2]} \quad \quad \text{(Jensen's Ineq)} \\
        = & L_{0} (\lVert \mu - \mu' \rVert_2  + \lVert s - s' \rVert_2) \\
        \leq & L_0 (\lVert \mu - \mu' \rVert_1 + \lVert s - s' \rVert_1).
    \end{align*}
    For the 6th line note
    that $\E_{\epsilon} [\lVert\epsilon \cdot (s - s')\rVert_2^2]=\E_{\epsilon} [\sum_i \epsilon_i^2 (s_i - s_i')^2]= \sum_i \E_{\epsilon_i \sim \dN(0, 1)} [\epsilon_i^2 (s_i -s_i)^2 ] = \sum_i (s_i - s_i')^2 = \lVert s - s'\rVert_2^2$.
    %\leq = \lVert s - s'\rVert_1^2$.
    % and that $\lVert \cdot \rVert_2 \leq \lVert \cdot \rVert_1$.
    % In the above equations the product and power are performed element-wise, and the derivation holds for multivariate $z$. 
    The loss function is Lipschitz in $\mu$, which is exactly $f_{W_1}(x)$.
    Further, $s$ is $L_1$-Lipschitz in the logit $f_{W_2}(x)$, thus, the loss function is $(L_0 \max\{1, L_1\})$-Lipschitz in the concatenation of $f_{W_1}(x)$ and $f_{W_2}(x)$,
    each of which is of dimension $K$.
    %Notice that both $\mu$ and $s$ are of dimension $K$.
    The theorem now follows from 
    %Lemma \ref{lemma:rademacher-bivariate}, 
    Corollary \ref{cor:multi-lipschitz}.
    %The theorem follows from Rademacher complexity bounds for bivariate and multivariate functions (Lemma \ref{lemma:rademacher-bivariate}, Corollary \ref{cor:multi-lipschitz}).
\end{proof}

\section{Derivations of Collapsed Variational Inference}
\label{sec:derive-cvi}
As is shown by \cite{collapsed-elbo}, for priors and approximate posteriors from the exponential family, we can derive the closed-form solution for the optimal $q^*(\mu_p, \sigma_p^2)$,
\begin{align}
    \log q^*(\mu_p, \sigma_p^2| x) \propto \log p(\mu_p, \sigma_p^2) + \E_{q(z|x)}[\log p(z|\mu_p, \sigma_p^2)], 
\end{align}
for optimizing $q(\mu_p, \sigma_p^2)$ for every single data. Our derivations follow the methodology of \citet{collapsed-elbo} but they are applied on the output $z$ instead of the weights $W$.

\subsection{Learn mean, fix variance}
Let $p(z|\mu_p)=\dN(z|\mu_p, \gamma I)$, $p(\mu_p)=\dN(\mu_p|0, \alpha I)$. Recall that $q(z|x)=\dN(\mu_q(x), \diag(\sigma_q^2(x)))$. Then 
\begin{align*}
    \log q^*(\mu_p|x) &\propto \log p(\mu_p) + \E_{q(z|x)}[\log p(z|\mu_p)] \\
    &\propto -\frac{1}{2\alpha} \mu_p^\top \mu_p -\frac{1}{2\gamma}[(\mu_p-\mu_q(x))^\top (\mu_p - \mu_q(x)) + 1^\top \sigma_q^2(x)] \\
    &\propto -\frac{\alpha + \gamma}{2\alpha \gamma} \left(\mu_p - \frac{\alpha}{\alpha + \gamma} \mu_q(x) \right)^\top \left(\mu_p - \frac{\alpha}{\alpha + \gamma} \mu_q(x) \right).
\end{align*}
Then $q^*(\mu_p)=\dN(\frac{\alpha}{\alpha+\gamma}\mu_q(x), \frac{\alpha\gamma}{\alpha+\gamma} I)$. Pluging $q^*$ into the regularizer, the new regularizer becomes 
\begin{align*}
    \frac{1}{2\gamma} \left[1^\top \sigma^2_q(x) + \frac{\gamma}{\gamma + \alpha} \mu_q(x)^\top \mu_q(x) \right] - \frac{1}{2} 1^\top \log \sigma_q^2(x)+\frac{K}{2} \log (\gamma+\alpha) - \frac{K}{2}.
\end{align*}

\subsection{Learn both mean and variance}
Let $p(z|\mu_p, \sigma_p^2)=\dN(z|\mu_p, \sigma_p^2)$, $p(\mu_p|\sigma_p^2)=\dN(\mu_p|0, \frac{1}{t} \sigma_p^2)$, $p(\sigma_p^2)=\dG(\sigma_p^2|\alpha, \beta)$, where $\dG$ indicates the inverse Gamma distribution. Let $q(\mu_p)$ be a diagonal Gaussian and $q(\sigma_p^2)$ be inverse Gamma. Use $\mu_{p, i}$ and $\sigma_{p, i}$ to denote the $i$-th entry of $\mu_p$ and $\sigma_p$ respectively, then
\begin{align*}
    &\log q^*(\mu_{p, i}, \sigma_{p, i}^2) \\
    \propto & \log p(\mu_{p, i}, \sigma_{p, i}^2) + \E_{q(z|x)}[\log p(z|\mu_p, \sigma_p^2)] \\
    \propto & -(\alpha+\frac{3}{2}) \log \sigma_{p,i}^2 - \frac{2\beta + t \mu_{p, i}^2}{2\sigma_{p, i}^2} - \frac{1}{2} \log \sigma_{p, i}^2 - \frac{1}{2} \frac{(\mu_{p,i} - \mu_{q, i}(x))^2}{\sigma_{p, i}^2} - \frac{1}{2} \frac{\sigma_{q, i}^2(x)}{\sigma_{p, i}^2} \\
    \propto & -(\alpha + 2) \log \sigma_{p, i}^2 - \frac{1}{2\sigma_{p, i}^2} \left(2 \left(\beta + \frac{t}{2(t+1)} \mu_{q, i}(x)^2 + \frac{1}{2} \sigma_{q, i}^2(x) \right) + (t+1)\left(\mu_{p, i} - \frac{\mu_{q, i}(x)}{t+1} \right)^2  \right) 
\end{align*}
follows the normal-inverse-gamma distribution. Thus $q^*(\mu_p|x) = \dN(\mu_p| \frac{1}{t+1} \mu_q(x), \frac{1}{t+1} \sigma_p^2)$ and $q^*(\sigma_p^2 | x) = \dG(\sigma_p^2|(\alpha + \frac{1}{2})1, \beta+\frac{t}{2(t+1)} \mu_q(x)^2 + \frac{1}{2}\sigma_q^2(x))$. Then the regularizer becomes
\begin{align}
    (\alpha+\frac{1}{2}) 1^\top \log \left[\beta 1 + \frac{t}{2(1+t)} \mu_q(x)^2 + \frac{1}{2}\sigma_q^2(x) \right] - \frac{1}{2} 1^\top \log \sigma_q^2(x).
\end{align}

\subsection{Empirical Bayes}
\label{eb}
Let $p(\sigma_p^2)=\dG(\sigma_p^2|\alpha, \beta)$, $p(z|\sigma_p^2)=\dN(z|0, \sigma_p^2 I)$, and let $q(\sigma_p^2)$ be a delta distribution. Then 
\begin{align*}
    &\kl(q(z|x)||p(z|\sigma^2_p)) - \log p(\sigma_p^2) \\
    =& \frac{1}{2}\left[K \log \sigma_p^2 - 1^\top \log \sigma_{q}^2(x) - K + \frac{1^\top \sigma_{q}^2(x)}{\sigma_p^2} + \frac{\mu_{q}(x)^\top \mu_{q}(x)}{\sigma_p^2} \right] + (\alpha + 1) \log \sigma_p^2 + \frac{\beta}{\sigma_p^2}.
\end{align*}
By taking the derivatives of the above equation with respect to $\sigma_p^2$ and solving, we obtain the optimal $\sigma_p^2 = \frac{\mu_q(x)^\top \mu_q(x) + 1^\top\sigma_q^2(x) + 2\beta}{K+2\alpha+2}$. If we plug this back into the KL term, we get the regularizer:
\begin{align}
    &\frac{1}{2} \left[K\log \frac{\mu_q(x)^\top \mu_q(x) + 1^\top\sigma_q^2(x) + 2\beta}{K + 2\alpha + 2} - 1^\top \log |\sigma_q^2(x)| \right] \nonumber \\
    &- \frac{K}{2}+ \frac{1}{2} \frac{(K+2\alpha+2)(\mu_q(x)^\top \mu_q(x) + 1^\top\sigma_q^2(x))}{\mu_q(x)^\top \mu_q(x) + 1^\top\sigma_q^2(x) + 2\beta}.
    \label{eq:eb-regularizer}
\end{align}

However, if we include the negative log-prior term $(\alpha + 1) \log \frac{\mu_q(x)^\top \mu_q(x) + 1^\top\sigma_q^2(x) + 2\beta}{K+2\alpha+2} + \beta \frac{K+2\alpha+2}{\mu_q(x)^\top \mu_q(x) + 1^\top\sigma_q^2(x) + 2\beta} $, adding them up we will have 
\begin{align*}
    \frac{1}{2} (K+2\alpha+2) \log \frac{\mu_q(x)^\top \mu_q(x) + 1^\top\sigma_q^2(x) + 2\beta}{K + 2\alpha + 2} - 1^\top \log \sigma_q^2(x) + \text{const},
\end{align*}
which highly reduces the complexity of the regularizer. This performs less well in practice and therefore we follow \cite{DVI} and use \eqref{eq:eb-regularizer}. 

\section{Optimizing the Variational Distribution for All Data}
\label{sec:opt-all-data}
In the previous section we show the derivation of collapsed variational inference where $q^*(\mu_p, \sigma_p^2)$ is optimized for every data point $x$. In this section we show how to optimize $q(\mu_p, \sigma_p^2)$ for all data and obtain different regularizers to the ones mentioned in the above section. These perform less well in practice but we include them here for completeness. The closed-form solution for $q^*(\mu_p, \sigma_p^2)$ for all data is
\begin{align}
    \log q^*(\mu_p, \sigma_p^2) \propto \frac{1}{N} \sum_{(x,y)\in \mathcal D} \Bigg\{\log p(\mu_p, \sigma_p^2) + \E_{q(z|x)}[\log p(z|\mu_p, \sigma_p^2)] \Bigg\}.
    \label{eq:solve-cVI-all}
\end{align}

\subsection{Learn mean, fix variance, optimize for all data}
Let $p(z|\mu_p)=\dN(z|\mu_p, \gamma)$, $p(\mu_p)=\dN(\mu_p|0, \alpha)$. Given a dataset $\mathcal D=\{(x, y)\}$, we can get one optimal $q^*(\mu_p)$ for all data. According to \eqref{eq:solve-cVI-all},
\begin{align*}
    \log q^*(\mu_p) &\propto \frac{1}{N} \sum_{(x,y)\in \mathcal D} \Bigg\{ \log p(\mu_p) + \E_{q(z|x)} [\log p(z|\mu_p)] \Bigg\} \\
    &\propto -\frac{1}{2 \alpha} \mu_p^\top \mu_p - \frac{1}{2\gamma N} \sum_{(x,y)\in \mathcal D} ((\mu_p - \mu_q(x))^\top (\mu_p - \mu_q(x)) + 1^\top \sigma_q^2(x)) \\
    &\propto -\frac{\alpha+\gamma}{2 \alpha \gamma} \left(\mu_p - \frac{1}{N}\sum_{(x,y)\in \mathcal D} \mu_q(x) \right)^\top \left(\mu_p - \frac{1}{N}\sum_{(x,y)\in \mathcal D} \mu_q(x) \right).
\end{align*}
Then the optimal $q^*(\mu_p)=\dN(\frac{\alpha}{\alpha+\gamma} \frac{1}{N} \sum_x \mu_q(x), \frac{\alpha \gamma}{\alpha + \gamma} I)$. Let $\Bar{\mu}_q=\frac{1}{N}\sum_x \mu_q(x)$. The regularizer now is:
\begin{align}
    &\sum_{(x, y)}\Bigg\{\E_{q(\mu_p)}[\kl(q(z|x) || p(z|\mu_p, \gamma))] + \kl(q(\mu_p)||p(\mu_p)) \Bigg\} \nonumber \\
    =& \sum_{(x, y)} \Bigg\{ \E_{q(\mu_p)}\left[ K \log \gamma - 1^\top \log \sigma_q^2(x) - K + \frac{1}{\gamma} 1^\top \sigma_q^2(x) - \frac{1}{\gamma} (\mu_q(x)-\mu_p)^\top (\mu_q(x)-\mu_p) \right] \Bigg\} \nonumber \\
    & + \frac{N}{2}\left[K \log \frac{\alpha + \gamma}{\gamma} - K + K \frac{\gamma}{\gamma+\alpha} + \frac{\alpha^2}{(\alpha+\gamma)^2} \Bar{\mu}_q^\top \Bar{\mu}_q \right] \nonumber \\
    =&\sum_{(x, y)} \Bigg\{ \frac{1}{2\gamma} (1^\top \sigma_q^2(x) + \mu_q(x)^\top \mu_q(x)) - \frac{1}{2} 1^\top \log \sigma_q^2(x) \Bigg\} - \frac{N}{2}\left(\frac{1}{\gamma} - \frac{1}{\alpha+\gamma} \right) \Bar{\mu}_q^\top \Bar{\mu}_q + \frac{NK}{2}\log (\alpha + \gamma) - \frac{NK}{2}.
\end{align}
We refer to this method as ``VIFO-mean\_all''.

\subsection{Learn both mean and variance, optimize mean for single data, and variance for all data}
\label{mv-all}
Let $p(z|\mu_p, \sigma_p^2) = \dN(z|\mu_p, \sigma_p^2)$, $p(\mu_p|\sigma_p^2)=\dN(\mu_p|0, \frac{1}{t} \sigma_p^2)$, $p(\frac{1}{\sigma_p^2}) = \dG(\frac{1}{\sigma_p^2}|\alpha, \beta)$. Consider that
\begin{align}
    &\log p(\mu_{p, i}, \sigma_{p, i}^2) + \E_{q(z|x)}[\log p(z|\mu_{p, i}, \sigma_{p,i}^2)] \\
    = & \log p(\mu_{p, i} | \sigma_{p, i}^2) + \log p(\sigma_{p, i}^2) + \E_{q(z|x)}[\log p(z|\mu_{p, i}, \sigma_{p,i}^2)] \\
    \propto & -\frac{t}{2} \frac{\mu_{p, i}^2 }{\sigma_{p, i}^2} - \frac{1}{2} \log \sigma_{p, i}^2 - (\alpha + 1) \log \sigma_{p, i}^2 - \frac{\beta}{\sigma_{p, i}^2} - \frac{1}{2} \log \sigma_{p, i}^2 - \frac{1}{2 \sigma_{p, i}^2} ((\mu_{p, i} - \mu_{q,i}(x))^2 + \sigma_{q, i}^2(x)) \\
    =&-\frac{t}{2} \frac{\mu_{p, i}^2 }{\sigma_{p, i}^2} - \log \sigma_{p, i}^2 - (\alpha + 1) \log \sigma_{p, i}^2 - \frac{\beta}{\sigma_{p, i}^2} - \frac{1}{2 \sigma_{p, i}^2} ((\mu_{p, i} - \mu_{q,i}(x))^2 + \sigma_{q, i}^2(x)), \\
    =& -\frac{t+1}{2 \sigma_{p,i}^2} \left(\mu_{p, i} - \frac{1}{t+1} \mu_{q,i}(x) \right)^2 - \frac{t \mu_{q,i}^2(x)}{2(t+1) \sigma_{p,i}^2} - (\alpha + 2) \log \sigma_{p, i}^2 - \frac{\beta}{\sigma_{p, i}^2} - \frac{\sigma_{q,i}^2(x)}{2 \sigma_{p, i}^2}
    \label{eq:mv}
\end{align}
Then by extracting the $\mu_p$ part from \eqref{eq:mv}, we have
\begin{align*}
    \log q^*(\mu_{p, i}|\sigma_{p, i}^2, x) \propto -\frac{t+1}{2 \sigma_{p,i}^2} \left(\mu_{p, i} - \frac{1}{t+1} \mu_{q,i}(x) \right)^2,
\end{align*}
and thus $q^*(\mu_p|x, \sigma_p^2)=\dN(\mu_p | \frac{1}{t+1} \mu_q(x), \frac{1}{t+1} \sigma_p^2)$. Then we try to marginalize out $\mu_p$ to compute $q^*(\sigma_p^2)$.
\begin{align*}
    \log q^*(\sigma_{p, i}^2) &\propto \frac{1}{N} \sum_{(x, y) \in \mathcal D} \log \int \exp\left(\log p(\mu_{p, i}, \sigma_{p, i}^2) + \E_{q(z|x)}[\log p(z|\mu_{p, i}, \sigma_{p,i}^2)] \right) d\mu_{p, i} \\
    &\propto \frac{1}{N} \sum_{x, y \in \mathcal D} \Bigg\{ \frac{1}{2}\log \int \exp\left(-\frac{t+1}{2 \sigma_{p,i}^2} \left(\mu_{p, i} - \frac{1}{t+1} \mu_{q,i}(x) \right)^2 \right) d\mu_{p, i} \\
    &- \frac{t \mu_{q,i}^2(x)}{2(t+1) \sigma_{p,i}^2} - (\alpha + 2) \log \sigma_{p, i}^2 - \frac{\beta}{\sigma_{p, i}^2} - \frac{\sigma_{q,i}^2(x)}{2 \sigma_{p, i}^2} \Bigg\} \\
    &= -(\alpha + 2) \log \sigma_{p, i}^2 - \frac{\beta}{\sigma_{p, i}^2} + \frac{1}{N} \sum_{(x, y)\in \mathcal D} \left(\frac{1}{2} \log \frac{2 \pi \sigma_{p, i}^2}{t+1} - \frac{\sigma_{q, i}^2(x)}{2 \sigma_{p, i}^2 } - \frac{t \mu_{q,i}^2(x)}{2(t+1) \sigma_{p,i}^2} \right) \\
    &=-(\alpha + \frac{3}{2}) \log \sigma_{p, i}^2 - \frac{\beta + \frac{t}{2(t+1)} \frac{1}{N} \sum \mu_{q, i}^2(x) + \frac{1}{2N}\sum \sigma_{q, i}^2(x)}{\sigma_{p, i}^2},
\end{align*}

and $q^*(\sigma_p^2) = \dG(\sigma_p^2|(\alpha + \frac{1}{2})1, \beta+\frac{t}{2(t+1)} \frac{1}{N} \sum_x\mu_q(x)^2 + \frac{1}{2}\frac{1}{N} \sum_x\sigma_q^2(x))$. Let $\Tilde{\mu}_q = \sqrt{\frac{1}{N} \sum_x \mu_q(x)^2}$ and $\Tilde{\sigma}_q = \sqrt{\frac{1}{N} \sum_x \sigma_q(x)^2}$, then the regularizer becomes:
\begin{align}
    &(\alpha+\frac{1}{2}) N 1^\top \log \left[\beta 1 + \frac{t}{2(1+t)} \Tilde{\mu}_q^2 + \frac{1}{2}\Tilde{\sigma}_q^2 \right] - \sum_{(x, y)} \frac{1}{2} 1^\top \log \sigma_q^2(x) \\
    &+ KN \log \frac{\Gamma(\alpha)}{\Gamma(\alpha + \frac{1}{2})} - NK \alpha \log \beta + \frac{NK}{2} \log \frac{t+1}{t} - \frac{NK}{2}.
\end{align}
We refer to this method as ``VIFO-mv\_all''.

\subsection{Empirical Bayes for all data}
If we optimize $\sigma_p^2$ for all data, then we have
\begin{align*}
    &\sum_{(x, y)\in \mathcal D} \Big\{\kl(q(z|x)||p(z|\sigma^2_p)) - \log p(\sigma_p^2) \Big\} \\
    =& \sum_{(x, y)\in \mathcal D} \Bigg\{ \frac{1}{2}\left[K \log \sigma_p^2 - 1^\top \log \sigma_{q}^2(x) - K + \frac{1^\top \sigma_{q}^2(x)}{\sigma_p^2} + \frac{\mu_{q}(x)^\top \mu_{q}(x)}{\sigma_p^2} \right] + (\alpha + 1) \log \sigma_p^2 + \frac{\beta}{\sigma_p^2} \Bigg\}
\end{align*}
and the optimal variance being $\frac{\Tilde{\mu}_q^\top \Tilde{\mu}_q + 1^\top \Tilde{\sigma}_p^2 + 2\beta }{K + 2\alpha + 2 }$ where $\Tilde{\mu}_q = \sqrt{\frac{1}{N} \sum_x \mu_q(x)^2}$ and $\Tilde{\sigma}_q = \sqrt{\frac{1}{N} \sum_x \sigma_q(x)^2}$. The objective is:
\begin{align*}
    &\frac{NK}{2} \log \frac{2\beta + \Tilde{\mu}_q^2 + \Tilde{\sigma}_q^2}{K + 2\alpha + 2} - \frac{1}{2} \sum_x 1^\top \log \sigma_q(x)^2 - \frac{NK}{2} \\
    +& \frac{1}{2} \frac{K+2\alpha+2}{2\beta + \Tilde{\mu}_q^2 + \Tilde{\sigma}_q^2}\sum_x (\mu_q(x)^\top \mu_q(x) + 1^\top \sigma_q^2(x)).
\end{align*}
This method is called ``VIFO-eb\_all''.

\section{Experimental Details}
\subsection{Experiments on Artificial Dataset}
\label{sec:detail-artificial}
To generate \cref{fig:aux} and \cref{fig:prior}, we generate 100 training data points $y = 2\sin x + 0.1 \epsilon, \epsilon \sim \dN(0, 1)$, where $x_{\text{train}} \in [-\frac{3}{4} \pi, -\frac{1}{2} \pi] \cup [\frac{1}{2} \pi, \frac{3}{4} \pi]$ and $x_{\text{test}} \in [-\pi, \pi]$. We use a multilayer perceptron neural network with 5 layers, each layer containing 50 hidden units to fit the data. For VI, we pick the prior standard deviation to be 1.0 and for VIFO-mean, we select $\gamma=0.3, \frac{\gamma}{\alpha + \gamma} = 0.05$. For both models, we select the regularization parameter $\eta=0.1$ and for VIFO we choose $\eta_\aux=1.0$. For linear regression, we can explicitly compute the predicted variance if we use exponential function as the link function. Suppose $p(y|z) = \dN(y \mid m, \exp(l))$ and $p(m \mid x) = \dN(m \mid \mu_m, \sigma_m^2), p(l \mid x) = \dN(l \mid \mu_l, \sigma_l^2)$, then $p(y \mid x) = \dN(y \mid \mu_m, \sigma_m^2 + \exp(m_l + \sigma_l^2 / 2))$. See Appendix B.3 of \cite{DVI} for the derivation. 
To visualize the predictive distribution inducted by the prior, we draw multiple $z$'s from prior, then draw multiple $y$'s from likelihood $p(y|z)$ and plot them in \cref{fig:prior}.

\subsection{Experiments on Large Image Datasets}
\label{sec:detail-exp}
In this section we elaborate the experimental details, including the choice of hyperparameters, learning rates and the number of training epochs. 

\newcommand{\RK}[1]{{\textcolor{blue}{[RK: #1]}}}

\paragraph{Number of training epochs:} We train all methods in 500 epochs. 
% \RK{Global for all methods? Missing in all}

\paragraph{Learning rate:}
For all methods other than SGD, SWA and SWAG, we use the Adam optimizer with learning rate 0.001.

\paragraph{VI and VIFO:} 
We first list the choices of the variance for naive variational methods. 
%and regularization parameters. 
The choice of prior variance significantly affects the performance. 
For image datasets with complex neural networks, the total prior variance of VI grows with the number of parameters so we have to pick a small variance and we use 0.05 following the setting of \cite{competition-approximate-inference}. Since VIFO samples in the output space which is small, using 0.05 regularizes too strongly and we therefore set a larger value of 1 for the variance. 

%Since all experiments use the same sets of hyperparameters for collapsed variational inference, we introduce these here. 

For collapsed variational inference, we pick $\gamma=0.3$, $\alpha_{reg}=\frac{\gamma}{\alpha + \gamma} = 0.05$ for learn-mean regularizer (VI-mean, VIFO-mean, VIFO-mean-all) and $\alpha=0.5, \beta=0.01, \delta = \frac{t}{1+t}=0.1$ for learn-mean-variance regularizer (VI-mv, VIFO-mv, VIFO-mv\_all), which exactly follows \cite{collapsed-elbo}. We pick $\alpha=4.4798$ and $\beta=10$ for empirical Bayes (VI-eb, VIFO-eb, VIFO-eb\_all). The choice of $\alpha$ in empirical Bayes follows \cite{DVI} but the choice of $\beta$ is unclear in \cite{DVI} so we just perform a simple search from $\{1, 10, 100\}$ and set $\beta=10$ that yields the best result. 

%Next we list the choices of the variance for naive methods and regularization parameters. The choice of prior variance significantly affects the performance. 
%For image datasets with complex neural networks, the total prior variance of VI grows with the number of parameters so we have to pick a small variance and we use 0.05 following the setting in \cite{competition-approximate-inference}. Since VIFO samples in the output space which is small, using 0.05 regularizes too strongly and we therefore set a larger value of 1 for the variance. 

% Next we list the choices for $\eta$ and optimization routine. 
For both VI and VIFO,
the regularization parameter $\eta$ is fixed at 0.1.

\paragraph{Hybrid Methods:}
%We next list the details for running the hybrid methods, including the base model, SWA and SWAG. 
The hybrid methods (SGD, SWA and SWAG) are not very stable so we have to tune learning rates carefully for each dataset. We choose the momentum to be 0.9 for all cases and list all other information in Table \ref{tab:swag-params}. Notice that it is hard to train the hybrid methods on SVHN using AlexNet, so we initialize with a pre-trained model that is trained with a larger learning rate 0.1 to find a region with lower training loss, and then continue to optimize with the parameters listed in Table \ref{tab:swag-params}. 
%For the weight decay on small data experiments, we select it from the same range of $\eta$ based on validation log loss. 

\paragraph{Dropout:} For Dropout we add a Dropout layer following each activation layer in the base model and set the Dropout probability $p=0.1$. 
% We use the same training parameters as in the base model. 

\paragraph{Repulsive Ensembles:} Repulsive ensembles run multiple copies of the base model with a kernel base penalty to make sure the models are diverse. We use RBF kernel with lengthscale being the median of the square of the norm.
% the same training parameters as in the base model. 

\paragraph{Dirichlet:}
Dirichlet-based models are deterministic and they interpret the output of the last layer as the parameters of dirichlet distributions, i.e., $\bm{\alpha}(x)=g(f_W(x))$, where $g$ maps the output to positive real numbers. 
Hence we run the Dirichlet models with the setting of the base model. We next explain the setting of hyperparameters.
As discussed by \cite{pitfalls}, the models of \citet{dirichlet, dir-flow} implicitly perform variational inference:
\begin{align}
    \mathbf{p} \sim \text{Dir}(\bm{\alpha}_0), \quad y|\mathbf{p} \sim \text{Cat}(\mathbf{p}),
\end{align}
and the ELBO becomes
\begin{align}
    \log p(y|x) &\geq \E_{q(\mathbf{p} |x)}[\log p(y|\mathbf{p})] - \kl(q(\mathbf{p} |x) || \text{Dir}(\mathbf{p}|\bm{\alpha}_0)), 
\end{align}
where $q(\mathbf{p}|x) = \text{Dir}(\mathbf{p}|\bm{\alpha}(x))$. In the experiments, following \citet{dirichlet, pitfalls}, we use a uniform prior with $\bm{\alpha}_0 = [1, \dots, 1]$. As in VI and VIFO, we pick the regularization parameter for KL divergence to be 0.1. 
% \RK{Is the regularization parameter missing in the ELBO? (No, it is right above the explaination of hybrid methods.) }

%, except that we choose learning rate 0.0005 for boston and concrete which are hard to optimize. 

\paragraph{Last layer Laplace:} We first train a neural network to obtain a MAP solution with a prior variance of 0.05. Then, we use the code from \citet{DaxbergerKIEBH21} to optimize the prior precision hyperparameter through post-hoc marginal likelihood maximization.

\paragraph{VBLL:} We adapt the code from \citet{HarrisonWS24} and utilize the default hyperparameters. We choose the discriminative classification setting, as it yields the best OOD performance according to \citet{HarrisonWS24}.

\begin{table}[t]
    \centering
    \caption{The parameters for running the hybrid algorithms. ``lr''-learning rate, ``wd''-weight decay, ``swag\_lr''-the learning rate after we start collecting models in SWA and SWAG algorithms, ``swag\_start''-the epochs when we start to collect models, ``epochs''-the number of training epochs.}
    \label{tab:swag-params}
    \begin{tabular}{cccccc}
    \hline
         & lr & wd & swag\_lr & swag\_start & epochs \\
         \hline
        % UCI Classification & 0.001 & - & 0.001 & 1001 & 2000 \\
        % \hline
        % boston & 0.0005 & - & 0.0005 & 1001 & 5000 \\
        % concrete & 0.0005 & - & 0.0005 & 1001 & 5000 \\
        % yacht & 0.001 & - & 0.001 & 1001 & 5000 \\
        % power & 0.001 & - & 0.001 & 1001 & 5000 \\
        % \hline
        CIFAR10 / CIFAR100 & 0.05 & 0.0001 & 0.01 & 161 & 500 \\
        SVHN$^*$ & 0.001 & 0.0001 & 0.005 & 161 & 500 \\
        STL10 & 0.05 & 0.001 & 0.01 & 161 & 500 \\
        \hline
    \end{tabular}
\end{table}

\section{Additional Plots}
\subsection{Accuracy on PreResNet20}
\begin{figure}[H]
    \centering
    \begin{subfigure}[b]{0.40\textwidth}
         \centering
         \includegraphics[width=\textwidth]{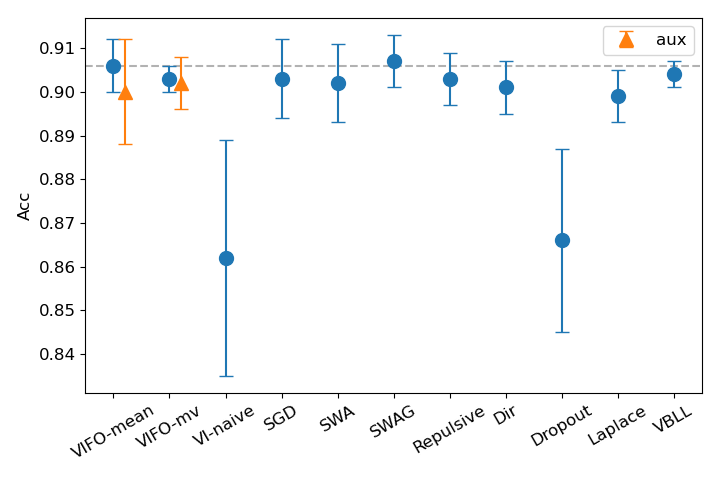}
         \caption{CIFAR10}
    \end{subfigure}
    \begin{subfigure}[b]{0.40\textwidth}
         \centering
         \includegraphics[width=\textwidth]{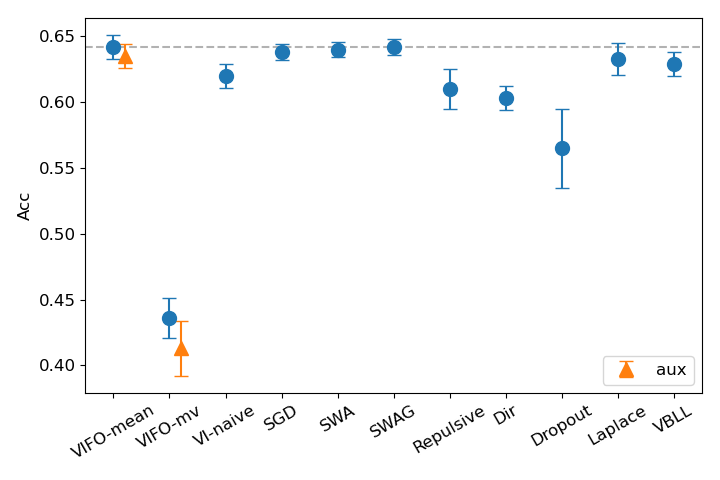}
         \caption{CIFAR100}
    \end{subfigure}
    \begin{subfigure}[b]{0.40\textwidth}
         \centering
         \includegraphics[width=\textwidth]{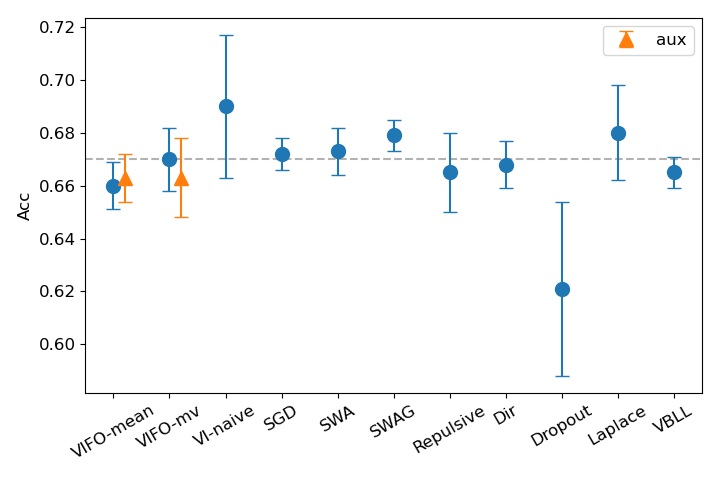}
         \caption{STL10}
    \end{subfigure}
    \begin{subfigure}[b]{0.40\textwidth}
         \centering
         \includegraphics[width=\textwidth]{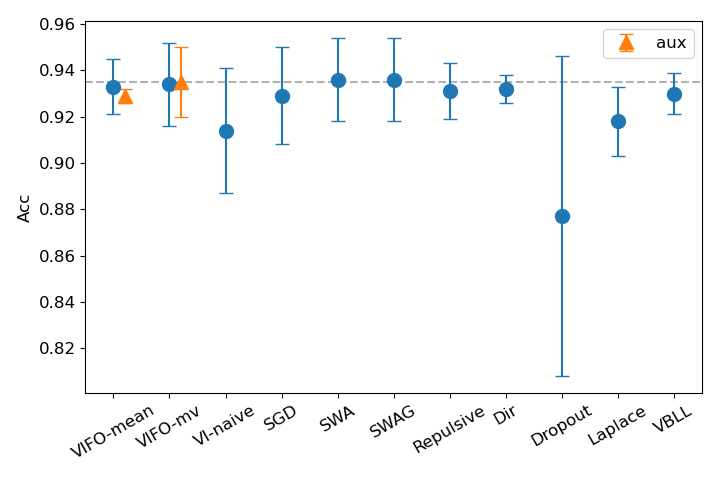}
         \caption{SVHN}
    \end{subfigure}
    \caption{Test accuracy ($\uparrow$) on PreResNet20. Dashed lines indicate the best version of VIFO.
    }
    \label{fig:acc-preresnet}
\end{figure}

\subsection{Results on AlexNet}
\PutAlexNetInD
\PutEntropyAlexNet

\subsection{AUROC Comparisons}
\begin{figure}[H]
    \centering
    \begin{subfigure}[b]{0.35\textwidth}
         \centering
         \includegraphics[width=\textwidth]{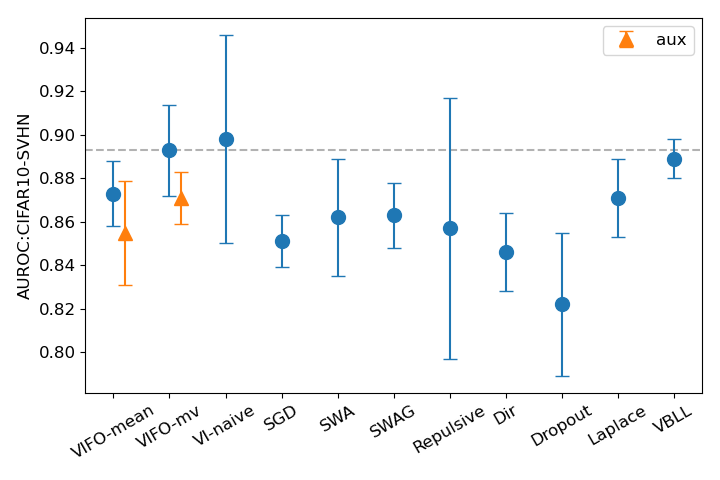}
         \caption{CIFAR10$\rightarrow$SVHN}
    \end{subfigure}
    \begin{subfigure}[b]{0.35\textwidth}
         \centering
         \includegraphics[width=\textwidth]{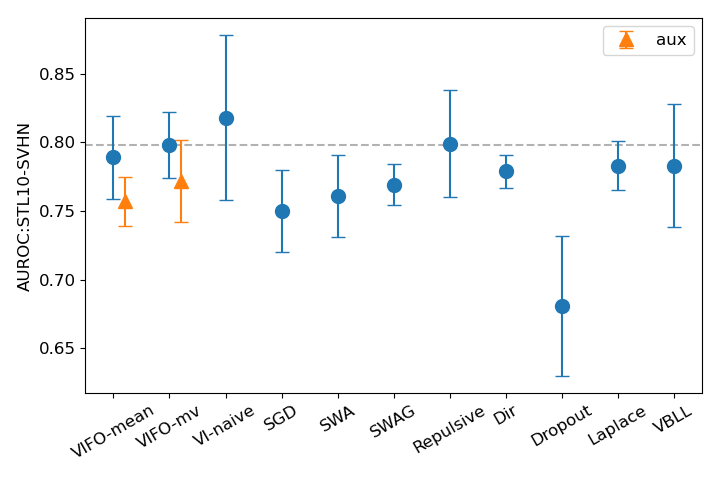}
         \caption{STL10$\rightarrow$SVHN}
    \end{subfigure}
    \begin{subfigure}[b]{0.35\textwidth}
         \centering
         \includegraphics[width=\textwidth]{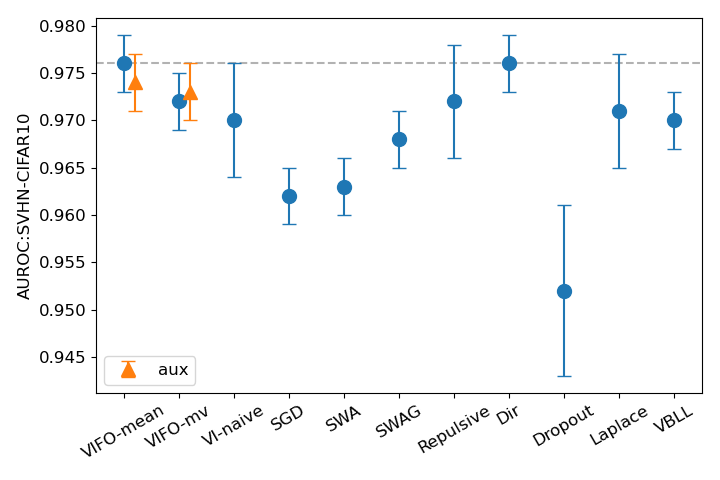}
         \caption{SVHN$\rightarrow$CIFAR10}
    \end{subfigure}
    \begin{subfigure}[b]{0.35\textwidth}
         \centering
         \includegraphics[width=\textwidth]{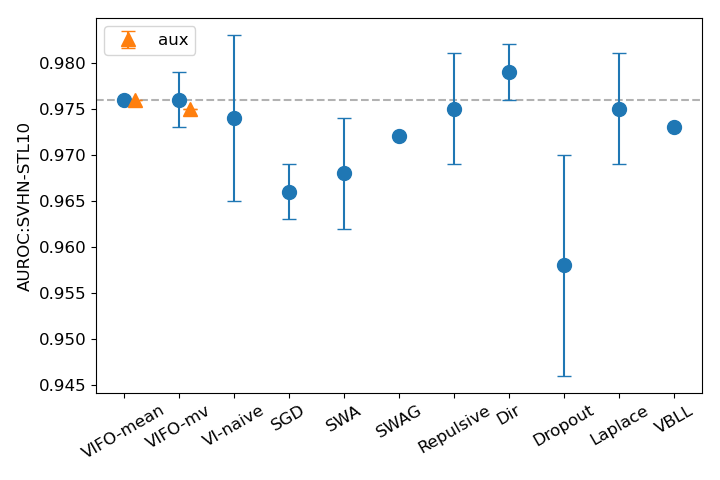}
         \caption{SVHN$\rightarrow$STL10}
    \end{subfigure}
    \caption{AUROC ($\uparrow$) on AlexNet.
    }
    \label{fig:auroc-alexnet}
\end{figure}

\begin{figure}[H]
    \centering
    \begin{subfigure}[b]{0.35\textwidth}
         \centering
         \includegraphics[width=\textwidth]{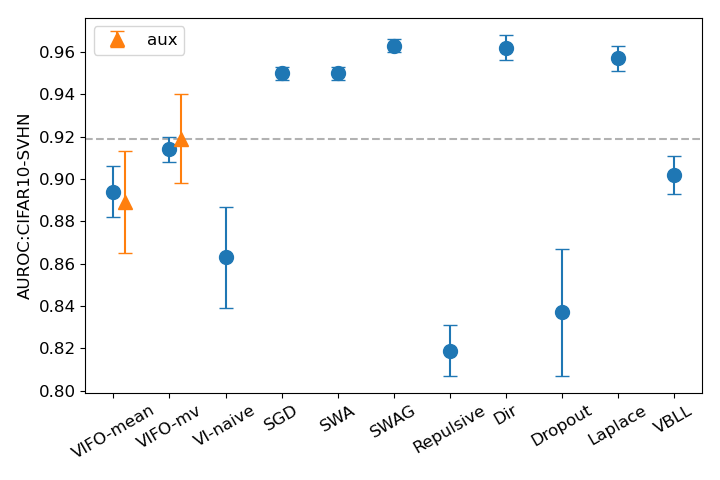}
         \caption{CIFAR10$\rightarrow$SVHN}
    \end{subfigure}
    \begin{subfigure}[b]{0.35\textwidth}
         \centering
         \includegraphics[width=\textwidth]{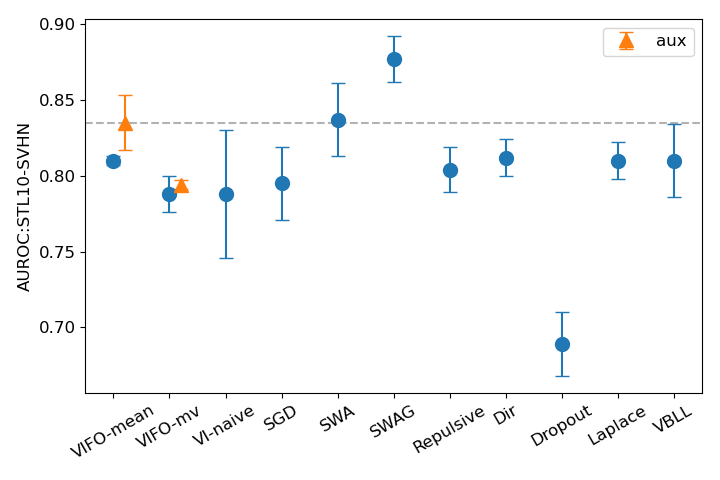}
         \caption{STL10$\rightarrow$SVHN}
    \end{subfigure}
    \begin{subfigure}[b]{0.35\textwidth}
         \centering
         \includegraphics[width=\textwidth]{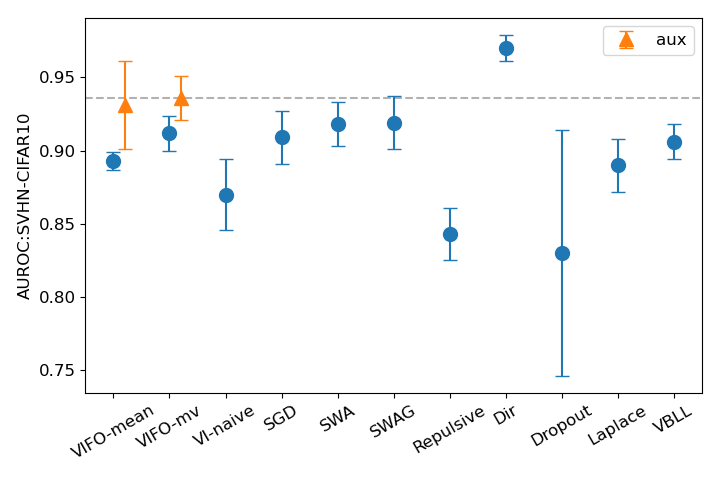}
         \caption{SVHN$\rightarrow$CIFAR10}
    \end{subfigure}
    \begin{subfigure}[b]{0.35\textwidth}
         \centering
         \includegraphics[width=\textwidth]{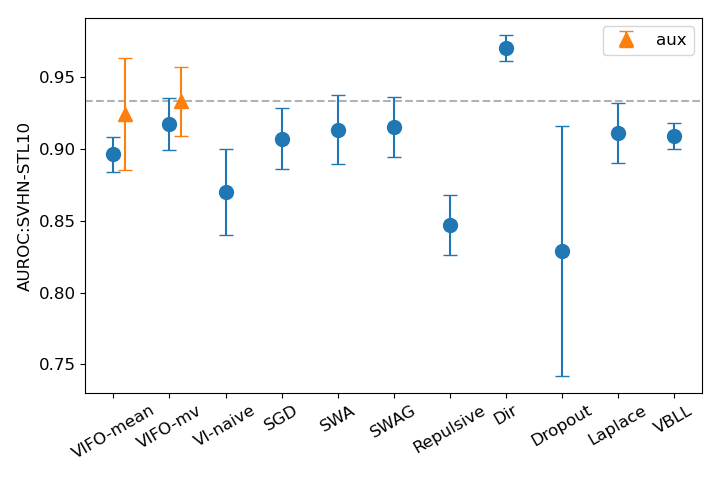}
         \caption{SVHN$\rightarrow$STL10}
    \end{subfigure}
    \caption{AUROC ($\uparrow$) on PreResNet20.
    }
    \label{fig:auroc-preresnet}
\end{figure}

\subsection{Learning Curves}
\label{sec:learning-curves}
\begin{figure}[H]
    \centering
    \begin{subfigure}[b]{0.35\textwidth}
         \centering
         \includegraphics[width=\textwidth]{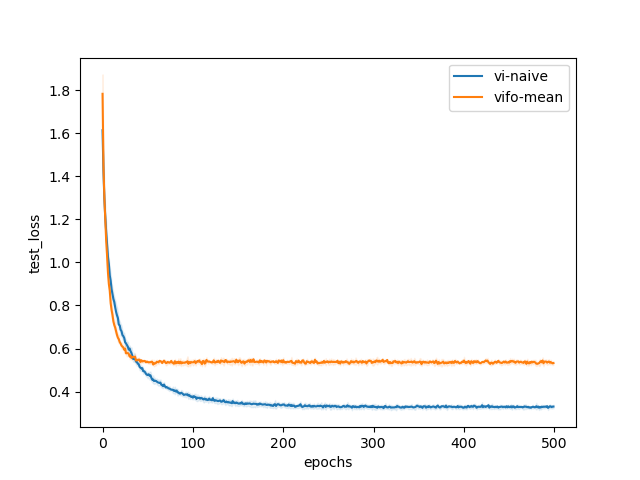}
         \caption{CIFAR10}
    \end{subfigure}
    \begin{subfigure}[b]{0.35\textwidth}
         \centering
         \includegraphics[width=\textwidth]{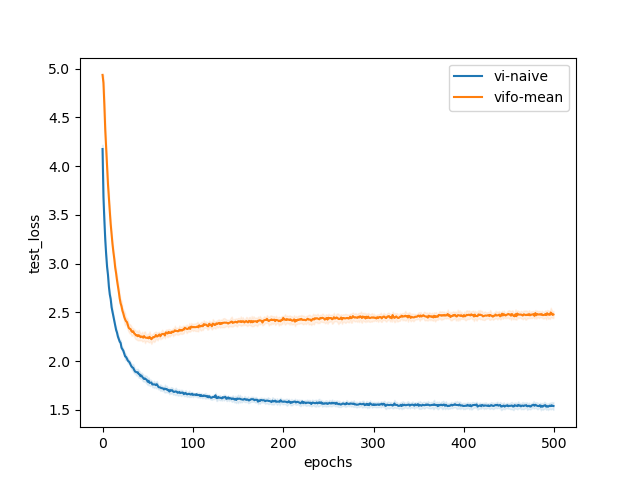}
         \caption{CIFAR100}
    \end{subfigure}
    \begin{subfigure}[b]{0.35\textwidth}
         \centering
         \includegraphics[width=\textwidth]{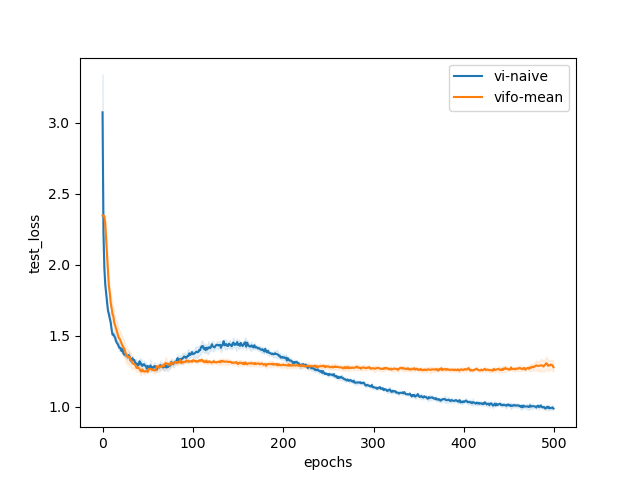}
         \caption{STL10}
    \end{subfigure}
    \begin{subfigure}[b]{0.35\textwidth}
         \centering
         \includegraphics[width=\textwidth]{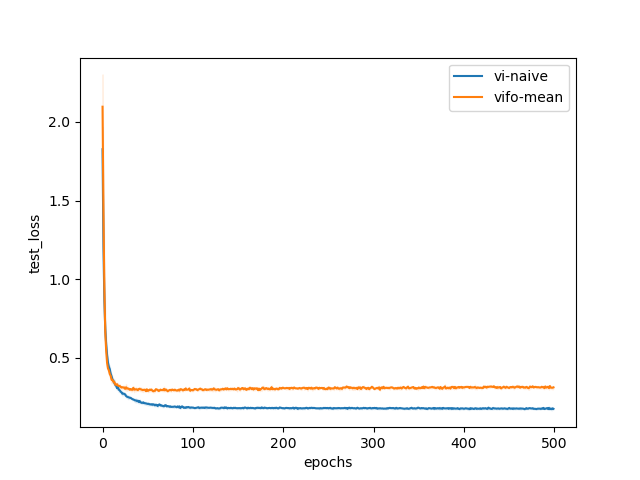}
         \caption{SVHN}
    \end{subfigure}
    \caption{Learning curves for all datasets on AlexNet. We conducted 5 independent runs and report the mean and standard deviation (which is very small). The results show that in all cases, VIFO-mean converges as fast as, or faster than, VI.
    }
    \label{fig:learning-curves}
\end{figure}

\newpage
\subsection{Ensembles}
\label{sec:ensemble}
\begin{figure}[H]
    \centering
    \begin{subfigure}[b]{0.35\textwidth}
         \centering
         \includegraphics[width=\textwidth]{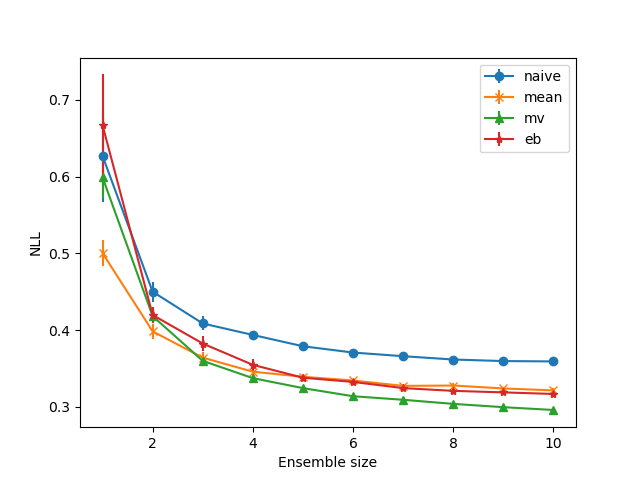}
         \caption{CIFAR10}
    \end{subfigure}
    \begin{subfigure}[b]{0.35\textwidth}
         \centering
         \includegraphics[width=\textwidth]{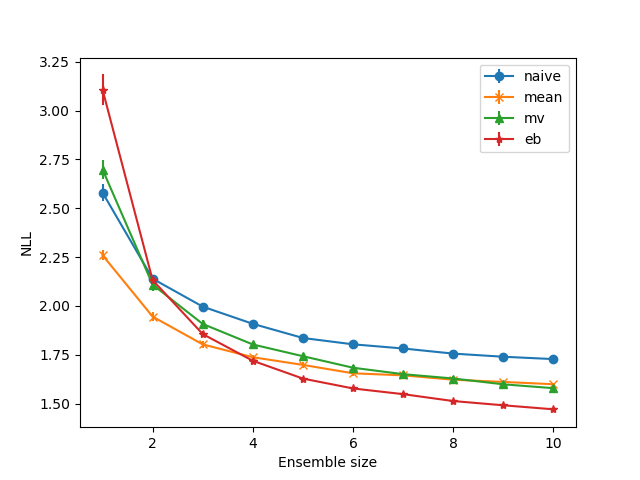}
         \caption{CIFAR100}
    \end{subfigure}
    \begin{subfigure}[b]{0.35\textwidth}
         \centering
         \includegraphics[width=\textwidth]{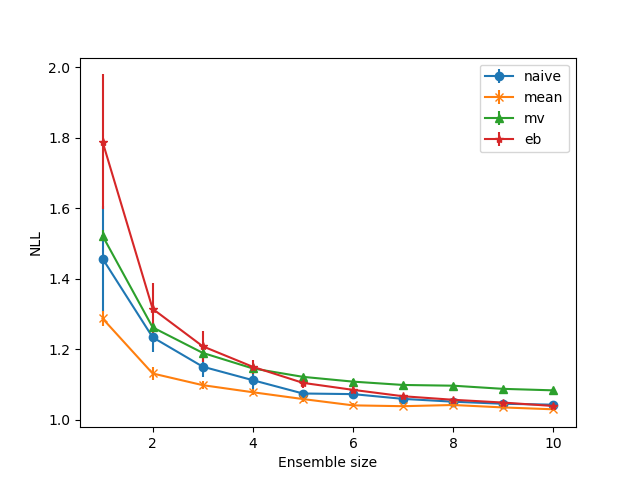}
         \caption{STL10}
    \end{subfigure}
    \begin{subfigure}[b]{0.35\textwidth}
         \centering
         \includegraphics[width=\textwidth]{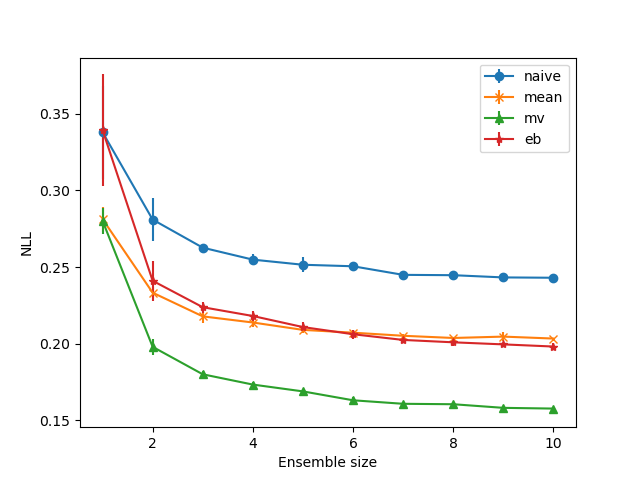}
         \caption{SVHN}
    \end{subfigure}
    \caption{Test losses vs. size of ensemble. Results are shown for 5 independent runs on AlexNet. 
    % We can see when the number of ensembles is larger than 5, increasing the number of ensembles does not improve the performance significantly. 
    }
\end{figure}

\begin{figure}[H]
    \centering
    \begin{subfigure}[b]{0.35\textwidth}
         \centering
         \includegraphics[width=\textwidth]{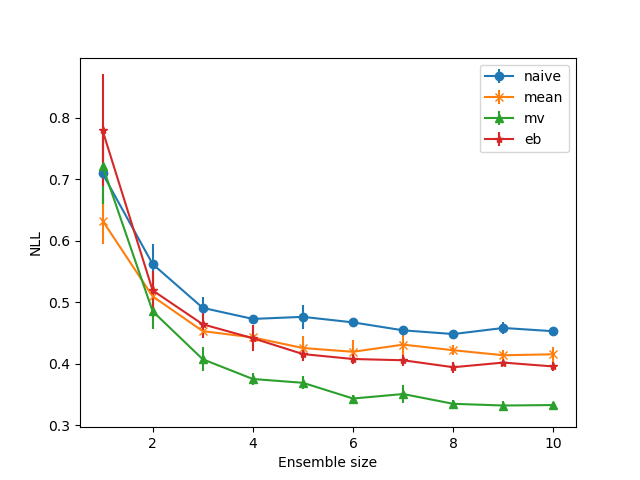}
         \caption{CIFAR10}
    \end{subfigure}
    \begin{subfigure}[b]{0.35\textwidth}
         \centering
         \includegraphics[width=\textwidth]{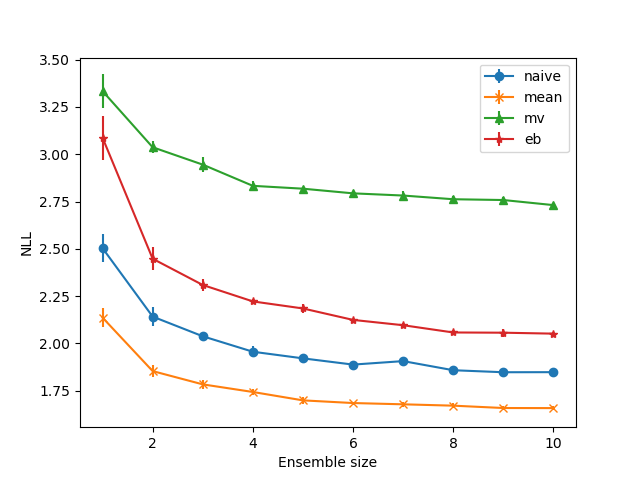}
         \caption{CIFAR100}
    \end{subfigure}
    \begin{subfigure}[b]{0.35\textwidth}
         \centering
         \includegraphics[width=\textwidth]{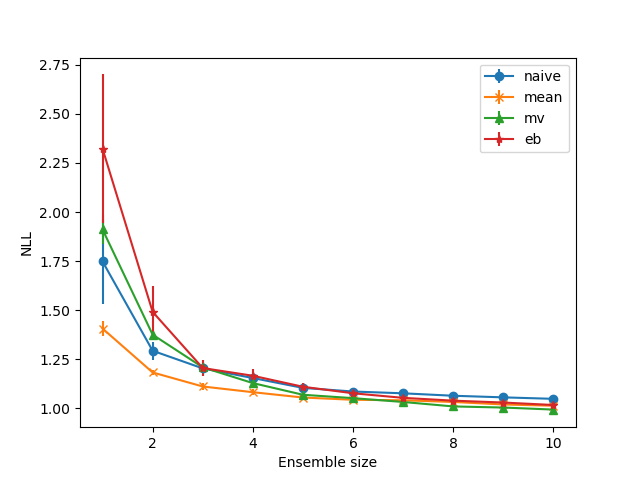}
         \caption{STL10}
    \end{subfigure}
    \begin{subfigure}[b]{0.35\textwidth}
         \centering
         \includegraphics[width=\textwidth]{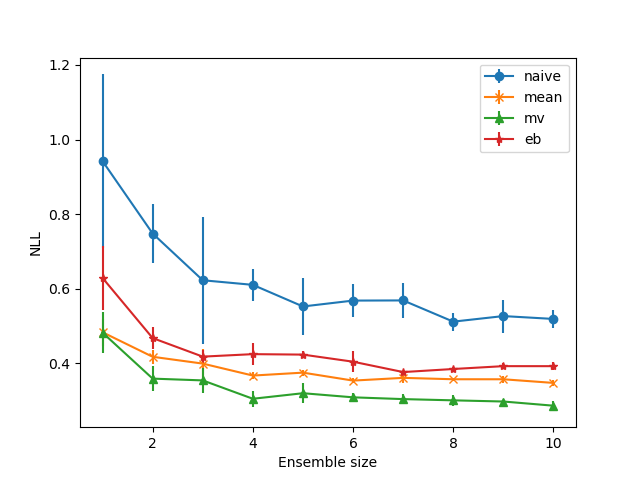}
         \caption{SVHN}
    \end{subfigure}
    \caption{Test losses vs. size of ensemble. Results are shown for 5 independent runs on PreResNet20.
    }
\end{figure}

\newpage
\section{Numerical Results}
\begin{table}[H]
\centering
\caption{CIFAR10, AlexNet, NLL}
\begin{tabular}{ccccc}
\hline
method & $\eta_\aux=0$ & $\eta_\aux=0.1$ & $\eta_\aux=0.5$ & $\eta_\aux=1.0$ \\
\hline
VIFO-naive &  $0.388 \pm 0.005$ & $0.383 \pm 0.002$ & $0.382 \pm 0.001$ & $0.383 \pm 0.002$ \\
VIFO-mean &  $0.338 \pm 0.002$ & $0.343 \pm 0.002$ & $0.344 \pm 0.003$ & $0.345 \pm 0.004$ \\
VIFO-mv &  $0.315 \pm 0.002$ & $0.324 \pm 0.000$ & $0.315 \pm 0.004$ & $0.311 \pm 0.002$ \\
VIFO-eb &  $0.347 \pm 0.003$ & $0.345 \pm 0.001$ & $0.345 \pm 0.003$ & $0.347 \pm 0.002$ \\
\hline
VI-naive &  $0.329 \pm 0.006$ &  &  &  \\
VI-mean &  $0.350 \pm 0.010$ &  &  &  \\
VI-mv &  $0.315 \pm 0.013$ &  &  &  \\
VI-eb &  $0.343 \pm 0.004$ &  &  &  \\
\hline
SGD &  $0.375 \pm 0.004$ &  &  &  \\
SWA &  $0.363 \pm 0.006$ &  &  &  \\
SWAG &  $0.329 \pm 0.002$ &  &  &  \\
\hline
Repulsive &  $0.785 \pm 0.003$ &  &  &  \\
\hline
Dir &  $0.788 \pm 0.002$ &  &  &  \\
\hline
Dropout &  $0.413 \pm 0.013$ &  &  &  \\
\hline
Laplace & $0.332 \pm 0.009$ & & & \\ 
\hline
VBLL & $0.373 \pm 0.003$ & & & \\ 
\hline
\end{tabular}
\end{table}

\begin{table}[H]
\centering
\caption{CIFAR100, AlexNet, NLL}
\begin{tabular}{ccccc}
\hline
method & $\eta_\aux=0$ & $\eta_\aux=0.1$ & $\eta_\aux=0.5$ & $\eta_\aux=1.0$ \\
\hline
VIFO-naive &  $1.774 \pm 0.008$ & $1.840 \pm 0.016$ & $1.784 \pm 0.008$ & $1.811 \pm 0.012$ \\
VIFO-mean &  $1.632 \pm 0.003$ & $1.687 \pm 0.002$ & $1.682 \pm 0.017$ & $1.683 \pm 0.021$ \\
VIFO-mv &  $1.642 \pm 0.011$ & $1.716 \pm 0.012$ & $1.971 \pm 0.053$ & $2.250 \pm 0.105$ \\
VIFO-eb &  $1.643 \pm 0.008$ & $1.651 \pm 0.003$ & $1.643 \pm 0.004$ & $1.649 \pm 0.006$ \\
\hline
VI-naive &  $1.513 \pm 0.024$ &  &  &  \\
VI-mean &  $1.642 \pm 0.023$ &  &  &  \\
VI-mv &  $1.817 \pm 0.022$ &  &  &  \\
VI-eb &  $1.441 \pm 0.016$ &  &  &  \\
\hline
SGD &  $1.894 \pm 0.024$ &  &  &  \\
SWA &  $1.768 \pm 0.026$ &  &  &  \\
SWAG &  $1.768 \pm 0.022$ &  &  &  \\
\hline
Repulsive &  $2.540 \pm 0.016$ &  &  &  \\
\hline
Dir &  $3.218 \pm 0.008$ &  &  &  \\
\hline
Dropout &  $2.024 \pm 0.022$ &  &  &  \\
\hline
Laplace & $1.738 \pm 0.006$ & & & \\ 
\hline
VBLL & $1.649 \pm 0.003$ & & & \\ 
\hline
\end{tabular}
\end{table}

\begin{table}[H]
\centering
\caption{STL10, AlexNet, NLL}
\begin{tabular}{ccccc}
\hline
method & $\eta_\aux=0$ & $\eta_\aux=0.1$ & $\eta_\aux=0.5$ & $\eta_\aux=1.0$ \\
\hline
VIFO-naive &  $1.113 \pm 0.006$ & $1.117 \pm 0.005$ & $1.112 \pm 0.005$ & $1.135 \pm 0.008$ \\
VIFO-mean &  $1.030 \pm 0.004$ & $1.056 \pm 0.005$ & $1.066 \pm 0.005$ & $1.067 \pm 0.011$ \\
VIFO-mv &  $1.078 \pm 0.005$ & $1.121 \pm 0.002$ & $1.112 \pm 0.010$ & $1.101 \pm 0.008$ \\
VIFO-eb &  $1.141 \pm 0.007$ & $1.134 \pm 0.008$ & $1.127 \pm 0.007$ & $1.135 \pm 0.002$ \\
\hline
VI-naive &  $0.975 \pm 0.010$ &  &  &  \\
VI-mean &  $1.021 \pm 0.013$ &  &  &  \\
VI-mv &  $1.095 \pm 0.018$ &  &  &  \\
VI-eb &  $1.560 \pm 0.054$ &  &  &  \\
\hline
SGD &  $1.419 \pm 0.025$ &  &  &  \\
SWA &  $1.139 \pm 0.011$ &  &  &  \\
SWAG &  $1.098 \pm 0.005$ &  &  &  \\
\hline
Repulsive &  $1.388 \pm 0.008$ &  &  &  \\
\hline
Dir &  $1.359 \pm 0.003$ &  &  &  \\
\hline
Dropout &  $2.359 \pm 0.045$ &  &  &  \\
\hline
Laplace & $1.157 \pm 0.007$ & & & \\ 
\hline
VBLL & $1.102 \pm 0.008$ & & & \\ 
\hline
\end{tabular}
\end{table}

\begin{table}[H]
\centering
\caption{SVHN, AlexNet, NLL}
\begin{tabular}{ccccc}
\hline
method & $\eta_\aux=0$ & $\eta_\aux=0.1$ & $\eta_\aux=0.5$ & $\eta_\aux=1.0$ \\
\hline
VIFO-naive &  $0.253 \pm 0.001$ & $0.256 \pm 0.001$ & $0.257 \pm 0.001$ & $0.251 \pm 0.002$ \\
VIFO-mean &  $0.211 \pm 0.002$ & $0.214 \pm 0.002$ & $0.209 \pm 0.001$ & $0.208 \pm 0.001$ \\
VIFO-mv &  $0.165 \pm 0.002$ & $0.169 \pm 0.001$ & $0.166 \pm 0.001$ & $0.170 \pm 0.002$ \\
VIFO-eb &  $0.211 \pm 0.001$ & $0.214 \pm 0.002$ & $0.212 \pm 0.002$ & $0.214 \pm 0.001$ \\
\hline
VI-naive &  $0.175 \pm 0.002$ &  &  &  \\
VI-mean &  $0.219 \pm 0.034$ &  &  &  \\
VI-mv &  $0.173 \pm 0.003$ &  &  &  \\
VI-eb &  $0.182 \pm 0.004$ &  &  &  \\
\hline
SGD &  $0.351 \pm 0.002$ &  &  &  \\
SWA &  $0.251 \pm 0.012$ &  &  &  \\
SWAG &  $0.238 \pm 0.001$ &  &  &  \\
\hline
Repulsive &  $0.686 \pm 0.007$ &  &  &  \\
\hline
Dir &  $0.692 \pm 0.003$ &  &  &  \\
\hline
Dropout &  $0.211 \pm 0.009$ &  &  &  \\
\hline
Laplace & $0.244 \pm 0.008$ & & & \\ 
\hline
VBLL & $0.177 \pm 0.001$ & & & \\ 
\hline
\end{tabular}
\end{table}

\begin{table}
\centering
\caption{CIFAR10, PreResNet20, NLL}
\begin{tabular}{ccccc}
\hline
method & $\eta_\aux=0$ & $\eta_\aux=0.1$ & $\eta_\aux=0.5$ & $\eta_\aux=1.0$ \\
\hline
VIFO-naive &  $0.473 \pm 0.005$ & $0.487 \pm 0.011$ & $0.486 \pm 0.010$ & $0.479 \pm 0.008$ \\
VIFO-mean &  $0.393 \pm 0.010$ & $0.433 \pm 0.018$ & $0.411 \pm 0.013$ & $0.430 \pm 0.020$ \\
VIFO-mv &  $0.361 \pm 0.010$ & $0.371 \pm 0.010$ & $0.367 \pm 0.006$ & $0.363 \pm 0.016$ \\
VIFO-eb &  $0.419 \pm 0.014$ & $0.429 \pm 0.016$ & $0.413 \pm 0.007$ & $0.425 \pm 0.008$ \\
\hline
VI-naive &  $0.410 \pm 0.028$ &  &  &  \\
VI-mean &  $0.415 \pm 0.032$ &  &  &  \\
VI-mv &  $0.437 \pm 0.029$ &  &  &  \\
VI-eb &  $0.429 \pm 0.035$ &  &  &  \\
\hline
SGD &  $0.335 \pm 0.013$ &  &  &  \\
SWA &  $0.336 \pm 0.008$ &  &  &  \\
SWAG &  $0.307 \pm 0.010$ &  &  &  \\
\hline
Repulsive &  $0.875 \pm 0.007$ &  &  &  \\
\hline
Dir &  $0.961 \pm 0.022$ &  &  &  \\
\hline
Dropout &  $0.423 \pm 0.026$ &  &  &  \\
\hline
Laplace & $0.324 \pm 0.009$ & & & \\ 
\hline
VBLL & $0.366 \pm 0.006$ & & & \\ 
\hline
\end{tabular}
\end{table}

\begin{table}
\centering
\caption{CIFAR100, PreResNet20, NLL}
\begin{tabular}{ccccc}
\hline
method & $\eta_\aux=0$ & $\eta_\aux=0.1$ & $\eta_\aux=0.5$ & $\eta_\aux=1.0$ \\
\hline
VIFO-naive &  $1.880 \pm 0.009$ & $1.935 \pm 0.020$ & $1.864 \pm 0.008$ & $1.844 \pm 0.011$ \\
VIFO-mean &  $1.632 \pm 0.013$ & $1.728 \pm 0.005$ & $1.686 \pm 0.004$ & $1.731 \pm 0.018$ \\
VIFO-mv &  $2.726 \pm 0.021$ & $2.826 \pm 0.008$ & $2.899 \pm 0.019$ & $2.867 \pm 0.014$ \\
VIFO-eb &  $2.076 \pm 0.006$ & $2.147 \pm 0.004$ & $2.340 \pm 0.043$ & $2.503 \pm 0.034$ \\
\hline
VI-naive &  $1.642 \pm 0.030$ &  &  &  \\
VI-mean &  $1.753 \pm 0.086$ &  &  &  \\
VI-mv &  $1.804 \pm 0.089$ &  &  &  \\
VI-eb &  $1.731 \pm 0.097$ &  &  &  \\
\hline
SGD &  $1.445 \pm 0.021$ &  &  &  \\
SWA &  $1.355 \pm 0.023$ &  &  &  \\
SWAG &  $1.354 \pm 0.019$ &  &  &  \\
\hline
Repulsive &  $2.948 \pm 0.020$ &  &  &  \\
\hline
Dir &  $3.580 \pm 0.009$ &  &  &  \\
\hline
Dropout &  $1.644 \pm 0.045$ &  &  &  \\
\hline
Laplace & $1.550 \pm 0.014$ & & & \\ 
\hline
VBLL & $1.484 \pm 0.017$ & & & \\ 
\hline
\end{tabular}
\end{table}

\begin{table}
\centering
\caption{STL10, PreResNet20, NLL}
\begin{tabular}{ccccc}
\hline
method & $\eta_\aux=0$ & $\eta_\aux=0.1$ & $\eta_\aux=0.5$ & $\eta_\aux=1.0$ \\
\hline
VIFO-naive &  $1.146 \pm 0.012$ & $1.159 \pm 0.008$ & $1.155 \pm 0.015$ & $1.145 \pm 0.005$ \\
VIFO-mean &  $1.067 \pm 0.009$ & $1.066 \pm 0.011$ & $1.056 \pm 0.007$ & $1.069 \pm 0.009$ \\
VIFO-mv &  $1.070 \pm 0.017$ & $1.078 \pm 0.003$ & $1.067 \pm 0.019$ & $1.073 \pm 0.011$ \\
VIFO-eb &  $1.162 \pm 0.015$ & $1.164 \pm 0.020$ & $1.191 \pm 0.014$ & $1.180 \pm 0.016$ \\
\hline
VI-naive &  $0.920 \pm 0.032$ &  &  &  \\
VI-mean &  $1.000 \pm 0.029$ &  &  &  \\
VI-mv &  $1.002 \pm 0.036$ &  &  &  \\
VI-eb &  $1.018 \pm 0.028$ &  &  &  \\
\hline
SGD &  $1.203 \pm 0.008$ &  &  &  \\
SWA &  $1.100 \pm 0.006$ &  &  &  \\
SWAG &  $1.100 \pm 0.010$ &  &  &  \\
\hline
Repulsive &  $1.365 \pm 0.009$ &  &  &  \\
\hline
Dir &  $1.418 \pm 0.019$ &  &  &  \\
\hline
Dropout &  $1.665 \pm 0.059$ &  &  &  \\
\hline
Laplace & $1.130 \pm 0.011$ & & & \\ 
\hline
VBLL & $1.072 \pm 0.009$ & & & \\ 
\hline
\end{tabular}
\end{table}

\begin{table}
\centering
\caption{SVHN, PreResNet20, NLL}
\begin{tabular}{ccccc}
\hline
method & $\eta_\aux=0$ & $\eta_\aux=0.1$ & $\eta_\aux=0.5$ & $\eta_\aux=1.0$ \\
\hline
VIFO-naive &  $0.391 \pm 0.019$ & $0.603 \pm 0.057$ & $0.486 \pm 0.050$ & $0.479 \pm 0.018$ \\
VIFO-mean &  $0.341 \pm 0.004$ & $0.357 \pm 0.004$ & $0.385 \pm 0.015$ & $0.334 \pm 0.015$ \\
VIFO-mv &  $0.269 \pm 0.009$ & $0.297 \pm 0.010$ & $0.323 \pm 0.013$ & $0.314 \pm 0.009$ \\
VIFO-eb &  $0.359 \pm 0.002$ & $0.391 \pm 0.008$ & $0.452 \pm 0.018$ & $0.402 \pm 0.028$ \\
\hline
VI-naive &  $0.314 \pm 0.024$ &  &  &  \\
VI-mean &  $0.342 \pm 0.040$ &  &  &  \\
VI-mv &  $0.359 \pm 0.033$ &  &  &  \\
VI-eb &  $0.379 \pm 0.057$ &  &  &  \\
\hline
SGD &  $0.337 \pm 0.011$ &  &  &  \\
SWA &  $0.320 \pm 0.009$ &  &  &  \\
SWAG &  $0.320 \pm 0.009$ &  &  &  \\
\hline
Repulsive & $0.822 \pm 0.020$ &  &  &  \\
\hline
Dir & $0.845 \pm 0.031$ &  &  &  \\
\hline
Dropout & $0.421 \pm 0.068$ &  &  &  \\
\hline
Laplace & $0.344 \pm 0.017$ & & & \\ 
\hline
VBLL & $0.324 \pm 0.008$ & & & \\ 
\hline
\end{tabular}
\end{table}

\begin{table}
\centering
\caption{CIFAR10, AlexNet, Accuracy}
\begin{tabular}{ccccc}
\hline
method & $\eta_\aux=0$ & $\eta_\aux=0.1$ & $\eta_\aux=0.5$ & $\eta_\aux=1.0$ \\
\hline
VIFO-naive &  $0.914 \pm 0.001$ & $0.916 \pm 0.001$ & $0.915 \pm 0.001$ & $0.916 \pm 0.001$ \\
VIFO-mean &  $0.916 \pm 0.002$ & $0.914 \pm 0.001$ & $0.914 \pm 0.001$ & $0.912 \pm 0.001$ \\
VIFO-mv &  $0.914 \pm 0.001$ & $0.910 \pm 0.001$ & $0.912 \pm 0.002$ & $0.915 \pm 0.002$ \\
VIFO-eb &  $0.914 \pm 0.002$ & $0.914 \pm 0.001$ & $0.913 \pm 0.001$ & $0.914 \pm 0.002$ \\
\hline
VI-naive &  $0.893 \pm 0.004$ &  &  &  \\
VI-mean &  $0.884 \pm 0.004$ &  &  &  \\
VI-mv &  $0.901 \pm 0.003$ &  &  &  \\
VI-eb &  $0.886 \pm 0.003$ &  &  &  \\
\hline
SGD &  $0.907 \pm 0.001$ &  &  &  \\
SWA &  $0.909 \pm 0.001$ &  &  &  \\
SWAG &  $0.909 \pm 0.001$ &  &  &  \\
\hline
Repulsive &  $0.915 \pm 0.001$ &  &  &  \\
\hline
Dir &  $0.915 \pm 0.001$ &  &  &  \\
\hline
Dropout &  $0.888 \pm 0.003$ &  &  &  \\
\hline
Laplace & $0.911 \pm 0.001$ & & & \\ 
\hline
VBLL & $0.918 \pm 0.001$ & & & \\ 
\hline
\end{tabular}
\end{table}

\begin{table}
\centering
\caption{CIFAR100, AlexNet, Accuracy}
\begin{tabular}{ccccc}
\hline
method & $\eta_\aux=0$ & $\eta_\aux=0.1$ & $\eta_\aux=0.5$ & $\eta_\aux=1.0$ \\
\hline
VIFO-naive &  $0.658 \pm 0.003$ & $0.651 \pm 0.004$ & $0.657 \pm 0.003$ & $0.644 \pm 0.003$ \\
VIFO-mean &  $0.667 \pm 0.002$ & $0.654 \pm 0.003$ & $0.650 \pm 0.007$ & $0.658 \pm 0.002$ \\
VIFO-mv &  $0.665 \pm 0.002$ & $0.638 \pm 0.010$ & $0.568 \pm 0.038$ & $0.558 \pm 0.028$ \\
VIFO-eb &  $0.669 \pm 0.002$ & $0.671 \pm 0.001$ & $0.668 \pm 0.003$ & $0.664 \pm 0.002$ \\
\hline
VI-naive &  $0.620 \pm 0.003$ &  &  &  \\
VI-mean &  $0.608 \pm 0.002$ &  &  &  \\
VI-mv &  $0.606 \pm 0.002$ &  &  &  \\
VI-eb &  $0.629 \pm 0.004$ &  &  &  \\
\hline
SGD &  $0.651 \pm 0.002$ &  &  &  \\
SWA &  $0.652 \pm 0.002$ &  &  &  \\
SWAG &  $0.652 \pm 0.002$ &  &  &  \\
\hline
Repulsive &  $0.653 \pm 0.002$ &  &  &  \\
\hline
Dir &  $0.654 \pm 0.001$ &  &  &  \\
\hline
Dropout &  $0.596 \pm 0.003$ &  &  &  \\
\hline
Laplace & $0.622 \pm 0.002$ & & & \\ 
\hline
VBLL & $0.660 \pm 0.001$ & & & \\ 
\hline
\end{tabular}
\end{table}

\begin{table}
\centering
\caption{STL10, AlexNet, Accuracy}
\begin{tabular}{ccccc}
\hline
method & $\eta_\aux=0$ & $\eta_\aux=0.1$ & $\eta_\aux=0.5$ & $\eta_\aux=1.0$ \\
\hline
VIFO-naive &  $0.670 \pm 0.003$ & $0.672 \pm 0.002$ & $0.674 \pm 0.002$ & $0.665 \pm 0.001$ \\
VIFO-mean &  $0.680 \pm 0.002$ & $0.666 \pm 0.003$ & $0.669 \pm 0.003$ & $0.663 \pm 0.003$ \\
VIFO-mv &  $0.678 \pm 0.002$ & $0.665 \pm 0.002$ & $0.665 \pm 0.002$ & $0.668 \pm 0.002$ \\
VIFO-eb &  $0.662 \pm 0.002$ & $0.668 \pm 0.002$ & $0.670 \pm 0.004$ & $0.670 \pm 0.002$ \\
\hline
VI-naive &  $0.661 \pm 0.003$ &  &  &  \\
VI-mean &  $0.649 \pm 0.009$ &  &  &  \\
VI-mv &  $0.644 \pm 0.006$ &  &  &  \\
VI-eb &  $0.414 \pm 0.025$ &  &  &  \\
\hline
SGD &  $0.614 \pm 0.004$ &  &  &  \\
SWA &  $0.598 \pm 0.006$ &  &  &  \\
SWAG &  $0.615 \pm 0.003$ &  &  &  \\
\hline
Repulsive &  $0.675 \pm 0.003$ &  &  &  \\
\hline
Dir &  $0.667 \pm 0.004$ &  &  &  \\
\hline
Dropout &  $0.622 \pm 0.004$ &  &  &  \\
\hline
Laplace & $0.648 \pm 0.004$ & & & \\ 
\hline
VBLL & $0.625 \pm 0.003$ & & & \\ 
\hline
\end{tabular}
\end{table}

\begin{table}
\centering
\caption{SVHN, AlexNet, Accuracy}
\begin{tabular}{ccccc}
\hline
method & $\eta_\aux=0$ & $\eta_\aux=0.1$ & $\eta_\aux=0.5$ & $\eta_\aux=1.0$ \\
\hline
VIFO-naive &  $0.962 \pm 0.001$ & $0.962 \pm 0.001$ & $0.962 \pm 0.000$ & $0.962 \pm 0.001$ \\
VIFO-mean &  $0.961 \pm 0.001$ & $0.962 \pm 0.000$ & $0.962 \pm 0.001$ & $0.963 \pm 0.001$ \\
VIFO-mv &  $0.963 \pm 0.001$ & $0.962 \pm 0.001$ & $0.963 \pm 0.000$ & $0.962 \pm 0.000$ \\
VIFO-eb &  $0.962 \pm 0.000$ & $0.961 \pm 0.001$ & $0.961 \pm 0.000$ & $0.961 \pm 0.000$ \\
\hline
VI-naive &  $0.953 \pm 0.002$ &  &  &  \\
VI-mean &  $0.945 \pm 0.008$ &  &  &  \\
VI-mv &  $0.955 \pm 0.001$ &  &  &  \\
VI-eb &  $0.950 \pm 0.001$ &  &  &  \\
\hline
SGD &  $0.955 \pm 0.001$ &  &  &  \\
SWA &  $0.954 \pm 0.002$ &  &  &  \\
SWAG &  $0.957 \pm 0.001$ &  &  &  \\
\hline
Repulsive &  $0.962 \pm 0.001$ &  &  &  \\
\hline
Dir &  $0.964 \pm 0.000$ &  &  &  \\
\hline
Dropout &  $0.949 \pm 0.001$ &  &  &  \\
\hline
Laplace & $0.961 \pm 0.001$ & & & \\ 
\hline
VBLL & $0.960 \pm 0.001$ & & & \\ 
\hline
\end{tabular}
\end{table}

\begin{table}
\centering
\caption{CIFAR10, PreResNet20, Accuracy}
\begin{tabular}{ccccc}
\hline
method & $\eta_\aux=0$ & $\eta_\aux=0.1$ & $\eta_\aux=0.5$ & $\eta_\aux=1.0$ \\
\hline
VIFO-naive &  $0.897 \pm 0.003$ & $0.896 \pm 0.002$ & $0.893 \pm 0.003$ & $0.899 \pm 0.003$ \\
VIFO-mean &  $0.906 \pm 0.002$ & $0.900 \pm 0.004$ & $0.907 \pm 0.002$ & $0.904 \pm 0.005$ \\
VIFO-mv &  $0.903 \pm 0.001$ & $0.902 \pm 0.002$ & $0.902 \pm 0.001$ & $0.903 \pm 0.002$ \\
VIFO-eb &  $0.900 \pm 0.003$ & $0.900 \pm 0.005$ & $0.902 \pm 0.002$ & $0.901 \pm 0.003$ \\
\hline
VI-naive &  $0.862 \pm 0.009$ &  &  &  \\
VI-mean &  $0.867 \pm 0.009$ &  &  &  \\
VI-mv &  $0.864 \pm 0.010$ &  &  &  \\
VI-eb &  $0.859 \pm 0.011$ &  &  &  \\
\hline
SGD &  $0.903 \pm 0.003$ &  &  &  \\
SWA &  $0.902 \pm 0.003$ &  &  &  \\
SWAG &  $0.907 \pm 0.002$ &  &  &  \\
\hline
Repulsive &  $0.903 \pm 0.002$ &  &  &  \\
\hline
Dir &  $0.901 \pm 0.002$ &  &  &  \\
\hline
Dropout &  $0.866 \pm 0.007$ &  &  &  \\
\hline
Laplace & $0.899 \pm 0.002$ & & & \\ 
\hline
VBLL & $0.904 \pm 0.001$ & & & \\ 
\hline
\end{tabular}
\end{table}

\begin{table}
\centering
\caption{CIFAR100, PreResNet20, Accuracy}
\begin{tabular}{ccccc}
\hline
method & $\eta_\aux=0$ & $\eta_\aux=0.1$ & $\eta_\aux=0.5$ & $\eta_\aux=1.0$ \\
\hline
VIFO-naive &  $0.631 \pm 0.004$ & $0.620 \pm 0.005$ & $0.631 \pm 0.002$ & $0.628 \pm 0.002$ \\
VIFO-mean &  $0.642 \pm 0.003$ & $0.635 \pm 0.003$ & $0.643 \pm 0.003$ & $0.637 \pm 0.006$ \\
VIFO-mv &  $0.436 \pm 0.005$ & $0.413 \pm 0.007$ & $0.398 \pm 0.006$ & $0.405 \pm 0.004$ \\
VIFO-eb &  $0.579 \pm 0.004$ & $0.567 \pm 0.009$ & $0.535 \pm 0.004$ & $0.499 \pm 0.017$ \\
\hline
VI-naive &  $0.620 \pm 0.003$ &  &  &  \\
VI-mean &  $0.608 \pm 0.002$ &  &  &  \\
VI-mv &  $0.606 \pm 0.002$ &  &  &  \\
VI-eb &  $0.629 \pm 0.004$ &  &  &  \\
\hline
SGD &  $0.638 \pm 0.002$ &  &  &  \\
SWA &  $0.640 \pm 0.002$ &  &  &  \\
SWAG &  $0.642 \pm 0.002$ &  &  &  \\
\hline
Repulsive &  $0.610 \pm 0.005$ &  &  &  \\
\hline
Dir &  $0.603 \pm 0.003$ &  &  &  \\
\hline
Dropout &  $0.565 \pm 0.010$ &  &  &  \\
\hline
Laplace & $0.633 \pm 0.004$ & & & \\ 
\hline
VBLL & $0.629 \pm 0.003$ & & & \\ 
\hline
\end{tabular}
\end{table}

\begin{table}
\centering
\caption{STL10, PreResNet20, Accuracy}
\begin{tabular}{ccccc}
\hline
method & $\eta_\aux=0$ & $\eta_\aux=0.1$ & $\eta_\aux=0.5$ & $\eta_\aux=1.0$ \\
\hline
VIFO-naive &  $0.658 \pm 0.004$ & $0.658 \pm 0.002$ & $0.656 \pm 0.001$ & $0.656 \pm 0.003$ \\
VIFO-mean &  $0.660 \pm 0.003$ & $0.663 \pm 0.003$ & $0.669 \pm 0.004$ & $0.663 \pm 0.004$ \\
VIFO-mv &  $0.670 \pm 0.004$ & $0.663 \pm 0.005$ & $0.662 \pm 0.005$ & $0.662 \pm 0.003$ \\
VIFO-eb &  $0.661 \pm 0.004$ & $0.661 \pm 0.004$ & $0.657 \pm 0.003$ & $0.660 \pm 0.005$ \\
\hline
VI-naive &  $0.690 \pm 0.009$ &  &  &  \\
VI-mean &  $0.674 \pm 0.010$ &  &  &  \\
VI-mv &  $0.688 \pm 0.009$ &  &  &  \\
VI-eb &  $0.638 \pm 0.011$ &  &  &  \\
\hline
SGD &  $0.672 \pm 0.002$ &  &  &  \\
SWA &  $0.673 \pm 0.003$ &  &  &  \\
SWAG &  $0.679 \pm 0.002$ &  &  &  \\
\hline
Repulsive &  $0.665 \pm 0.005$ &  &  &  \\
\hline
Dir &  $0.668 \pm 0.003$ &  &  &  \\
\hline
Dropout &  $0.621 \pm 0.011$ &  &  &  \\
\hline
Laplace & $0.680 \pm 0.006$ & & & \\ 
\hline
VBLL & $0.665 \pm 0.002$ & & & \\ 
\hline
\end{tabular}
\end{table}

\begin{table}
\centering
\caption{SVHN, PreResNet20, Accuracy}
\begin{tabular}{ccccc}
\hline
method & $\eta_\aux=0$ & $\eta_\aux=0.1$ & $\eta_\aux=0.5$ & $\eta_\aux=1.0$ \\
\hline
VIFO-naive &  $0.930 \pm 0.004$ & $0.911 \pm 0.014$ & $0.912 \pm 0.004$ & $0.915 \pm 0.008$ \\
VIFO-mean &  $0.933 \pm 0.004$ & $0.929 \pm 0.001$ & $0.930 \pm 0.001$ & $0.936 \pm 0.009$ \\
VIFO-mv &  $0.934 \pm 0.006$ & $0.935 \pm 0.005$ & $0.928 \pm 0.008$ & $0.930 \pm 0.002$ \\
VIFO-eb &  $0.929 \pm 0.004$ & $0.917 \pm 0.005$ & $0.905 \pm 0.003$ & $0.924 \pm 0.007$ \\
\hline
VI-naive &  $0.914 \pm 0.009$ &  &  &  \\
VI-mean &  $0.901 \pm 0.015$ &  &  &  \\
VI-mv &  $0.902 \pm 0.012$ &  &  &  \\
VI-eb &  $0.890 \pm 0.019$ &  &  &  \\
\hline
SGD &  $0.929 \pm 0.007$ &  &  &  \\
SWA &  $0.936 \pm 0.006$ &  &  &  \\
SWAG &  $0.936 \pm 0.006$ &  &  &  \\
\hline
Repulsive &  $0.931 \pm 0.004$ &  &  &  \\
\hline
Dir &  $0.932 \pm 0.002$ &  &  &  \\
\hline
Dropout &  $0.877 \pm 0.023$ &  &  &  \\
\hline
Laplace & $0.918 \pm 0.005$ & & & \\ 
\hline
VBLL & $0.930 \pm 0.003$ & & & \\ 
\hline
\end{tabular}
\end{table}

\begin{table}
\centering
\caption{ECE, CIFAR10-STL10, AlexNet}
\begin{tabular}{ccccc}
\hline
method & $\eta_\aux=0$ & $\eta_\aux=0.1$ & $\eta_\aux=0.5$ & $\eta_\aux=1.0$ \\
\hline
VIFO-naive &  $0.042 \pm 0.003$ & $0.042 \pm 0.002$ & $0.047 \pm 0.004$ & $0.045 \pm 0.003$ \\
VIFO-mean &  $0.056 \pm 0.002$ & $0.053 \pm 0.003$ & $0.055 \pm 0.001$ & $0.059 \pm 0.002$ \\
VIFO-mv &  $0.067 \pm 0.002$ & $0.069 \pm 0.002$ & $0.068 \pm 0.002$ & $0.068 \pm 0.003$ \\
VIFO-eb &  $0.039 \pm 0.003$ & $0.042 \pm 0.002$ & $0.038 \pm 0.002$ & $0.038 \pm 0.002$ \\
\hline
VI-naive &  $0.108 \pm 0.002$ &  &  &  \\
VI-mean &  $0.105 \pm 0.007$ &  &  &  \\
VI-mv &  $0.131 \pm 0.003$ &  &  &  \\
VI-eb &  $0.118 \pm 0.004$ &  &  &  \\
\hline
SGD &  $0.152 \pm 0.002$ &  &  &  \\
SWA &  $0.150 \pm 0.001$ &  &  &  \\
SWAG &  $0.144 \pm 0.001$ &  &  &  \\
\hline
Repulsive &  $0.246 \pm 0.003$ &  &  &  \\
\hline
Dir &  $0.248 \pm 0.002$ &  &  &  \\
\hline
Dropout &  $0.180 \pm 0.002$ &  &  &  \\
\hline
Laplace & $0.169 \pm 0.007$ & & & \\ 
\hline
VBLL & $0.095 \pm 0.002$ & & & \\ 
\hline
\end{tabular}
\end{table}

\begin{table}
\centering
\caption{ECE, STL10-CIFAR10, AlexNet}
\begin{tabular}{ccccc}
\hline
method & $\eta_\aux=0$ & $\eta_\aux=0.1$ & $\eta_\aux=0.5$ & $\eta_\aux=1.0$ \\
\hline
VIFO-naive &  $0.053 \pm 0.004$ & $0.053 \pm 0.004$ & $0.040 \pm 0.004$ & $0.034 \pm 0.004$ \\
VIFO-mean &  $0.065 \pm 0.002$ & $0.063 \pm 0.003$ & $0.055 \pm 0.003$ & $0.046 \pm 0.002$ \\
VIFO-mv &  $0.075 \pm 0.002$ & $0.067 \pm 0.003$ & $0.078 \pm 0.003$ & $0.073 \pm 0.002$ \\
VIFO-eb &  $0.081 \pm 0.004$ & $0.078 \pm 0.002$ & $0.066 \pm 0.002$ & $0.075 \pm 0.004$ \\
\hline
VI-naive &  $0.072 \pm 0.007$ &  &  &  \\
VI-mean &  $0.107 \pm 0.007$ &  &  &  \\
VI-mv &  $0.155 \pm 0.004$ &  &  &  \\
VI-eb &  $0.030 \pm 0.006$ &  &  &  \\
\hline
SGD &  $0.238 \pm 0.007$ &  &  &  \\
SWA &  $0.137 \pm 0.016$ &  &  &  \\
SWAG &  $0.117 \pm 0.007$ &  &  &  \\
\hline
Repulsive &  $0.197 \pm 0.006$ &  &  &  \\
\hline
Dir &  $0.146 \pm 0.004$ &  &  &  \\
\hline
Dropout &  $0.365 \pm 0.007$ &  &  &  \\
\hline
Laplace & $0.085 \pm 0.004$ & & & \\ 
\hline
VBLL & $0.129 \pm 0.007$ & & & \\ 
\hline
\end{tabular}
\end{table}

\begin{table}
\centering
\caption{ECE, CIFAR10-STL10, PreResNet20}
\begin{tabular}{ccccc}
\hline
method & $\eta_\aux=0$ & $\eta_\aux=0.1$ & $\eta_\aux=0.5$ & $\eta_\aux=1.0$ \\
\hline
VIFO-naive &  $0.029 \pm 0.001$ & $0.032 \pm 0.003$ & $0.032 \pm 0.004$ & $0.037 \pm 0.004$ \\
VIFO-mean &  $0.036 \pm 0.003$ & $0.028 \pm 0.005$ & $0.032 \pm 0.002$ & $0.029 \pm 0.003$ \\
VIFO-mv &  $0.058 \pm 0.003$ & $0.047 \pm 0.004$ & $0.052 \pm 0.002$ & $0.051 \pm 0.002$ \\
VIFO-eb &  $0.024 \pm 0.003$ & $0.028 \pm 0.005$ & $0.031 \pm 0.002$ & $0.026 \pm 0.002$ \\
\hline
VI-naive &  $0.126 \pm 0.004$ &  &  &  \\
VI-mean &  $0.147 \pm 0.002$ &  &  &  \\
VI-mv &  $0.159 \pm 0.003$ &  &  &  \\
VI-eb &  $0.137 \pm 0.005$ &  &  &  \\
\hline
SGD &  $0.092 \pm 0.004$ &  &  &  \\
SWA &  $0.089 \pm 0.001$ &  &  &  \\
SWAG &  $0.078 \pm 0.003$ &  &  &  \\
\hline
Repulsive &  $0.259 \pm 0.004$ &  &  &  \\
\hline
Dir &  $0.319 \pm 0.007$ &  &  &  \\
\hline
Dropout &  $0.146 \pm 0.002$ &  &  &  \\
\hline
Laplace & $0.063 \pm 0.002$ & & & \\ 
\hline
VBLL & $0.095 \pm 0.001$ & & & \\ 
\hline
\end{tabular}
\end{table}

\begin{table}
\centering
\caption{ECE, STL10-CIFAR10, PreResNet20}
\begin{tabular}{ccccc}
\hline
method & $\eta_\aux=0$ & $\eta_\aux=0.1$ & $\eta_\aux=0.5$ & $\eta_\aux=1.0$ \\
\hline
VIFO-naive &  $0.021 \pm 0.001$ & $0.022 \pm 0.002$ & $0.025 \pm 0.003$ & $0.028 \pm 0.004$ \\
VIFO-mean &  $0.029 \pm 0.003$ & $0.024 \pm 0.005$ & $0.019 \pm 0.002$ & $0.019 \pm 0.003$ \\
VIFO-mv &  $0.088 \pm 0.005$ & $0.081 \pm 0.003$ & $0.081 \pm 0.004$ & $0.086 \pm 0.002$ \\
VIFO-eb &  $0.051 \pm 0.005$ & $0.039 \pm 0.002$ & $0.030 \pm 0.003$ & $0.038 \pm 0.002$ \\
\hline
VI-naive &  $0.129 \pm 0.003$ &  &  &  \\
VI-mean &  $0.151 \pm 0.004$ &  &  &  \\
VI-mv &  $0.171 \pm 0.011$ &  &  &  \\
VI-eb &  $0.094 \pm 0.005$ &  &  &  \\
\hline
SGD &  $0.112 \pm 0.002$ &  &  &  \\
SWA &  $0.107 \pm 0.003$ &  &  &  \\
SWAG &  $0.093 \pm 0.004$ &  &  &  \\
\hline
Repulsive &  $0.185 \pm 0.006$ &  &  &  \\
\hline
Dir &  $0.200 \pm 0.005$ &  &  &  \\
\hline
Dropout &  $0.290 \pm 0.007$ &  &  &  \\
\hline
Laplace & $0.107 \pm 0.003$ & & & \\ 
\hline
VBLL & $0.097 \pm 0.009$ & & & \\ 
\hline
\end{tabular}
\end{table}

\begin{table}
\centering
\caption{Entropy, CIFAR10-SVHN, AlexNet}
\begin{tabular}{ccccc}
\hline
method & $\eta_\aux=0$ & $\eta_\aux=0.1$ & $\eta_\aux=0.5$ & $\eta_\aux=1.0$ \\
\hline
VIFO-naive &  $1.357 \pm 0.009$ & $1.399 \pm 0.023$ & $1.286 \pm 0.049$ & $1.355 \pm 0.033$ \\
VIFO-mean &  $1.344 \pm 0.020$ & $1.336 \pm 0.021$ & $1.313 \pm 0.042$ & $1.342 \pm 0.039$ \\
VIFO-mv &  $1.344 \pm 0.020$ & $1.342 \pm 0.018$ & $1.354 \pm 0.023$ & $1.332 \pm 0.017$ \\
VIFO-eb &  $1.330 \pm 0.045$ & $1.330 \pm 0.026$ & $1.323 \pm 0.017$ & $1.296 \pm 0.027$ \\
\hline
VI-naive &  $1.238 \pm 0.075$ &  &  &  \\
VI-mean &  $1.280 \pm 0.157$ &  &  &  \\
VI-mv &  $1.053 \pm 0.068$ &  &  &  \\
VI-eb &  $1.133 \pm 0.069$ &  &  &  \\
\hline
SGD &  $0.633 \pm 0.014$ &  &  &  \\
SWA &  $0.676 \pm 0.014$ &  &  &  \\
SWAG &  $0.705 \pm 0.011$ &  &  &  \\
\hline
Repulsive &  $1.990 \pm 0.021$ &  &  &  \\
\hline
Dir &  $1.970 \pm 0.002$ &  &  &  \\
\hline
Dropout &  $0.585 \pm 0.033$ &  &  &  \\
\hline
Laplace & $1.245 \pm 0.025$ & & & \\ 
\hline
VBLL & $0.971 \pm 0.049$ & & & \\ 
\hline
\end{tabular}
\end{table}

\begin{table}
\centering
\caption{Entropy, STL10-SVHN, AlexNet}
\begin{tabular}{ccccc}
\hline
method & $\eta_\aux=0$ & $\eta_\aux=0.1$ & $\eta_\aux=0.5$ & $\eta_\aux=1.0$ \\
\hline
VIFO-naive &  $1.772 \pm 0.017$ & $1.773 \pm 0.015$ & $1.779 \pm 0.024$ & $1.763 \pm 0.011$ \\
VIFO-mean &  $1.840 \pm 0.023$ & $1.748 \pm 0.012$ & $1.818 \pm 0.029$ & $1.826 \pm 0.015$ \\
VIFO-mv &  $1.889 \pm 0.026$ & $1.915 \pm 0.031$ & $1.848 \pm 0.051$ & $1.859 \pm 0.018$ \\
VIFO-eb &  $1.764 \pm 0.026$ & $1.754 \pm 0.018$ & $1.798 \pm 0.028$ & $1.716 \pm 0.012$ \\
\hline
VI-naive &  $1.601 \pm 0.049$ &  &  &  \\
VI-mean &  $1.495 \pm 0.018$ &  &  &  \\
VI-mv &  $1.345 \pm 0.046$ &  &  &  \\
VI-eb &  $2.024 \pm 0.019$ &  &  &  \\
\hline
SGD &  $1.127 \pm 0.030$ &  &  &  \\
SWA &  $1.479 \pm 0.047$ &  &  &  \\
SWAG &  $1.525 \pm 0.010$ &  &  &  \\
\hline
Repulsive &  $2.208 \pm 0.006$ &  &  &  \\
\hline
Dir &  $2.157 \pm 0.008$ &  &  &  \\
\hline
Dropout &  $0.601 \pm 0.037$ &  &  &  \\
\hline
Laplace & $1.349 \pm 0.007$ & & & \\ 
\hline
VBLL & $1.454 \pm 0.044$ & & & \\ 
\hline
\end{tabular}
\end{table}

\begin{table}
\centering
\caption{Entropy, SVHN-CIFAR10, AlexNet}
\begin{tabular}{ccccc}
\hline
method & $\eta_\aux=0$ & $\eta_\aux=0.1$ & $\eta_\aux=0.5$ & $\eta_\aux=1.0$ \\
\hline
VIFO-naive &  $1.694 \pm 0.008$ & $1.711 \pm 0.005$ & $1.694 \pm 0.003$ & $1.693 \pm 0.002$ \\
VIFO-mean &  $1.721 \pm 0.013$ & $1.735 \pm 0.007$ & $1.729 \pm 0.009$ & $1.742 \pm 0.008$ \\
VIFO-mv &  $1.709 \pm 0.014$ & $1.738 \pm 0.011$ & $1.754 \pm 0.014$ & $1.758 \pm 0.011$ \\
VIFO-eb &  $1.740 \pm 0.015$ & $1.749 \pm 0.005$ & $1.739 \pm 0.007$ & $1.773 \pm 0.033$ \\
\hline
VI-naive &  $1.624 \pm 0.048$ &  &  &  \\
VI-mean &  $1.711 \pm 0.052$ &  &  &  \\
VI-mv &  $1.448 \pm 0.036$ &  &  &  \\
VI-eb &  $1.577 \pm 0.058$ &  &  &  \\
\hline
SGD &  $1.237 \pm 0.009$ &  &  &  \\
SWA &  $1.323 \pm 0.035$ &  &  &  \\
SWAG &  $1.465 \pm 0.004$ &  &  &  \\
\hline
Repulsive &  $2.129 \pm 0.003$ &  &  &  \\
\hline
Dir &  $2.121 \pm 0.002$ &  &  &  \\
\hline
Dropout &  $1.206 \pm 0.047$ &  &  &  \\
\hline
Laplace & $1.696 \pm 0.036$ & & & \\ 
\hline
VBLL & $1.466 \pm 0.013$ & & & \\ 
\hline
\end{tabular}
\end{table}

\begin{table}
\centering
\caption{Entropy, SVHN-STL10, AlexNet}
\begin{tabular}{ccccc}
\hline
method & $\eta_\aux=0$ & $\eta_\aux=0.1$ & $\eta_\aux=0.5$ & $\eta_\aux=1.0$ \\
\hline
VIFO-naive &  $1.720 \pm 0.007$ & $1.734 \pm 0.005$ & $1.711 \pm 0.004$ & $1.708 \pm 0.008$ \\
VIFO-mean &  $1.744 \pm 0.007$ & $1.756 \pm 0.006$ & $1.753 \pm 0.011$ & $1.756 \pm 0.007$ \\
VIFO-mv &  $1.754 \pm 0.012$ & $1.763 \pm 0.006$ & $1.804 \pm 0.017$ & $1.790 \pm 0.009$ \\
VIFO-eb &  $1.764 \pm 0.016$ & $1.776 \pm 0.007$ & $1.775 \pm 0.017$ & $1.765 \pm 0.030$ \\
\hline
VI-naive &  $1.685 \pm 0.056$ &  &  &  \\
VI-mean &  $1.772 \pm 0.042$ &  &  &  \\
VI-mv &  $1.512 \pm 0.041$ &  &  &  \\
VI-eb &  $1.631 \pm 0.047$ &  &  &  \\
\hline
SGD &  $1.277 \pm 0.009$ &  &  &  \\
SWA &  $1.371 \pm 0.013$ &  &  &  \\
SWAG &  $1.507 \pm 0.007$ &  &  &  \\
\hline
Repulsive &  $2.138 \pm 0.004$ &  &  &  \\
\hline
Dir &  $2.132 \pm 0.002$ &  &  &  \\
\hline
Dropout &  $1.272 \pm 0.054$ &  &  &  \\
\hline
Laplace & $1.702 \pm 0.025$ & & & \\ 
\hline
VBLL & $1.508 \pm 0.009$ & & & \\ 
\hline
\end{tabular}
\end{table}

\begin{table}
\centering
\caption{Entropy, CIFAR10-SVHN, PreResNet20}
\begin{tabular}{ccccc}
\hline
method & $\eta_\aux=0$ & $\eta_\aux=0.1$ & $\eta_\aux=0.5$ & $\eta_\aux=1.0$ \\
\hline
VIFO-naive &  $1.465 \pm 0.007$ & $1.516 \pm 0.009$ & $1.553 \pm 0.006$ & $1.540 \pm 0.004$ \\
VIFO-mean &  $1.469 \pm 0.009$ & $1.553 \pm 0.018$ & $1.599 \pm 0.005$ & $1.580 \pm 0.010$ \\
VIFO-mv &  $1.559 \pm 0.013$ & $1.611 \pm 0.016$ & $1.590 \pm 0.017$ & $1.543 \pm 0.015$ \\
VIFO-eb &  $1.532 \pm 0.011$ & $1.540 \pm 0.013$ & $1.582 \pm 0.018$ & $1.548 \pm 0.005$ \\
\hline
VI-naive &  $1.192 \pm 0.025$ &  &  &  \\
VI-mean &  $1.097 \pm 0.015$ &  &  &  \\
VI-mv &  $1.048 \pm 0.019$ &  &  &  \\
VI-eb &  $1.174 \pm 0.049$ &  &  &  \\
\hline
SGD &  $1.398 \pm 0.009$ &  &  &  \\
SWA &  $1.398 \pm 0.015$ &  &  &  \\
SWAG &  $1.586 \pm 0.016$ &  &  &  \\
\hline
Repulsive &  $1.971 \pm 0.006$ &  &  &  \\
\hline
Dir &  $2.246 \pm 0.002$ &  &  &  \\
\hline
Dropout &  $0.995 \pm 0.022$ &  &  &  \\
\hline
Laplace & $1.523 \pm 0.010$ & & & \\ 
\hline
VBLL & $1.252 \pm 0.012$ & & & \\ 
\hline
\end{tabular}
\end{table}

\begin{table}
\centering
\caption{Entropy, STL10-SVHN, PreResNet20}
\begin{tabular}{ccccc}
\hline
method & $\eta_\aux=0$ & $\eta_\aux=0.1$ & $\eta_\aux=0.5$ & $\eta_\aux=1.0$ \\
\hline
VIFO-naive &  $1.763 \pm 0.008$ & $1.770 \pm 0.005$ & $1.742 \pm 0.009$ & $1.728 \pm 0.007$ \\
VIFO-mean &  $1.796 \pm 0.006$ & $1.908 \pm 0.008$ & $1.876 \pm 0.002$ & $1.862 \pm 0.008$ \\
VIFO-mv &  $1.565 \pm 0.007$ & $1.607 \pm 0.004$ & $1.588 \pm 0.005$ & $1.574 \pm 0.008$ \\
VIFO-eb &  $1.603 \pm 0.005$ & $1.637 \pm 0.011$ & $1.634 \pm 0.008$ & $1.656 \pm 0.017$ \\
\hline
VI-naive &  $1.394 \pm 0.020$ &  &  &  \\
VI-mean &  $1.323 \pm 0.013$ &  &  &  \\
VI-mv &  $1.267 \pm 0.030$ &  &  &  \\
VI-eb &  $1.517 \pm 0.060$ &  &  &  \\
\hline
SGD &  $1.432 \pm 0.022$ &  &  &  \\
SWA &  $1.462 \pm 0.013$ &  &  &  \\
SWAG &  $1.579 \pm 0.013$ &  &  &  \\
\hline
Repulsive &  $2.160 \pm 0.005$ &  &  &  \\
\hline
Dir &  $2.194 \pm 0.004$ &  &  &  \\
\hline
Dropout &  $0.830 \pm 0.024$ &  &  &  \\
\hline
Laplace & $1.313 \pm 0.006$ & & & \\ 
\hline
VBLL & $1.481 \pm 0.008$ & & & \\ 
\hline
\end{tabular}
\end{table}

\begin{table}
\centering
\caption{Entropy, SVHN-CIFAR10, PreResNet20}
\begin{tabular}{ccccc}
\hline
method & $\eta_\aux=0$ & $\eta_\aux=0.1$ & $\eta_\aux=0.5$ & $\eta_\aux=1.0$ \\
\hline
VIFO-naive &  $1.531 \pm 0.008$ & $2.019 \pm 0.026$ & $1.942 \pm 0.051$ & $1.959 \pm 0.039$ \\
VIFO-mean &  $1.515 \pm 0.012$ & $1.908 \pm 0.080$ & $1.850 \pm 0.065$ & $1.780 \pm 0.028$ \\
VIFO-mv &  $1.568 \pm 0.014$ & $1.758 \pm 0.051$ & $1.848 \pm 0.067$ & $1.706 \pm 0.057$ \\
VIFO-eb &  $1.539 \pm 0.004$ & $1.775 \pm 0.038$ & $1.898 \pm 0.022$ & $1.923 \pm 0.015$ \\
\hline
VI-naive &  $1.213 \pm 0.010$ &  &  &  \\
VI-mean &  $1.169 \pm 0.017$ &  &  &  \\
VI-mv &  $1.080 \pm 0.011$ &  &  &  \\
VI-eb &  $1.219 \pm 0.029$ &  &  &  \\
\hline
SGD &  $1.384 \pm 0.005$ &  &  &  \\
SWA &  $1.403 \pm 0.002$ &  &  &  \\
SWAG &  $1.476 \pm 0.005$ &  &  &  \\
\hline
Repulsive &  $2.006 \pm 0.013$ &  &  &  \\
\hline
Dir &  $2.250 \pm 0.001$ &  &  &  \\
\hline
Dropout &  $1.059 \pm 0.019$ &  &  &  \\
\hline
Laplace & $1.502 \pm 0.008$ & & & \\ 
\hline
VBLL & $1.315 \pm 0.007$ & & & \\ 
\hline
\end{tabular}
\end{table}

\begin{table}
\centering
\caption{Entropy, SVHN-STL10, PreResNet20}
\begin{tabular}{ccccc}
\hline
method & $\eta_\aux=0$ & $\eta_\aux=0.1$ & $\eta_\aux=0.5$ & $\eta_\aux=1.0$ \\
\hline
VIFO-naive &  $1.529 \pm 0.016$ & $1.995 \pm 0.042$ & $1.860 \pm 0.041$ & $1.885 \pm 0.054$ \\
VIFO-mean &  $1.543 \pm 0.012$ & $1.837 \pm 0.070$ & $1.875 \pm 0.066$ & $1.762 \pm 0.052$ \\
VIFO-mv &  $1.588 \pm 0.010$ & $1.763 \pm 0.061$ & $1.875 \pm 0.041$ & $1.671 \pm 0.039$ \\
VIFO-eb &  $1.544 \pm 0.004$ & $1.724 \pm 0.045$ & $1.895 \pm 0.054$ & $1.785 \pm 0.028$ \\
\hline
VI-naive &  $1.222 \pm 0.013$ &  &  &  \\
VI-mean &  $1.174 \pm 0.016$ &  &  &  \\
VI-mv &  $1.078 \pm 0.014$ &  &  &  \\
VI-eb &  $1.219 \pm 0.024$ &  &  &  \\
\hline
SGD &  $1.375 \pm 0.007$ &  &  &  \\
SWA &  $1.389 \pm 0.011$ &  &  &  \\
SWAG &  $1.464 \pm 0.005$ &  &  &  \\
\hline
Repulsive &  $2.014 \pm 0.011$ &  &  &  \\
\hline
Dir &  $2.252 \pm 0.002$ &  &  &  \\
\hline
Dropout &  $1.053 \pm 0.014$ &  &  &  \\
\hline
Laplace & $1.506 \pm 0.008$ & & & \\ 
\hline
VBLL & $1.319 \pm 0.006$ & & & \\ 
\hline
\end{tabular}
\end{table}

\begin{table}
\centering
\caption{AUROC, CIFAR10-SVHN, AlexNet}
\begin{tabular}{ccccc}
\hline
method & $\eta_\aux=0$ & $\eta_\aux=0.1$ & $\eta_\aux=0.5$ & $\eta_\aux=1.0$ \\
\hline
VIFO-naive &  $0.860 \pm 0.005$ & $0.857 \pm 0.008$ & $0.833 \pm 0.018$ & $0.853 \pm 0.018$ \\
VIFO-mean &  $0.873 \pm 0.005$ & $0.855 \pm 0.008$ & $0.858 \pm 0.007$ & $0.856 \pm 0.008$ \\
VIFO-mv &  $0.893 \pm 0.007$ & $0.871 \pm 0.004$ & $0.873 \pm 0.004$ & $0.861 \pm 0.006$ \\
VIFO-eb &  $0.860 \pm 0.007$ & $0.859 \pm 0.007$ & $0.864 \pm 0.002$ & $0.844 \pm 0.010$ \\
\hline
VI-naive &  $0.898 \pm 0.016$ &  &  &  \\
VI-mean &  $0.886 \pm 0.029$ &  &  &  \\
VI-mv &  $0.893 \pm 0.009$ &  &  &  \\
VI-eb &  $0.885 \pm 0.010$ &  &  &  \\
\hline
SGD &  $0.851 \pm 0.004$ &  &  &  \\
SWA &  $0.862 \pm 0.009$ &  &  &  \\
SWAG &  $0.863 \pm 0.005$ &  &  &  \\
\hline
Repulsive &  $0.857 \pm 0.020$ &  &  &  \\
\hline
Dir &  $0.846 \pm 0.006$ &  &  &  \\
\hline
Dropout &  $0.822 \pm 0.011$ &  &  &  \\
\hline
Laplace & $0.871 \pm 0.006$ & & & \\ 
\hline
VBLL & $0.889 \pm 0.003$ & & & \\ 
\hline
\end{tabular}
\end{table}

\begin{table}
\centering
\caption{AUROC, STL10-SVHN, AlexNet}
\begin{tabular}{ccccc}
\hline
method & $\eta_\aux=0$ & $\eta_\aux=0.1$ & $\eta_\aux=0.5$ & $\eta_\aux=1.0$ \\
\hline
VIFO-naive &  $0.787 \pm 0.005$ & $0.776 \pm 0.009$ & $0.781 \pm 0.010$ & $0.755 \pm 0.015$ \\
VIFO-mean &  $0.789 \pm 0.010$ & $0.757 \pm 0.006$ & $0.768 \pm 0.013$ & $0.774 \pm 0.006$ \\
VIFO-mv &  $0.798 \pm 0.008$ & $0.772 \pm 0.010$ & $0.764 \pm 0.016$ & $0.779 \pm 0.013$ \\
VIFO-eb &  $0.793 \pm 0.007$ & $0.793 \pm 0.010$ & $0.769 \pm 0.009$ & $0.758 \pm 0.005$ \\
\hline
VI-naive &  $0.818 \pm 0.020$ &  &  &  \\
VI-mean &  $0.792 \pm 0.011$ &  &  &  \\
VI-mv &  $0.775 \pm 0.018$ &  &  &  \\
VI-eb &  $0.736 \pm 0.044$ &  &  &  \\
\hline
SGD &  $0.750 \pm 0.010$ &  &  &  \\
SWA &  $0.761 \pm 0.010$ &  &  &  \\
SWAG &  $0.769 \pm 0.005$ &  &  &  \\
\hline
Repulsive &  $0.799 \pm 0.013$ &  &  &  \\
\hline
Dir &  $0.779 \pm 0.004$ &  &  &  \\
\hline
Dropout &  $0.681 \pm 0.017$ &  &  &  \\
\hline
Laplace & $0.783 \pm 0.006$ & & & \\ 
\hline
VBLL & $0.783 \pm 0.015$ & & & \\ 
\hline
\end{tabular}
\end{table}

\begin{table}
\centering
\caption{AUROC, SVHN-CIFAR10, AlexNet}
\begin{tabular}{ccccc}
\hline
method & $\eta_\aux=0$ & $\eta_\aux=0.1$ & $\eta_\aux=0.5$ & $\eta_\aux=1.0$ \\
\hline
VIFO-naive &  $0.969 \pm 0.001$ & $0.968 \pm 0.001$ & $0.967 \pm 0.001$ & $0.967 \pm 0.001$ \\
VIFO-mean &  $0.976 \pm 0.001$ & $0.974 \pm 0.001$ & $0.973 \pm 0.001$ & $0.973 \pm 0.002$ \\
VIFO-mv &  $0.972 \pm 0.001$ & $0.973 \pm 0.001$ & $0.972 \pm 0.000$ & $0.973 \pm 0.001$ \\
VIFO-eb &  $0.969 \pm 0.001$ & $0.971 \pm 0.001$ & $0.970 \pm 0.001$ & $0.973 \pm 0.003$ \\
\hline
VI-naive &  $0.970 \pm 0.002$ &  &  &  \\
VI-mean &  $0.963 \pm 0.014$ &  &  &  \\
VI-mv &  $0.965 \pm 0.003$ &  &  &  \\
VI-eb &  $0.967 \pm 0.003$ &  &  &  \\
\hline
SGD &  $0.962 \pm 0.001$ &  &  &  \\
SWA &  $0.963 \pm 0.001$ &  &  &  \\
SWAG &  $0.968 \pm 0.001$ &  &  &  \\
\hline
Repulsive &  $0.972 \pm 0.002$ &  &  &  \\
\hline
Dir &  $0.976 \pm 0.001$ &  &  &  \\
\hline
Dropout &  $0.952 \pm 0.003$ &  &  &  \\
\hline
Laplace & $0.971 \pm 0.002$ & & & \\ 
\hline
VBLL & $0.970 \pm 0.001$ & & & \\ 
\hline
\end{tabular}
\end{table}

\begin{table}
\centering
\caption{AUROC, SVHN-STL10, AlexNet}
\begin{tabular}{ccccc}
\hline
method & $\eta_\aux=0$ & $\eta_\aux=0.1$ & $\eta_\aux=0.5$ & $\eta_\aux=1.0$ \\
\hline
VIFO-naive &  $0.971 \pm 0.000$ & $0.972 \pm 0.001$ & $0.969 \pm 0.001$ & $0.971 \pm 0.001$ \\
VIFO-mean &  $0.976 \pm 0.000$ & $0.976 \pm 0.000$ & $0.977 \pm 0.001$ & $0.977 \pm 0.001$ \\
VIFO-mv &  $0.976 \pm 0.001$ & $0.975 \pm 0.000$ & $0.977 \pm 0.001$ & $0.976 \pm 0.000$ \\
VIFO-eb &  $0.974 \pm 0.001$ & $0.974 \pm 0.001$ & $0.973 \pm 0.001$ & $0.977 \pm 0.002$ \\
\hline
VI-naive &  $0.974 \pm 0.003$ &  &  &  \\
VI-mean &  $0.968 \pm 0.013$ &  &  &  \\
VI-mv &  $0.970 \pm 0.002$ &  &  &  \\
VI-eb &  $0.971 \pm 0.002$ &  &  &  \\
\hline
SGD &  $0.966 \pm 0.001$ &  &  &  \\
SWA &  $0.968 \pm 0.002$ &  &  &  \\
SWAG &  $0.972 \pm 0.000$ &  &  &  \\
\hline
Repulsive &  $0.975 \pm 0.002$ &  &  &  \\
\hline
Dir &  $0.979 \pm 0.001$ &  &  &  \\
\hline
Dropout &  $0.958 \pm 0.004$ &  &  &  \\
\hline
Laplace & $0.975 \pm 0.002$ & & & \\ 
\hline
VBLL & $0.973 \pm 0.000$ & & & \\ 
\hline
\end{tabular}
\end{table}

\begin{table}
\centering
\caption{AUROC, CIFAR10-SVHN, PreResNet20}
\begin{tabular}{ccccc}
\hline
method & $\eta_\aux=0$ & $\eta_\aux=0.1$ & $\eta_\aux=0.5$ & $\eta_\aux=1.0$ \\
\hline
VIFO-naive &  $0.885 \pm 0.004$ & $0.880 \pm 0.006$ & $0.889 \pm 0.005$ & $0.898 \pm 0.006$ \\
VIFO-mean &  $0.894 \pm 0.004$ & $0.889 \pm 0.008$ & $0.907 \pm 0.007$ & $0.896 \pm 0.002$ \\
VIFO-mv &  $0.914 \pm 0.002$ & $0.919 \pm 0.007$ & $0.917 \pm 0.001$ & $0.915 \pm 0.003$ \\
VIFO-eb &  $0.908 \pm 0.002$ & $0.912 \pm 0.005$ & $0.918 \pm 0.002$ & $0.913 \pm 0.004$ \\
\hline
VI-naive &  $0.863 \pm 0.008$ &  &  &  \\
VI-mean &  $0.868 \pm 0.011$ &  &  &  \\
VI-mv &  $0.868 \pm 0.008$ &  &  &  \\
VI-eb &  $0.858 \pm 0.013$ &  &  &  \\
\hline
SGD &  $0.950 \pm 0.001$ &  &  &  \\
SWA &  $0.950 \pm 0.001$ &  &  &  \\
SWAG &  $0.963 \pm 0.001$ &  &  &  \\
\hline
Repulsive &  $0.819 \pm 0.004$ &  &  &  \\
\hline
Dir &  $0.962 \pm 0.002$ &  &  &  \\
\hline
Dropout &  $0.837 \pm 0.010$ &  &  &  \\
\hline
Laplace & $0.957 \pm 0.002$ & & & \\ 
\hline
VBLL & $0.902 \pm 0.003$ & & & \\ 
\hline
\end{tabular}
\end{table}

\begin{table}
\centering
\caption{AUROC, STL10-SVHN, PreResNet20}
\begin{tabular}{ccccc}
\hline
method & $\eta_\aux=0$ & $\eta_\aux=0.1$ & $\eta_\aux=0.5$ & $\eta_\aux=1.0$ \\
\hline
VIFO-naive &  $0.810 \pm 0.005$ & $0.810 \pm 0.004$ & $0.804 \pm 0.004$ & $0.803 \pm 0.003$ \\
VIFO-mean &  $0.810 \pm 0.001$ & $0.835 \pm 0.006$ & $0.836 \pm 0.004$ & $0.820 \pm 0.003$ \\
VIFO-mv &  $0.788 \pm 0.004$ & $0.794 \pm 0.001$ & $0.789 \pm 0.003$ & $0.784 \pm 0.005$ \\
VIFO-eb &  $0.799 \pm 0.004$ & $0.800 \pm 0.005$ & $0.803 \pm 0.002$ & $0.803 \pm 0.004$ \\
\hline
VI-naive &  $0.788 \pm 0.014$ &  &  &  \\
VI-mean &  $0.784 \pm 0.006$ &  &  &  \\
VI-mv &  $0.788 \pm 0.015$ &  &  &  \\
VI-eb &  $0.728 \pm 0.030$ &  &  &  \\
\hline
SGD &  $0.795 \pm 0.008$ &  &  &  \\
SWA &  $0.837 \pm 0.008$ &  &  &  \\
SWAG &  $0.877 \pm 0.005$ &  &  &  \\
\hline
Repulsive &  $0.804 \pm 0.005$ &  &  &  \\
\hline
Dir &  $0.812 \pm 0.004$ &  &  &  \\
\hline
Dropout &  $0.689 \pm 0.007$ &  &  &  \\
\hline
Laplace & $0.810 \pm 0.004$ & & & \\ 
\hline
VBLL & $0.810 \pm 0.008$ & & & \\ 
\hline
\end{tabular}
\end{table}

\begin{table}
\centering
\caption{AUROC, SVHN-CIFAR10, PreResNet20}
\begin{tabular}{ccccc}
\hline
method & $\eta_\aux=0$ & $\eta_\aux=0.1$ & $\eta_\aux=0.5$ & $\eta_\aux=1.0$ \\
\hline
VIFO-naive &  $0.890 \pm 0.005$ & $0.958 \pm 0.003$ & $0.934 \pm 0.010$ & $0.948 \pm 0.004$ \\
VIFO-mean &  $0.893 \pm 0.002$ & $0.931 \pm 0.010$ & $0.925 \pm 0.006$ & $0.928 \pm 0.014$ \\
VIFO-mv &  $0.912 \pm 0.004$ & $0.936 \pm 0.005$ & $0.927 \pm 0.007$ & $0.916 \pm 0.006$ \\
VIFO-eb &  $0.916 \pm 0.007$ & $0.922 \pm 0.002$ & $0.923 \pm 0.008$ & $0.941 \pm 0.009$ \\
\hline
VI-naive &  $0.870 \pm 0.008$ &  &  &  \\
VI-mean &  $0.867 \pm 0.016$ &  &  &  \\
VI-mv &  $0.866 \pm 0.012$ &  &  &  \\
VI-eb &  $0.847 \pm 0.026$ &  &  &  \\
\hline
SGD &  $0.909 \pm 0.006$ &  &  &  \\
SWA &  $0.918 \pm 0.005$ &  &  &  \\
SWAG &  $0.919 \pm 0.006$ &  &  &  \\
\hline
Repulsive &  $0.843 \pm 0.006$ &  &  &  \\
\hline
Dir &  $0.970 \pm 0.003$ &  &  &  \\
\hline
Dropout &  $0.830 \pm 0.028$ &  &  &  \\
\hline
Laplace & $0.890 \pm 0.006$ & & & \\ 
\hline
VBLL & $0.906 \pm 0.004$ & & & \\ 
\hline
\end{tabular}
\end{table}

\begin{table}
\centering
\caption{AUROC, SVHN-STL10, PreResNet20}
\begin{tabular}{ccccc}
\hline
method & $\eta_\aux=0$ & $\eta_\aux=0.1$ & $\eta_\aux=0.5$ & $\eta_\aux=1.0$ \\
\hline
VIFO-naive &  $0.894 \pm 0.009$ & $0.957 \pm 0.009$ & $0.934 \pm 0.008$ & $0.936 \pm 0.010$ \\
VIFO-mean &  $0.896 \pm 0.004$ & $0.924 \pm 0.013$ & $0.914 \pm 0.010$ & $0.921 \pm 0.009$ \\
VIFO-mv &  $0.917 \pm 0.006$ & $0.933 \pm 0.008$ & $0.928 \pm 0.009$ & $0.922 \pm 0.001$ \\
VIFO-eb &  $0.911 \pm 0.006$ & $0.916 \pm 0.005$ & $0.923 \pm 0.004$ & $0.936 \pm 0.006$ \\
\hline
VI-naive &  $0.870 \pm 0.010$ &  &  &  \\
VI-mean &  $0.867 \pm 0.016$ &  &  &  \\
VI-mv &  $0.865 \pm 0.011$ &  &  &  \\
VI-eb &  $0.846 \pm 0.027$ &  &  &  \\
\hline
SGD &  $0.907 \pm 0.007$ &  &  &  \\
SWA &  $0.913 \pm 0.008$ &  &  &  \\
SWAG &  $0.915 \pm 0.007$ &  &  &  \\
\hline
Repulsive &  $0.847 \pm 0.007$ &  &  &  \\
\hline
Dir &  $0.970 \pm 0.003$ &  &  &  \\
\hline
Dropout &  $0.829 \pm 0.029$ &  &  &  \\
\hline
Laplace & $0.911 \pm 0.007$ & & & \\ 
\hline
VBLL & $0.909 \pm 0.003$ & & & \\ 
\hline
\end{tabular}
\end{table}

\end{document}